\documentclass[12pt]{article}

\def\includeversion#1{%
  \expandafter\def\csname #1\endcsname{}%
  \expandafter\def\csname end#1\endcsname{}%
  \expandafter\let\csname not-#1\endcsname\comment
  \expandafter\let\csname endnot-#1\endcsname\endcomment
  \expandafter\def\csname if#1\endcsname##1##2{##1}
}
\def\excludeversion#1{%
  \expandafter\let\csname #1\endcsname\comment
  \expandafter\let\csname end#1\endcsname\endcomment
  \expandafter\def\csname not-#1\endcsname{}%
  \expandafter\def\csname endnot-#1\endcsname{}%
  \expandafter\def\csname if#1\endcsname##1##2{##2}
}
\usepackage{authblk}
\usepackage{caption}
\usepackage{float}
\usepackage{changes}
\usepackage{algorithm}
\usepackage{algpseudocode}
\usepackage{graphicx}
\usepackage{mathtools}
\usepackage{adjustbox}
\usepackage{subcaption}
\usepackage{setspace,epstopdf,amsmath,amsfonts,amssymb,amsthm}
\usepackage{grffile} 
\usepackage{marginnote,datetime,enumitem,rotating,fancyvrb}
\usepackage{float}
\usepackage[
  colorlinks=true,
  linkcolor=blue,
  citecolor=blue]{hyperref}
\usepackage{booktabs}
\usepackage[longnamesfirst]{natbib}
\usepackage{mathtools}
\usepackage{lscape}
\usepackage[toc,page]{appendix}

\usepackage{tikz}
\usetikzlibrary{positioning}
\usepackage{forest}

\usetikzlibrary{trees}

\newlength\mytemplength

\usdate

\excludeversion{notes}        
\includeversion{links}          

\iflinks{}{\hypersetup{draft=true}}
\ifnotes{%
\usepackage[margin=1.0in,paperwidth=10in,right=2.5in]{geometry}%
\usepackage[textwidth=1.4in,shadow,colorinlistoftodos]{todonotes}%
}{%
\usepackage[margin=1in]{geometry}%
\usepackage[]{todonotes}%
}

\makeatletter\let\chapter\@undefined\makeatother 

\newtheorem{theorem}{Theorem}[section]
\newtheorem{proposition}[theorem]{Proposition}
\newtheorem{lemma}[theorem]{Lemma}
\newtheorem{corollary}[theorem]{Corollary}
\newtheorem{assumption}[theorem]{Assumption}
\newtheorem{definition}{Definition}

\newcommand{\gD}{\Delta}
\newcommand{\ga}{\alpha}

\newcommand{\gl}{\lambda}
\newcommand{\eps}{\varepsilon}
\newcommand{\gd}{\delta}

\newcommand{\R}{{\mathbb R}}

\newcommand{\N}{{\mathbb N}}

\newcommand{\gO}{\mathcal X}
\newcommand{\go}{X}

\newcommand{\cL}{{\mathcal L}}
\newcommand{\diag}{\operatorname{diag}}

\newcommand{\cN}{\mathcal N}

\newcommand{\supp}{{\rm supp}}
\newcommand{\cI}{{\mathcal I}}
\newcommand{\cZ}{{\mathcal Z}}

\newcommand{\cF}{{\mathcal F}}
\newcommand{\cE}{{\mathcal E}}
\newcommand{\cM}{{\mathcal M}}
\newcommand{\cY}{{\mathcal Y}}

\newcommand{\cP}{\mathcal{P}}
\newcommand{\cD}{\mathcal D}
\newcommand{\cR}{\mathcal{R}}

\newcommand{\cX}{{\mathcal X}}

\newcommand{\cA}{{\mathcal A}}
\newcommand{\cG}{{\mathcal G}}

\newcommand{\by}{{\bf y}}
\newcommand{\bx}{{\bf x}}
\newcommand{\vertiii}[1]{{\left\vert\kern-0.25ex\left\vert\kern-0.25ex\left\vert #1 
\right\vert\kern-0.25ex\right\vert\kern-0.25ex\right\vert}}

\newcommand{\cC}{\mathcal C}

\usepackage{bbm}

\usepackage{indentfirst} 
\usepackage[labelfont=bf,labelsep=period]{caption}   
\let\comment\undefined

\definechangesauthor[name={HC}, color=red]{HC}

\begin{document}

\setlist{noitemsep}  
\onehalfspacing      
\parskip 3pt

\title{Benign Autoencoders}
\author[$\dagger$]{Semyon Malamud\thanks{Semyon Malamud is at the Swiss Finance Institute, EPFL, and CEPR. Teng Andrea Xu is at EPFL. Antoine Didisheim is at the University of Geneva. Email: semyon.malamud@epfl.ch. We thank Emanuel Abbe and Philipp Schneider for their helpful comments and suggestions. We also acknowledge the financial support of the Swiss National Science Foundation, Grant 100018\_192692. and the Swiss Finance Institute. All errors are our own. This work was supported by a grant from the Swiss National Supercomputing Centre (CSCS) under project ID sm81.}}
\author[$\dagger$]{Teng Andrea Xu}
\author[$\ddagger$]{Antoine Didisheim}

\affil[$\dagger$]{École Polytechnique Fédérale de Lausanne (EPFL) }
\affil[$\ddagger$]{University of Geneva }

\date{\today}


\maketitle
\thispagestyle{empty}

\bigskip

\begin{abstract}
Recent progress in Generative Artificial Intelligence (AI) relies on efficient data representations, often featuring encoder-decoder architectures. We formalize the mathematical problem of finding the optimal encoder-decoder pair and characterize its solution, which we name the ``benign autoencoder'' (BAE). We prove that BAE projects data onto a manifold whose dimension is the {\it optimal compressibility dimension} of the generative problem. 
We highlight surprising connections between BAE and several recent developments in AI, such as conditional GANs, context encoders, stable diffusion, stacked autoencoders, and the learning capabilities of generative models. As an illustration, we show how BAE can find optimal, low-dimensional latent representations that improve the performance of a discriminator under a distribution shift. By compressing ``malignant'' data dimensions, BAE leads to smoother and more stable gradients.  
\end{abstract}

\newpage
\section{Introduction} 

The success of modern generative models relies on neural network architectures for building powerful representations of the data, typically featuring an encoder (responsible for feature learning) and a decoder (responsible for data generation).\footnote{While the original text generation and translation models used encoder-decoder architectures, the recent progress in large language models (LLMs) relies on decoder-only architectures. Understanding the role of encoders for LLMs is an important direction for future research.} Most such encoder-decoder architectures feature a bottleneck, with the latent dimension of the encoder often being much smaller than the dimension of the original data. Extensive experimental evidence suggests that a lower-dimensional latent space improves the quality of generative models by allowing them to generate data based on several key features of the latent representation. For example, this is the case for the variational autoencoder (VAE; \cite{kingma2013auto}, \cite{makhzani2015adversarial}), generative adversarial networks (GANs; \cite{radford2015unsupervised}, \cite{che2016mode}, \cite{peng2018variational}, \cite{goodfellow2020generative}, \cite{donahue2016adversarial}, \cite{dumoulin2016adversarially}, \cite{pathak2016context}), and stable diffusion (\cite{sohl2015deep, ho2020denoising}). The same idea of encoding data into a low-dimensional manifold and then decoding it for discriminative purposes underlies recent successful attempts to build powerful, general  perception models, such as those of \cite{jaegle2021perceiver} and \cite{girdhar2023imagebind}.\footnote{The Perceiver of \cite{jaegle2021perceiver} is designed to handle arbitrary configurations of different modalities (images, audio, and video data) using a single Transformer-based architecture. It introduces a small set of latent units that forms a bottleneck eliminating the quadratic scaling problem of classical Transformers and decoupling the network depth from the input’s size. The authors use a bottleneck of dimension 512 for the image encoding, which is a huge dimensionality reduction, compared to the input dimension of $224 \times 224 = 50176$ pixels.}

The impressive empirical achievements of the models cited above have further widened the gap between their performance and our theoretical understanding thereof. In particular, little is known about the role of bottlenecks and the geometry of the respective latent spaces.\footnote{For some recent progress in the theoretical understanding of GANs, see, \cite{arjovsky2017towards}.} In this paper, we try to bridge this gap. To this end, we  formally define the generative problem of finding the best encoder-decoder architecture. Using novel mathematical techniques combining ideas from optimal transport theory \cite{villani2009optimal} and metric geometry \cite{burago2022course}, we characterize the solution to the optimal encoder-decoder problem, that we name the {\it benign autoencoder (BAE)}. We show that BAE optimally regularizes the generative problem by compressing the ``malignant'' dimensions of the data, thus convexifying the problem through dimensionality reduction.\footnote{It is known that convex problems are well-behaved because they have unique global minima and gradient descent algorithms are guaranteed to converge to these minima. However, BAE exploits a different form of convexity: It makes the average model accuracy depend on the input (training) data in a convex fashion. The dependence on training data is an important ingredient of the theory of adversarial attacks. See, e.g., \cite{goodfellow2014explaining} and \cite{ilyas2019adversarial}. BAE regularizes the dependence on training data by removing ``spikes in the gradient'' and making the gradient map monotone.}  We also characterize the latent dimension of the optimal BAE that we refer to as the {\it compressibility dimension} of the learning problem. 

In addition to providing a theoretical foundation for optimal latent representations in several important generative problems (see, e.g.,  \cite{che2016mode}, \cite{peng2018variational}, \cite{goodfellow2020generative}, 
\cite{pathak2016context}), we test our theory on the distance-regularized GAN and context-encoder settings with the CelebA-HQ dataset \cite{karras2017progressive}. In the Appendix, we also show how to use our results to study { optimal, supervised, denoising autoencoders} with the MNIST \cite{lecun1989handwritten} and FMNIST \cite{xiao2017fashion} datasets. In all experiments, we find evidence of the existence of an optimal latent dimension (much lower than the dimension of the data). In particular, we show that using an encoder with a latent dimension larger than the compressibility dimension either deteriorates generative models' performance or is meaningless. This is due to wasted computational resources, and it does not lead to any performance increase.

In an effort to understand the benefits of encoder-decoder architectures, previous papers used heuristics and intuition to suggest that penalization of the reconstruction error in generative models leads to smoother and more stable gradients (see, e.g., \cite{che2016mode}). This paper vindicates and provides a theoretical formalization for this intuition. Our main theorem implies that BAE convexifies the objective function's dependence on the data. Namely, the objective becomes convex  when restricted to the optimal feature manifold (the low-dimensional manifold on which the auto-encoded data lives). The gradient of a convex function is always regular because it is a monotone map; this monotonicity removes ``spikes'' and makes the gradient more stable.

\section{Background}

Since the onset of GANs, lower-dimensional representations have played a key role in generative AI. For example, in image generation, modern GAN architectures \cite{karras2017progressive, karras2019style, karras2020analyzing, karras2021alias} use a lower-dimensional latent space of $4 \times 4 \times 512$ to generate high-resolution images and videos ($1024 \times 1024 \times 3$). Similar behavior is observed in the latest Diffusion Probabilistic Models \cite{sohl2015deep, ho2020denoising, li2023your}. 

Although GANS achieve state-of-the-art results on various tasks, they are often highly unstable. As \cite{che2016mode} show, this behavior is driven by a special form of a curse of dimensionality that can be solved by training an autoencoder with a small latent dimension. In a similar vein, \cite{peng2018variational} show that introducing an auto-encoder trained with a VAE-type reconstruction loss and a low-dimensional bottleneck significantly improves the performance of GANs, as well as models of imitation learning and inverse reinforcement learning. 

Our paper also relates to the tight connection between generative and discriminative problems, which has been discussed in many papers, starting with the influential work of \cite{hinton2007recognize}: {\it ``To Recognize Shapes, First Learn to Generate Images.''} See also \cite{ng2001discriminative}. Recent evidence suggests that conditional generative models are also good classifiers. See, e.g., \cite{li2023your, brown2020language}. Our results provide additional intuition for this phenomenon and its link to efficient latent representations. 

For discriminative (classification and regression) problems, \cite{tishby2015deep} argue that the success of deep neural networks might be related to their ability to extract efficient representations
of the relevant features of the input layer for predicting the
output label. \cite{tishby2015deep} refer to this phenomenon as the {\it optimal information bottleneck.}\footnote{Recent research shows that the ability of NNs to learn efficient low-dimensional representations is key to their performance. See \cite{NEURIPS2020_a9df2255}.} 
Several subsequent papers have introduced methodologies targeted at creating optimal bottlenecks with a minimal loss of mutual information. See, e.g., \cite{alemi2016deep} \cite{oord2018representation}, \cite{hjelm2018learning}, \cite{achille2018emergence}, \cite{alemi2020variational}. In particular, \cite{alemi2016deep} provide evidence that efficiently trained bottlenecks improve classification accuracy and adversarial robustness; \cite{achille2018emergence} link information bottlenecks to invariance to nuisances, irrelevant features that provide no useful information. 
The mechanism behind the BAE algorithm proposed in this paper is different. The bottleneck created by BAE  does not remove noise or useless features; instead, we prove that some dimensions of data are useful (contain important information) but are {\it malignant for the specific learning algorithm}. BAE identifies those dimensions and erases them.

\section{Preliminaries on Autoencoders}

We start our analysis by introducing a mathematical formalism behind encoder-decoder architectures. 

Let $\tilde \cX$ be a set of {\it messages}, and $\cZ$ the space of encoded messages (henceforth, {\it code space}). Data pre-processing is a map $\cF:\ \tilde\cX\to\cX$, where $\cX\subset\R^L$ is the space of pre-processed messages. E.g., $\cF$ could be a form of data normalization, image resizing, data whitening, or masking (for context encoders).  An encoder is a map $\cE:\cX\to \cZ$, and a decoder is a map $\cD:\ \cZ\to \hat\cX.$ An autoencoder (AE) is the composition of the two: $\cA:\ \cX\to\hat\cX,\ \cA(x)\ =\ \cD(\cE(x)).$ Given a parametric family $\{\cE_\phi\}_{\phi\in \Phi}$ of encoders and a parametric family $\{\cD_\theta\}_{\theta\in \Theta}$ of decoders, the classic optimal encoding problem is to solve $\min_{\theta,\phi}\mathbb{E}[\ell(\cD_\theta(\cE_\phi(\cF(\tilde x))),g(\tilde x))]$ for some loss function $\ell$, where $g:\ \tilde\cX\to \hat\cX$ is a {\it target data transformation.} For example: (i) For a standard autoencoder,\footnote{
One of the most popular algorithms for unsupervised data representation is based on training an autoencoder \citep{rumbert1986learning}: An artificial neural network that learns how to efficiently encode data in a lower-dimensional space with a minimal reconstruction loss. These models play a key role in unsupervised data representation and feature engineering as powerful non-linear dimensionality reduction techniques, see \cite{hinton2006fast}, \cite{hinton2006reducing}, \cite{bengio2007scaling}, \cite{erhan2010does}, \cite{baldi2012autoencoders}, \cite{zemel2013learning}, \cite{makhzani2013k},  \cite{makhzani2015winner}, \cite{achille2018information}, \cite{makhzani2018implicit}, \cite{kenfack2021adversarial},  and \cite{gu2021autoencoder}.} $\tilde\cX=\cX=\hat\cX,$ and both $\cF$ and $g$ are identity maps so that the objective becomes to reconstruct the original data $x=\tilde x$ based on its latent representation $\min_{\theta,\phi}\mathbb{E}[\ell(\cD_\theta(\cE_\phi(x)),x)].$ (ii) In the context encoding problem of images (\cite{pathak2016context}), a part of the data is {\it masked} using a mask indicator $\hat M,$ so that $\cF(\tilde x)=(1-\hat M)\odot \tilde x$ is the partially masked image. At the same time, the optimal encoding-decoding problem is to reconstruct the masked part of the image, $g(\tilde x)=\cM\odot \tilde x,$ based entirely on the partially masked image: The goal is to solve $\min_{\theta,\phi}\mathbb{E}[\ell(\cD_\theta(\cE_\phi((1-\hat M)\odot \tilde x)),\hat M\odot \tilde x].$ (iii) For image-to-image translation (\cite{isola2017image}),\footnote{See also \cite{choi2020stargan} for a related problem of image synthesis.} $\tilde x=(y,x)$ is a pair, and the objective is to morph $x$ into $y$, so that $\cF(\tilde x)=x$ and $g(\tilde x)=y,$ and we minimize $\min_{\theta,\phi}\mathbb{E}[\ell(\cD_\theta(\cE_\phi(x)),y)].$

Given a prior probability distribution $p(dx)$ of $x=\cF(\tilde x)$ on $\cX,$ a probabilistic encoder is a joint probability distribution $p(dx,dz)$ on $\cX\times \cZ$ satisfying
\begin{equation}\label{marginal}
p_x(dx)\ =\  \int_{\cZ} p(dx,dz)\ =\ p(dx)\,.
\end{equation}
In this case, given a value of $x,$ we sample $z$ from the distribution $p(dz|x)$:
\begin{equation}
\cE(x)\ =\ Random(z \sim p(dz|x))\,.    
\end{equation}
Given an encoder (probabilistic or not) $\cE:\cX\to \cZ,$ the optimal decoding problem is to find a map $\cD:\cZ\to \cX$ to minimize the {\it reconstruction loss,} given by 
\begin{equation}
\cD(z)\ =\ \arg\min_{a\in \cX}\,\mathbb{E}[\ell(a,g(\tilde x))|\cE(\cF(\tilde x))=z]\,. 
\end{equation}
In the Appendix, we derive results for generic loss functions. In our experiments, we utilize any of the following loss functions as an additional generative model loss: Mean Squared Error (MSE) or $\ell_2$ loss, pixel loss or $\ell_1$ loss, and binary cross-entropy.\footnote{Please refer to the Appendix for comprehensive details regarding the experiment, including information about the training algorithm, hyperparameters, and model architectures.} However, in the main body of the paper, we focus on $\ell(x,y)=\|x-y\|^2,$ corresponding to the MSE loss function. In this case, the optimal decoder is just the conditional expectation: 
\begin{equation}\label{unbiased}
\cD(z)\ =\ \mathbb{E}[g(\tilde x)|\cE(\cF(\tilde x))=z]\ =\ \int_{\cX} \cG(x)\,p(dx|z)\,,
\end{equation}
where we have defined 
\begin{equation}
\cG(x)\ =\ \mathbb{E}[g(\tilde x)|\cF(\tilde x)=x]\,. 
\end{equation}
That is, an optimal decoder represents the optimal prediction of $g(\tilde x)$ given the encoded (compressed) information in $\cE_\phi(x),\ x=\cF(\tilde x)\,.$ 
The superposition of $\cD$ and $\cE$ is called an {\it autoencoder:} $\cA=\cD\circ \cE:\ \cX\to\hat\cX$ is given by $\cA(x)\ =\ \cD(\cE(x)).$

\section{Optimal Encoders for Generative and Discriminative Problems} 
\label{sec:optimal_encoders}

It is known (see, e.g., \cite{hinton2006reducing}) that autoencoders are able to efficiently encode high-dimensional data into much lower dimensions so that $\cZ\subset \R^\nu$ with $\nu \ll L.$ Formally, this means that there exist parametric families  $\{\cE_\phi\}_{\phi\in \Phi}$ of encoders with $\cZ=\R^\nu$, and parametric families $\{\cD_\theta\}_{\theta\in \Theta}$ of decoders, such that the minimum
\begin{equation}\label{dim-red}
\min_{\theta,\phi}\mathbb{E}[\ell(\cD_\theta(\cE_\phi(x)),x)]
\end{equation}
is relatively small for many real-world datasets with $\cX=\R^L$, even when $\nu$ is much smaller than $L.$ 

The objective of the minimization problem \eqref{dim-red} is to achieve efficient dimensionality reduction. 
By contrast, Generative AI is concerned with a different objective. Given a probabilistic autoencoder $p(dx,dz),$ the objective is to generate objects (e.g., texts or images) by sampling $z$ from the marginal distribution $p_z(dz)=\int_\cX p(dx,dz)$, and then decoding them into $\cD(z)$ to make them pleasing to human perception. The problem of evaluating the quality of generated content is extremely difficult, and there is no consensus about the way of doing it. See, e.g., \cite{borji2022pros} for the discussion of this problem for image generation (most papers on generative models for images still show large samples of generated content directly in their papers to convince human readers that it ``looks good''). For text generation models such as GPT, the problem is even harder, and Reinforcement Learning from Human Feedback (\cite{ouyang2022training}) has been proposed as one potential remedy. 
We conjecture that human beings receiving generated content $\cD(z)$ (be it images or text) evaluate its quality using some function $W$ that is probably combining a ``feeling of common sense'' with some (non-linear) outlier detection. This motivates the following definition. 

\begin{definition}[Optimal Autoencoder] \label{main-def} Let $W:\tilde\cX\to \R $ be a function evaluating the quality of generated content. Then, the solutions to the problems 
\begin{equation}\label{non-par}
\begin{aligned}
&\cL_{PBAE}\ =\ -\sup_{\cZ,\ p(X;z)}\{\mathbb{E}_{p}[W(\cD(z))]:\ \eqref{marginal},\ \eqref{unbiased} \ hold\}\\
&\cL_{BAE}\ =\ -\sup_{\cZ,\ \cE:\ \cX\to\cZ}\{\mathbb{E}[W(\cD(z))]:\ z=\cE(\cF(\tilde x))\ and\ \eqref{unbiased}\ holds\}
\end{aligned}
\end{equation}
(if they exist) are called the optimal probabilistic autoencoder and the optimal autoencoder, respectively. Since an AE is also a PAE, we always have $\cL_{PBAE}\le \cL_{BAE}.$
\end{definition} 


Formally introducing the performance measure $W$ is key to our analysis. In practice, the quality of generative models is typically evaluated using extremely complex, non-linear, and highly non-convex metrics such that the Frech\'et Inception Distance (FID; see \cite{heusel2017gans}) and the Learned Perceptual Image Patch Similarity metric (LPIPS; see \cite{zhang2018unreasonable}). The convexity of $W$ is key to the emergence of optimal latent dimensions, as we explain below. Note also that imposing \eqref{unbiased} is approximately equivalent to penalizing the objective function with an $L_2$-penalty $\gl \mathbb{E}[\|g(\tilde x)-\cD(z)\|^2]$ with a very large penalty coefficient $\gl.$ We will use this observation in our experiments and the examples below. 

\subsection{Examples}\label{sec:examples}

This section shows how many important generative models are tightly linked to the theoretical framework of Definition \ref{main-def}, with model-specific $W$ functions. 

{\bf Distance-Regularized GANs.} \cite{che2016mode} recommend regularizing GANs with a distance penalty. Given a discriminator $D(x),$ the optimal PBAE, $p(dx,dz)$ samples $z$ from the marginal distribution $p_z(dz)=\int_X p(dx,dz)$ of the encoding $z\in \cZ$, builds an {\it unbiased} reconstruction $\cD(z)$ of $x,$ and the objective is to minimize 
\begin{equation}
\min_{p(dx,dz),\ \cD}(\mathbb{E}_{x\sim p(dx)}[\log (D(x))]\ +\ \mathbb{E}_{z\sim p_z(dz)}\mathbb{E}[\log (1-D(\cD(z)))]+\gl_{rec} \mathbb{E}_{(x,z)\sim p(dx,dz)}[\|x-\cD(z)\|^2]),
\end{equation}
which is equivalent to \eqref{non-par} with $W(\hat x)\ =\ \log (1-D(\hat x)).$

{\bf Context Encoders.} Following \cite{pathak2016context}, given a discriminator $D,$ the objective is to minimize a combination of adversarial and reconstruction losses: 
$\gl_{rec}\cL_{rec}\ +\ \gl_{adv}\cL_{adv},$ where $\cL_{adv}\ =\ \mathbb{E}[\log (D(\tilde x))]\ +\ \mathbb{E}[W(\cD(\cE((1-\hat M)\odot \tilde x)))],$ with $W(\hat x)=\log (1-D(\hat x))$ 
and  $\cL_{rec}\ =\ \mathbb{E}[\|\hat M\odot \cD(\cE((1-\hat M)\odot \tilde x))-\hat M\odot \tilde x\|^2]\,.$\footnote{The decoder $\cD$ only reconstructs the masked part, $\hat M\odot \tilde x,$ and keeps the context, 
$(1-\hat M)\odot \tilde x.$}

{\bf Evaluating the Quality of the Generator with a Discriminator.} Given a classifier $D$, with labels $y,$ trained to minimize a distance $\gd(\cdot,\cdot)$ (e.g., the cross-entropy) between $y$ and $D(x),$ the quality of an autoencoder $\cA(x)\ =\ \cD(\cE(x))$ can be evaluated by computing the classification error $\delta(y,D(\cA(x)))$ with $x$ replaced by $\cA(x)$. When $\delta$ is the mean-squared error and $\mathbb{E}[y|x]=f(x),$ we can re-define the decoder $\hat\cD(z)=(\mathbb{E}[f(x)|z],\cD(z))$ (with just one additional dimension) to get 
$\mathbb{E}[\|y-D(\cA(x))\|^2]\ =\ \mathbb{E}[\|y\|^2]\ +\ \mathbb{E}[W(\hat\cD(z))]\,,$
with $W(\hat\cD(z))\ =\ -2\mathbb{E}[f(x)|z]D(\hat\cD(z))+D^2(\hat\cD(z))\,.$
Hence, the problem \eqref{non-par} is equivalent to the problem of compressing the data $x$ to ``help'' the discriminator. Such a framework could also be useful in a situation of a distribution shift, e.g., when $D$ was trained on high-quality data while the new dataset is corrupted by noise. In this case, \eqref{non-par} becomes a problem of finding the {\it optimal, denoising, supervised autoencoder,} whose objective is to denoise the data for better classification accuracy. 

{\bf Conditional Generative and Discriminative Problems.} It is known that there exists a tight link between generative and discriminative problems. See, e.g., \cite{ng2001discriminative} and \cite{hinton2007recognize}.

Many conditional generative models, such as conditional GANs (cGANs), have objectives related to \eqref{non-par} and feature an encoder-decoder architecture. See, e.g., \cite{mirza2014conditional}, \cite{gauthier2014conditional}, \cite{denton2015deep},\cite{isola2017image},\cite{antipov2017face}, \cite{mao2019mode}, cGANs search for a generator $G:(x,z)\to y$ that maps a combination of an observed image $x$ and random noise $z$ into another image $y.$ Given a discriminator $D,$ the objective is to minimize the adversarial objective penalized by the reconstruction loss, 
$\cL_{adv}\ +\ \gl\, \cL_{rec},$
where $\cL_{adv}\ =\ \mathbb{E}[\log (D(y))]\ +\ \mathbb{E}[\log (1-D(G(x,z)))]$
and $\cL_{rec}\ =\ \mathbb{E}[\|y-G(x,z)\|^2]\,.$ 
As we explain above, one can represent $G(x,z)$ as a probabilistic auto-encoder, with full initial data $\tilde x=(y,x)$, the pre-processing map $\cF(\tilde x)=x$ (only $x$ is used for generation), and the target $g(\tilde x)$ in \eqref{unbiased} defined via $g(\tilde x)\ =\ y.$ 

A competitor to conditional GANs is the denoising diffusion model of \cite{ho2020denoising}. Based on this model, \cite{li2023your} introduce a diffusion classifier. As \cite{ho2020denoising} show, efficiently training diffusion models can be done by minimizing the reconstruction loss between the $\eps$ (the noise) and the original data, $x:$ They do it by solving $\min_\theta \mathbb{E}[\|\eps-\eps_\theta(x_t,c)\|^2],$ where $x_t$ is the image, ``diffused'' after several steps of adding noise, and $c$ is the class label. \cite{li2023your} then show how that investigating the whole vector $(\|\eps-\eps_\theta(x_t,c_i)\|^2)_{i=1}^K$ (where $K$ is the number of classes in the dataset) can be used to build classifiers of the form $f((\|\eps-\eps_\theta(x_t,c_i)\|^2)_{i=1}^K)$ for some function $f.$ Given a loss function $\ell$ (e.g., $L_2$-distance or cross-entropy), we end up with an objective 
$\cL=\mathbb{E}[\ell(c,f((\|\eps-\eps_\theta(x_t,c_i)\|^2)_{i=1}^K))],$ 
which directly depends on the auto-encoded noise $\eps_\theta,$ consistent with \eqref{non-par}. Let now $\tilde x=(x_t,\eps)$ and let $(\eps_\theta(x;c_i))_{i=1}^K=\cD(\cE(x))$ be an auto-encoder. Define $W(z;c)=-\mathbb{E}[\ell(c,f((\|\eps-z\|^2)_{i=1}^K))\,|\, z]$ so that, by the law of iterated expectations, our objective is to minimize $\cL=-E[W(\cD(\cE(x_t));c)],$ consistent with \eqref{non-par}.

{\bf Non-Linear Supervised Principal Components Regression and Supervised Principal Manifolds.}  Given a dimension $\nu$, the solution to the optimal linear auto-encoder problem $
\min_{linear\ \cE,\cD}\{\mathbb{E}[\|x-\cD(\cE(x))\|^2]:\ \cD(z)\ =\ \mathbb{E}[x|z]\}$
over linear $\cE,\cD$ with $\cE$ of latent dimension $\nu$ is given by the projection of $x$ on the top $\nu$ principal components (assuming, e.g., a Gaussian distribution). Thus, the general problem of optimal autoencoding could be thought of as a form of non-linear principal component analysis (PCA). Some papers attempt to use PCA for supervised problems. See, for example, \cite{jolliffe1982note}, \cite{bair2006prediction}. Here, we argue that the problem \eqref{non-par} is tightly linked to the general problem of {\it Non-Linear Supervised Principal Components}. Indeed, consider a data sample consisting of $N$ observations $(y;X)\ =\  ((y_i;x_i))_{i=1}^N,$ with $(y_i;x_i)\in \cY\times \cX$ (the training sample) where $\cY$ and $\cX$ are, respectively, label and feature spaces. Given a parametric family of functions $\{f_{\psi}(x)\}_{\psi\in\Psi},$ a learning algorithm is a map from the (train) data sample $(y;X)$ into an estimator $\hat\psi$ of $\psi$, given by a function of the data.\footnote{For example, the linear ridge regression has $f_\psi(\bx)=\psi'\bx,$ and the learning algorithm is given by $\hat\psi(y;X)=(\gl I+X'X)^{-1}X'y,$ where $\gl$ is the ridge penalty.} The algorithm's objective is to minimize the test loss $\cL\ =\ \mathbb{E}[\ell(\by;f_{\hat\psi(y;X)}(\bx))]\,,$
where $(\by;\bx)\in \cY\times \cX$ is a test sample realization, and where the expectation is taken over all possible realizations of $(\by;\bx;y;X),$ drawn from their joint distribution on $\cY\times\cX\times \cY^N\times \cX^N$. Using the law of iterated expectations and defining 
$W(\bx,X)\ =\ -\,{\mathbb{E}[\ell(\by;f_{\hat\psi(y;X)}(\bx))|\bx,X]},$
we can rewrite $\cL\ =\ -\mathbb{E}[W(\bx,X)]\,.$\footnote{Similar problems are commonly referred to as ``feature learning.''  The latter is indeed often associated with dimensionality reduction. For example, \cite{radhakrishnan2022feature} show how a linear dimensionality reduction of the original data (linear feature learning) significantly improves the performance of shallow neural networks. Similarly, \cite{radford2015unsupervised} argue that the quality of the encoded representations (i.e., $\cE(x)$) can be evaluated by their ability to serve as features for a regression or classification problem. Here, we argue that one should use the fully auto-encoded features $\cA(x)=\cD(\cE(x))$ instead of just the encodings, $\cE(x),$ for the following reasons: (1) The generative problem typically has a unique solution (defined as the joint distribution of $x$ and $\cA(x)=\cD(z)$). By contrast, $\cE(x)$ is not uniquely defined. In fact, any injective transformation $\Phi:\cZ\to\cZ$ gives rise to the same auto-encoder; 
(2) A given learning algorithm might be designed to take as input data of a specific dimension (e.g., a CNN is often sensitive to the exact dimensions and the number of channels for images); (3) As we show in the Appendix, the optimal decoder $\cD(z)$ is always bijective. Hence, no information is lost when the encoding $z$ is replaced with $\cD(z).$}


\subsection{Solution}

The problems in \eqref{non-par} are non-parametric: They impose no constraints on the code space $\cZ$ and the encoders $p(x;z)$ ($z=\cE(x)$ in the case of BAE). By \eqref{unbiased}, a trivial autoencoder is $\cE(x)=\cG(x)=\mathbb{E}[g(\tilde x)|x].$ The existence of a non-trivial solution to \eqref{non-par} would imply the existence of a {\it beneficial autoencoder.} As we now explain, the nature of the problems \eqref{non-par} depends crucially on the convexity properties of the function $W$ (the generative score). If $W$ is convex, then the Jensen inequality implies that data compression is always suboptimal: $
\mathbb{E}[W(\cD(z))]\ \underbrace{=}_{\eqref{unbiased}}\ \mathbb{E}[W(\mathbb{E}[\cG(x)|z])]\ \le\ \mathbb{E}[W(\cG(x))]\,.$
By contrast, if $W$ is concave, $W(\mathbb{E}[\cG(x)]\ \ge\ \mathbb{E}[W(\cG(x))]\,,$ and the optimal autoencoder is $\cE(x)=\mathbb{E}[\cG(x)]=\mathbb{E}[g(\tilde x)]=\cD(z),$ with the code space $\cZ=\{\mathbb{E}[g(\tilde x)]\}$ being a single point. Hence, all ``context'' information in $x$ is fully discarded.  In general, it is known that most of the modern statistical (machine) learning problems are neither convex nor concave,\footnote{This non-convexity usually refers to the dependence on the neural network weights (\cite{goodfellow2014qualitatively}, \cite{li2018visualizing}). Still, the non-convexity concerning the training data is also ubiquitous in most machine learning problems.} Hence, the optimal code space $\cZ$ is neither $\cX$ nor a single point. Finding the optimal code space $\cZ$ without any information about its structure seems daunting. In this paper, we use optimal transport theory and metric geometry techniques to characterize some natural regularity properties of $\cZ$ and prove that $\cD(\cZ)$ is always a Lipschitz manifold of a computable dimension. 

For any symmetric matrix $A,$ let $\nu_+(A)$ be the number of non-negative eigenvalues of $A$, and, similarly, $\nu_-(A)$ be the number of non-positive eigenvalues of $A$. We use $D_XW(X)$ to denote the gradient of $W$ and $D_{XX}W(X)$ the Hessian of $W.$ The following is the main theoretical result of this paper. 

\begin{theorem}\label{main-gen-char0} Suppose that $\cX\subset\R^L$ is an open subset and  $p(x)$ is absolutely continuous with respect to the Lebesgue measure on $\R^L$. Suppose that $W(X)$ is smooth and either $\gO$ is bounded or the derivatives of $W(X)$ decay sufficiently fast at infinity.\footnote{See the Appendix for precise conditions.} 
Then, there always exists an optimal probabilistic autoencoder $(p(dx;z),\cZ,\cD)$. Furthermore: (i) For any such autoencoder and any open subset $Q\subset \cX$, the autoencoded space $\cD(\cZ)\cap Q$ is a subset of a $\nu$-dimensional Lipschitz manifold in $\R^L$, with $\nu\le \sup_{X\in Q} \nu_+(D_{XX}W(X));$ (ii) $W$ is convex when restricted onto the $\cD(\cZ);$ (iii) If $W$ is concave along rays for large $\|X\|$, then $\cA(\cX)\ =\ \cD(Z)$ is compact. 
\end{theorem} 

We call $\nu$ (the latent dimension of the optimal autoencoder from Theorem \ref{main-gen-char0}) the {\it compressibility dimension} of the generative problem of Definition \ref{main-def}. Directions of information along which $W(X)$ is concave are ``malign,'' hurt (or are irrelevant to) model performance  and are compressed by the autoencoder. They are orthogonal to the autoencoder space $\cA(\cX)$ ({\it the optimal feature manifold}). Directions along which $W(X)$ is convex are benign. These are directions tangent to the autoencoder space $\cA(\cX),$ and  $W$ is locally convex along the tangent space of the Lipschitz manifold $\cA(\cX).$

The compactness of $\cA(\cX)$ has interesting implications for popular feature processing algorithms such as batch normalization. Batch normalization is known to improve the performance of learning algorithms. By construction, batch normalization does two things: (1) it performs dimensionality reduction (just like autoencoders), projecting data on a sphere;\footnote{E.g., the mapping $(x_1,x_2)\to (x_1,x_2)/\sqrt{x_1^2+x_2^2}$ maps $\R^2$ onto the unit circle.} and (2) compactifies data (because unit sphere is compact). Theorem \ref{main-gen-char} implies that compactification of data (such as batch normalization) is optimal when the sensitivity to extreme outliers is small. 

As an illustration, suppose that the humans evaluate the ``creativity'' of the model by its ability to generate data $a=\cA(x)$ that is not too different from the original data $x$, penalizing outliers (e.g., image generation would like to generate cats that look ``normal''). Formally, we assume that the objective function is to minimize $\mathbb{E}[\|a-x\|^2+\gamma \|a\|^4]\,,$
where $\gamma$ measures the sensitivity to outliers. By direct calculation, using that $\mathbb{E}[x|a]=a$ by \eqref{unbiased}, we get $\mathbb{E}[\|a-x\|^2+\gamma \|a\|^4]\ =\ \mathbb{E}[\|x\|^2]+\mathbb{E}[-\|a\|^2+\gamma\|a\|^4],$ implying that $W(a)=\|a\|^2-\gamma\|a\|^4.$ In this case, as we show in the Appendix, when $x\in \R^L$ is standard normal, the optimal encoded space $\cA(\cX)$ is a $(L-1)$-dimensional sphere of some data-dependent radius, and, hence, the optimal encoder is just a simple batch normalization map.\footnote{
This example has important implications for learning representations. In machine learning, using encoded representations (i.e., $\cE(x)$) as inputs into discriminative problems is common practice. See, e.g., \cite{radford2015unsupervised}. However, if the manifold on which $\cE(x)$ lives is compact, no natural global coordinate system may allow us to parametrize it without introducing artificial boundary effects. One needs to embed the non-Euclidean features in $\cE(x)$ into a Euclidean space. We argue here that the fully auto-encoded features $\cA(x)=\cD(\cE(x))$ represent the most natural embedding of $\cE(x)$ into $\R^L$ and, hence, one should directly use $\cA(x)$ instead of $\cE(x).$ E.g., in the above example, $\cE$ lives on a sphere, which is a $(L-1)$-dimensional manifold, but there is no natural way to parametrize it by $(L-1)$ coordinates globally. Instead, it can be embedded into $\R^L$, and this embedding is more convenient to work with.} In the case when $W$ satisfies some strong regularity conditions, it is possible to get more information about the structure of the optimal autoencoder. 

\begin{theorem}\label{main-gen-char} 
Suppose for simplicity that $\cG(x):\ \R^L\to\R^L$ is a linear, non-degenerate map, and that $\cX$ is convex, and $D_XW(X):\cX\to\R^L$ is such that the inverse mapping $D_XW^{-1}(X)$ has a finite number of continuous branches in $\cX$. Then, $\cL_{PBAE}= \cL_{BAE}$ in \eqref{non-par}, and there exists an optimal deterministic autoencoder $\cA(X)=\cD(\cE(X))$ such that 
    (1) for every $X,$ the pre-image $(D_XW\circ \cA)^{-1}(X)$ is a convex set that almost surely has dimension less than or equal to  $\nu_-(D_{XX}W(X))\,;$
    (2) The map $D_XW(\cA(X))$ is monotone increasing;\footnote{A map $F:\R^L\to\R^L$ is monotone increasing in $(F(x)-F(y))^\top (x-y)\ge 0$ for all $x,y.$}
    (3) If $D_XW(X)$ is injective, then $\cA(X)$ is a projection: $\cA(\cA(X))=\cA(X)$ Lebesgue-almost surely;
    (4) If $D_XW(X)$ is injective and $D_{XX}W(X)$ is non-degenerate, then the encoded space $\cA(\cX)$ is a Lipschitz manifold of dimension exactly $\nu_+(D_{XX}W(X)),$ while the pre-image $\cA^{-1}(X)$ almost surely has dimension exactly $\nu_-(D_{XX}W(X)).$
\end{theorem}

The most relevant part of Theorem \ref{main-gen-char} concerns the existence and properties of the optimal deterministic autoencoder. It shows that the autoencoder decomposes the data space\footnote{While $\cZ$ and $\cE$ are not uniquely defined (e.g., any injective transformation of $\cE$ leads the same information and, hence, the same autoencoder), the image $\cA(\cX)$ is defined uniquely and is a Lipschitz manifold.} into a {\it fiber bundle}, with the base space given by the image $\cA(X)$ of the autoencoder (of dimension $\nu_+(D_{XX}W(X))$, and fibers $\cA^{-1}(X)$ being convex subsets of dimension $\nu_-(D_{XX}W(X)).$ The fact that $\cA$ is a projection formalizes a form of optimality: Once $X$ is autoencoded to $\cA(X),$ encoding it again does not bring additional gains and hence does not modify it. This theoretical finding suggests an iterative algorithm for finding the optimal autoencoder: For any given (suboptimal) autoencoder $\cA$, defining the {\it stacked autoencoder} $\cA^{(k)}=\underbrace{\cA\circ \cdots\circ \cA}_{k\ times},$ we get that the limit (if it exists) $\cA_\infty\ =\ \lim_{k\to\infty} \cA^{(k)}$ satisfies $\cA_\infty=\cA_\infty\circ \cA_\infty.$ Thus, potentially, one might come closer to the optimum by stacking several sub-optimal autoencoders together. This observation might be related to the popular practice of stacking multiple autoencoders. See, e.g., \cite{vincent2010stacked}. Note finally that, in applications to generative problems, if $W(x)=\mathbb{E}_\eps[w(\mu(x)+\sigma(x)\eps)]$ for some function $w,$ then a deterministic autoencoder $\cA(x)$ leads to a generative model $G(x,\eps)\ =\ \mu(\cA(x))+\sigma(\cA(x))\eps.$

\section{Experiments}
\label{sec:numerical_evidences}
\begin{figure}
    \centering
    \makebox[\textwidth]{\includegraphics[width=1.2\textwidth]{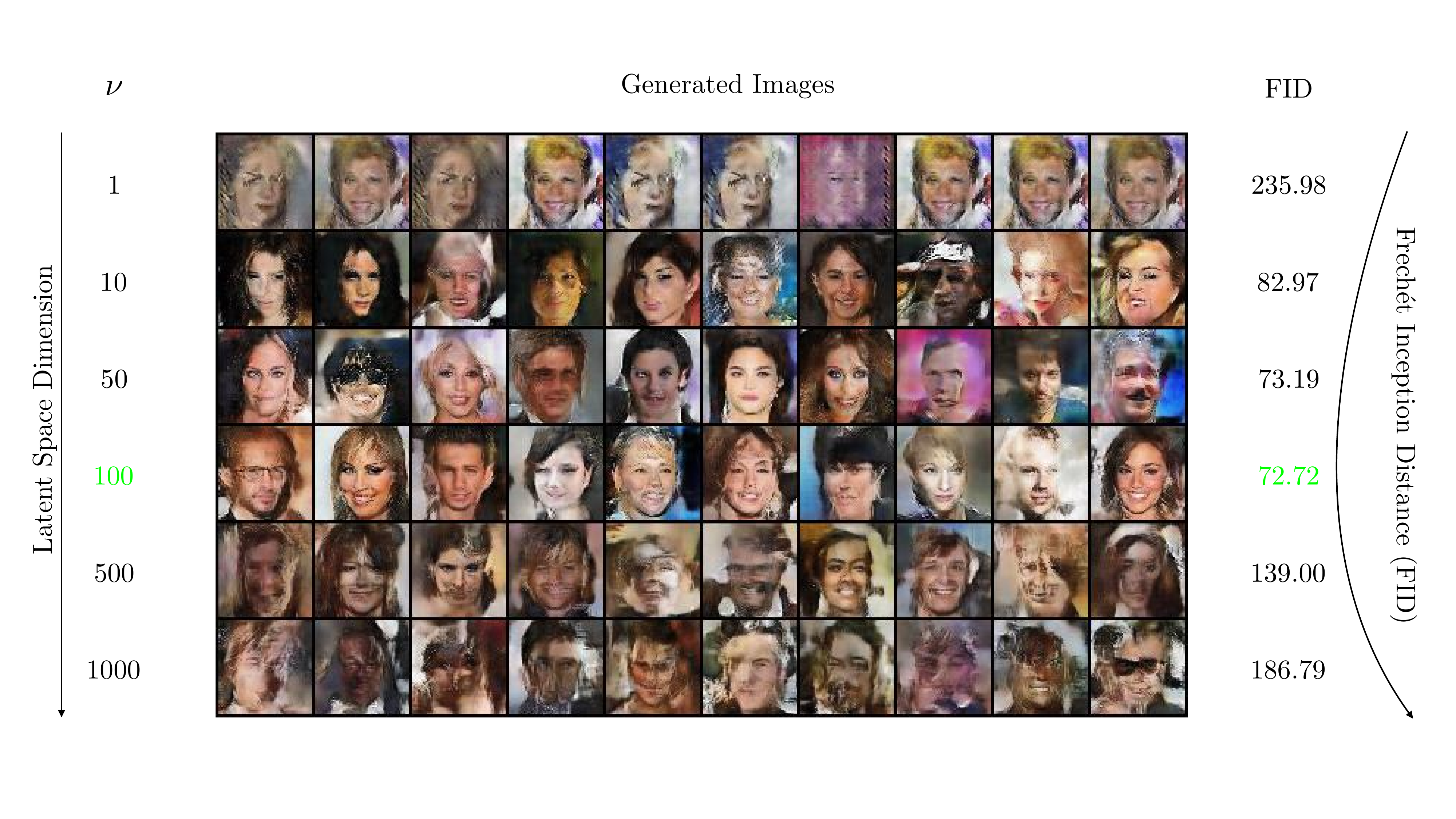}}
    \caption{Distance-regularized GAN on CelebA-HQ. FID score with varying latent space dimension $\nu$, while keeping constant the discriminator $D$ and the non-bottleneck layers of the decoder $\cD_{\theta}$ and encoder $\cE_{\phi}$ architectures. Images were resized to $64 \times 64.$}
    \label{fig:distance_gan}
\end{figure}
The key testable implication of our theory is the existence of an optimal bottleneck (latent) dimension for the encoder: With too few latent dimensions, the model is not rich enough; with too many, it encodes malignant dimensions that hurt (or simply do not improve) performance: The encoded information ``saturates.'' In this section, we test this prediction through experimentation with various generative and discriminative problems outlined in Section~\ref{sec:examples}, utilizing a variety of datasets. Experiments were conducted on either a single NVIDIA RTX 4090 24GB GPU or a single NVIDIA TITAN X 12GB GPU. For comprehensive details about the architectures, algorithms, and training settings, we refer the reader to the Appendix.\footnote{The repository is: \url{https://github.com/tengandreaxu/benign-autoencoders}.}

\subsection{Distance-Regularized GANs}
\label{sec:exp_distance_gan}

As many modes of the true data-generating distribution are missed in the generated samples with standard GANs \cite{che2016mode, mao2019mode}, the literature has proposed distance-regularized GAN to solve the problem. This problem is ideal for testing our theory, as our results imply that, with distance regularization, an optimal latent dimension of the encoder that maximizes GAN performance, exists. 

We train a discriminator $D$ in the normal GAN setting and then optimize $\cD_{\theta}$ and $\cE_{\phi}$ with respect to the distance regularized adversarial loss $\mathcal{L_{BAE}} =\ \log(D(\cD_{\theta}(\cE_{\phi}(x))))\ -\ \lVert x - \cD_{\theta}(◦\cE_{\phi}(x)) \rVert^2$. We remind that this training is equivalent to \ref{non-par}, with $W(x) = (1-D(\cD_{\theta}(\cE_{\phi}(x))).$

To demonstrate the existence of an optimal $\nu$, we train the auto-encoder while varying $\nu \in \{1,10,50,100,500,1000\},$ maintaining constant architectures for $D$, and the non-bottleneck layers of $\cE_{\phi}$ and $\cD_{\theta}$. Our experiment, conducted on the CelebA-HQ dataset \cite{karras2017progressive, CelebAMask-HQ}, assesses the quality of the generative model using the FID score. Figure~\ref{fig:distance_gan} indicates a striking agreement with our theory, with the optimal latent dimension $\nu$ being about 100. Conversely, when the latent dimension $\nu$ becomes larger, the performance of the generative model deteriorates. 

\subsection{Context-Encoders and In-Painting}
\begin{figure}
    \centering
    \makebox[\textwidth]{\includegraphics[width=1.2\textwidth]{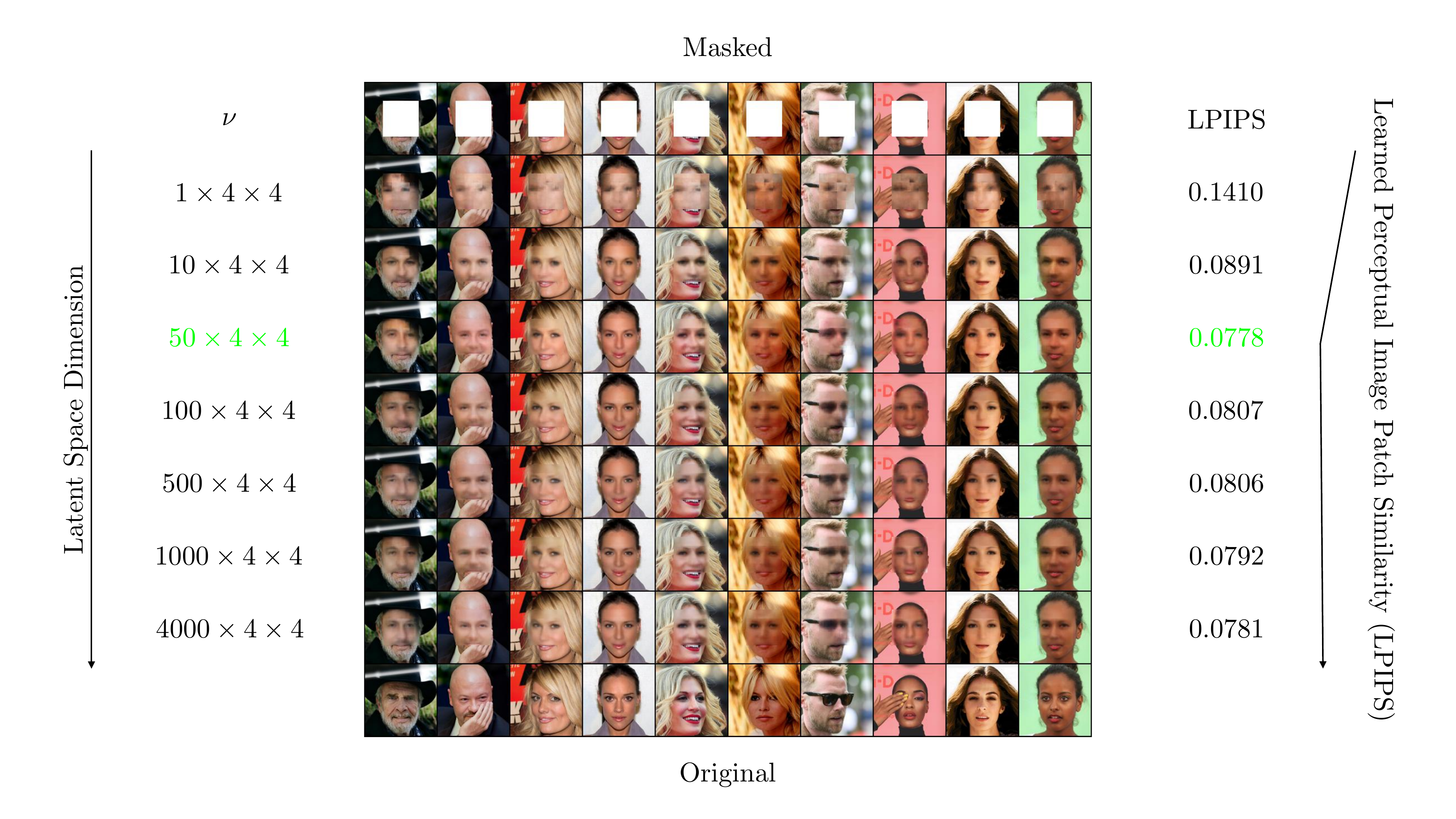}}
    \caption{Context-Encoder on CelebA-HQ. LPIPS score with varying latent space dimension $\nu$, while keeping constant the discriminator $D$ and the non-bottleneck layers of the decoder $\cD_{\theta}$ and encoder $\cE_{\phi}$ architectures. Images were resized to $128 \times 128.$ The mask area is $64 \times 64$. }
    \label{fig:context}
\end{figure}

Introduced by \cite{pathak2016context}, context encoders are generative models trained to fill (in-paint) the contents of an arbitrary image region conditioned on its surroundings. These generative models are penalized by a distance loss to maximize the model's generalization and, thus, fit our theory for the same reason as in the distance-regularized example.

Again, we train a discriminator $D$ to distinguish between fake, denoted as $\hat M\odot \cD_\theta(\cE_\phi((1-\hat M)\odot \tilde x)$, and real, denoted as $M\odot \tilde x$, content images. At the same time, $\cD_{\theta}$ and $\cE_{\phi}$ try to minimize $\cL_{rec} =\lVert 1 - D (M\odot \cD_\theta(\cE_\phi((1-\hat M)\odot \tilde x)) \rVert^2 +\ \|\hat M\odot \cD_\theta(\cE_\phi((1-\hat M)\odot \tilde x))-\hat M\odot \tilde x\|.$ Similar to Section~\ref{sec:exp_distance_gan}, this is equivalent to \ref{non-par}, with $W(x) = (1-D(\hat M\odot \cD_\theta(\cE_\phi((1-\hat M)\odot \tilde x)))).$

As before, we vary $\nu$ on a grid and keep constant architectures for $D$, $\cD_{\theta}$, and the non-bottleneck layers of $\cE_{\phi}$. The experiment is run again on the CelebA-HQ dataset, but we use the more suitable LPIPS score \cite{zhang2018unreasonable} for quality assessment. We train $D, \cD_{\theta}, \text{and } \cE_{\phi}$ on the first 26,000 samples and ``in-paint'' the remaining 4,000 out of sample. The LPIPS score is computed using these 4,000 in-painted out-of-sample images and the ground truth. Figure~\ref{fig:context} shows how after reaching the optimal $\nu$, around $50 \times 4 \times 4$ (the {\it compressibility dimension}), the generative model performance is not increasing but rather stays in a plateau.

\section{Discussion}

{\bf Limitations.} We gave many examples of known generative problems (distance-regularized GANs, context encoders, etc.) that can be reformulated in terms of  finding the optimal {\em benign autoencoder}. The key prediction of our theory is the existence of an optimal latent dimension $\nu.$ In our experiments, we find $\nu$ through a grid search. However, our theory implies that $\nu$ can be computed as the number of positive eigenvalues of the Hessian of $W.$ Developing efficient algorithms for computing $\nu$ is an important direction for future research.

The image generation performance that we presented lags behind the state-of-the-art. The primary goal of our experiment was to empirically validate our main Theorem~\ref{main-gen-char0}, thereby shedding light on the {\em optimal information bottleneck} phenomenon.

{\bf Conclusion.} Efficient data representations are crucial for many recent breakthroughs in machine learning. However, as more and more innovations rest on representation learning, our fundamental understanding of this pivotal concept is lagging behind. In this work, we try to bridge that gap through a novel theory. We prove that (under minimal regularity conditions) every distance-regularized generative problem admits an optimal encoder-decoder architecture with encoded features that live on a surface (manifold) on the {\it optimal compressibility dimension} that we characterize. Consistent with out theory, best model performance is achieved by autoencoders with low latent dimensions in every experiment we run.

\clearpage

\appendix

\section{Introduction}

This document contains relevant background, technical conditions, and theoretical proofs of all results from the main text. Furthermore, Section~\ref{sec:empirics} provides elaborate information regarding the architectures, algorithms, and training settings outlined in the main paper.

\section{Optimal Autoencoders} \label{sec:moment}

We assume that the train data $X$ takes values in a (potentially unbounded) subset $\gO\subset \R^L$ for some $L\in \mathbb N.$ We assume that the distribution $p(dx)$ of $X$ has a density with respect to the Lebesgue measure on $\gO.$ The problem of findings of the Optimal (benign) Autoencoder can be stated as follows. 

\begin{definition}[Optimal Autoencoder]\label{def1}   We have 
\begin{itemize}
\item A probabilistic autoencoder is a triple $(p(X;z), \cD,\cZ),$ where $\cZ$ is a Borel space, $p(X;z)$ is a probability distribution on $\gO\times \cZ$ such that $p(dX;\cZ)=p(dX)$ (that is, the marginal of the joint distribution $p(X;z)$ coincides with the actual distribution of the train data, $p(dX));$ the conditional distribution $p(dX|z)$ is regular; and $\cD:\ \cZ\to \R^L$ is a Borel map defined by 
\begin{equation}
\cD(z)\ =\ \mathbb{E}[X|z]\ =\ \int X p(dX|z)\,.
\end{equation}

\item A deterministic autoencoder is a triple $(\cE,\cD,\cZ),$ where $\cZ$ is a Borel space,  $\cE:\ \gO\to \cZ$ is a Borel map, and 
\begin{equation}
\cD(z)\ =\ \mathbb{E}[X|\cE(X)=z]
\end{equation}
The corresponding conditional distribution $p(dX|z)$ is given by 
\begin{equation}
p(\cX|z\in Z)\ =\ p(\cE^{-1}(Z))\,
\end{equation}
for any Borel $Z\in \cZ.$ In this case, we denote $\cA(X)=\cD(\cE(X))\,.$

\item The code space of an autoencoder is the support of the marginal distribution $p(dz)=p(d\gO;z)$ of $z$.

\item The solutions to the problems 
\begin{equation}\label{non-par1}
\begin{aligned}
&W^*\ =\ \sup_{\cZ,\ p(X;z)}\{\mathbb{E}_{p}[W(\mathbb{E}[X|z])]:\ p(dX;\cZ)=p(dX)\}\\
&W_*\ =\ \sup_{\cZ,\ \cE:\ \cX\to\cZ}\{\mathbb{E}[W(\mathbb{E}[X|\cE(X)=z])]\}
\end{aligned}
\end{equation}
(if they exist) are called the optimal probabilistic benign autoencoder (PBAE) and the optimal benign autoencoder (BAE), respectively. 
\end{itemize}
\end{definition}
The, we report here the main theorem:

\begin{theorem}\label{main-gen-char-app} Suppose that $\cX\subset\R^L$ is an open subset and  $p(x)$ is absolutely continuous with respect to the Lebesgue measure on $\R^L$. Suppose that $W(X)$ is smooth and either $\gO$ is bounded or the derivatives of $W(X)$ decay sufficiently fast at infinity.\footnote{See Assumption \ref{integrability} for precise conditions.} 
Then there always exists an optimal probabilistic autoencoder $(p(dx;z),\cZ,\cD)$. For any such autoencoder and any open subset $Q\subset \cX$, the autoencoded space $\cD(\cZ)\cap Q$ is a subset of a $\nu$-dimensional Lipschitz manifold in $\R^L$, with $\nu\le \sup_{X\in Q} \nu_+(D_{XX}W(X)).$ 

Suppose now that $\cX$ is convex, and $D_XW(X):\cX\to\R^L$ is such that the inverse mapping $D_XW^{-1}(X)$ has a finite number of continuous branches in $\cX$. Then, $\cL_{PBAE}= \cL_{BAE}$ and there exists an optimal deterministic autoencoder $\cA(X)=\cD(\cE(X))$ such that 
\begin{itemize}
    \item[(1)] for every $X,$ the pre-image $(D_XW\circ \cA)^{-1}(X)$ is a convex set that almost surely has dimension less than or equal to  $\nu_-(D_{XX}W(X))\,.$
    
    \item[(2)] The map $D_XW(\cA(X))$ is monotone increasing.\footnote{A map $F:\R^L\to\R^L$ is monotone increasing in $(F(x)-F(y))^\top (x-y)\ge 0$ for all $x,y.$}
    
    \item[(3)] If $D_XW(X)$ is injective, then $\cA(X)$ is a projection: $\cA(\cA(X))=\cA(X)$ Lebesgue-almost surely.
    
    \item[(4)] If $D_XW(X)$ is injective and $D_{XX}W(X)$ is non-degenerate, then the encoded space $\cA(\cX)$ is a Lipschitz manifold of dimension exactly $\nu_+(D_{XX}W(X)),$ while the pre-image $\cA^{-1}(X)$ almost surely has dimension exactly $\nu_-(D_{XX}W(X)).$
    
    \item[(5)] If $W$ is concave along rays for large $X$ (see Definition \ref{conc-rays}), then $\cA(\cX)$ is compact. 
\end{itemize}
\end{theorem}

The proof of Theorem \ref{main-gen-char-app} is non-trivial and combines techniques from optimal transportation theory and metric geometry. The rest of the Supplementary material is devoted to the formal proof of this theorem. It is organized as follows: 

\begin{itemize}
    \item In Section \ref{relax}, we introduce a relaxation of the problem. 
    
    \item In Section \ref{sec-discrete}, we introduce the problem of finding the optimal autoencoder with a finite code space $\cZ$ (effectively, such an autoencoder is a discrete classifier) and prove that an optimal discrete deterministic autoencoder always exists using the theory of real analytic functions (Theorems \ref{mainth1} and \ref{mainth1-analytic}).
    
    \item In Section \ref{sec-lim}, we take the limit of finite autoencoders as the code space size increases to infinity, and and establish the existence of a probabilistic autoencoder, as well as sufficient conditions for the existence of an optimal deterministic autoencoder. 
    
    \item In Theorem \ref{thm-unif}, we prove that the image of the optimal code space, $\cD(\cZ),$ is a subset lower-dimensional Lipschitz manifold. 
    
    \item In Proposition  \ref{regularity} we derive sufficient conditions allowing us to compute the exact dimension of the manifold and show that it exactly equals the number of positive eigenvalues of the Hessian of $W$. 
    
    \item In Proposition \ref{main convexity}, we prove the convexity of pre-images $\cA^{-1}(X).$ 
    
    \item In Corollary \ref{dimension-pool}, we prove that pools are convex subsets of dimension equal to the number of negative eigenvalues. 
    
    \item Finally, Proposition \ref{conv-bound} establishes sufficient conditions for the compactness of $\cA(\cZ).$
\end{itemize}

\section{Relaxation of the Problem} \label{relax}

The problem of Definition \ref{def1} is (indirectly) related to the classic Monge problem of optimal transport. 
It is known that this problem is difficult to tackle directly and one usually studies the Kantorovich relaxation of the problem, and only then proves that the solution of the relaxed problem is given by a Monge map (under technical conditions). See, \cite{villani2009optimal}, \cite{galichon2018optimal}, \cite{villani2021topics}. In this paper, we follow a similar approach. We consider a relaxation (and a significant generalization) of the basic problem from Definition \ref{def1} and then prove (under technical conditions) that the solution to the relaxed problem is in fact given by an Optimal Autoencoder according to Definition \ref{def1}. 

\subsection{Generalized Unbiasedness Constraints}\label{generalized_loss}

The unbiasedness constraint 
\begin{equation}\label{unb1}
\cD(z)\ =\ \mathbb{E}[\go|z]
\end{equation}
can be also formulated as $a=\cD(z)$ being the unique $z-$measurable solution to 
\[
\mathbb{E}[G(a,\go)|z]\ =\ 0\,,
\]
where $G(a,\go)\ =\ a-X\,.$ Here, we use this intuition to introduce generalized unbiasedness constraints. 

Let $G:\R^L\times \gO\to\R^M$ be a Borel-measurable map. 
We will need the following technical condition 

\begin{assumption}\label{ac} The map $G$ satisfies the following conditions:  
	\begin{itemize}
		\item $G$ is continuously differentiable in $a$. 
		
		\item $G$ is uniformly monotone in $a$ for each $\go$ so that $\eps\|z\|^2\ \le -z^\top D_aG(a,\go)z\le\eps^{-1}\|z\|^2$ for some $\eps>0$ and all $z\in \R^L;$\footnote{Strict monotonicity is important here. Without it, there could be multiple equilibria.} 
		
		\item the unique solution $a_*(\go)$ to $G(a_*(\go),\go)=0$ is square integrable: $\mathbb{E}[\|a_*(\go)\|^2]<\infty.$
	\end{itemize}
\end{assumption}

Assumption \ref{ac} implies that the following is true:

\begin{lemma} \label{existence} For any posterior distribution $\mu$ of $X$, there exists a unique action $a=a(\mu)$ to 
	\begin{equation}\label{sys-3}
		\int G(a(\mu),\go)d\mu(\go)\ =\ 0
	\end{equation}
and such that $\|a(\mu)\|^2\ \le\ \kappa \int_\gO \|a_*(\go)\|^2d\mu(\go)$ for some universal $\kappa>0.$
\end{lemma}

\begin{proof}[Proof of Lemma \ref{existence}] First, by uniform monotonicity, the map 
	\[
	a\to\ F(a)= a+\gd \mathbb{E}[G(a,\go)|k]
	\]
	is a contraction for sufficiently small $\gd.$ Indeed, by monotonicity, 
	\[
	\|F(a_1)-F(a_2)\|^2\ \le\ \|a_1-a_2\|^2-2\eps\gd\|a_1-a_2\|^2\ +\ \gd^2\eps^{-2} \|a_1-a_2\|^2\,. 
	\]
	As a result, there exists a unique equilibrium by the Banach fixed point theorem. Then, with $a=a(k),$
	\begin{equation}
		\begin{aligned}
			&\mathbb{E}[(a_*(\go)-a)^\top\,G(a,\go)|k]\ =\ \mathbb{E}[(a_*(\go)-a)^\top\,(G(a,\go)-G(a_*(\go),\go))|k]\\ 
			&\ge\ \eps\,\mathbb{E}[\|a_*(\go)-a\|^2|k]\ \ge\ \eps(\mathbb{E}[\|a_*(\go)\|^2+2\|a\|\|a_*(\go)\| |k]\ +\ \|a\|^2)\,. 
		\end{aligned}
	\end{equation}
	At the same time, 
	\begin{equation}
		\begin{aligned}
			&\mathbb{E}[(a_*(\go)-a)^\top\,G(a,\go)|k]\ =\ \mathbb{E}[a_*(\go)^\top\,G(a,\go)|k]\ \le\ \eps^{-1} \mathbb{E}[\|a_*(\go)\|\,\|a-a_*(\go)\||k]\\
			& =\ \eps^{-1}(\mathbb{E}[\|a_*(\go)\|^2|k]+\|a\|\mathbb{E}[\|a_*(\go)\||k])
		\end{aligned}
	\end{equation}
	and the claim follows. 
\end{proof}

In applications to real data, the most important case for us corresponds to $G(a,\go)=a-g(\go)$ for some Borel map $g$ representing manually engineered input features of the problem. However, one could envision other types of maps $G$ representing different forms of regularizations of the learning problem. For example, while $a=\mathbb{E}[g(X)|z]$ solves $a=\arg\min_a \mathbb{E}[(a-g(X))^2|z],$ one could assume that $a$ solves a different problem. For example, $a=\arg\min_a \mathbb{E}[\ell(a,g(X))]$ for some other loss function $\ell.$

\subsection{Generalized Optimal Autoencoder Formulation} 

Everywhere in the sequel, we use $\gD(\gO)$ to denote the set of Borel probability distributions on a Borel space $\gO.$ 
Similarly, we can define $\gD(\gD(\gO))$ to be the space of Borel probability distributions on $\gD(\gO).$   
We start with an observation that specifying an encoding $z$ and a joint probability distribution $p(X;z)$ is equivalent to specifying a distribution $\tau\in \gD(\gD(\gO))$ that can be defined via\footnote{For example, if the code space $\cZ$ consists of only three points, $s_1,s_2,s_3,$ with probabilities $p_1,p_2,p_3,$ let $\mu_i\in \gD(\gO)$ be the posterior distribution of $\go$ conditional on $s_i.$ This is equivalent to specifying a distribution $\tau$ on $\gD(\gO),$ with a support of three points, $\mu_1,\mu_2,\mu_3\in \gD(\gO),$ with $\mu_i$ occurring with probability $p_i.$ Hence, $\tau$ is a distribution on posterior distributions, $\tau\in \gD(\gD(\gO)).$ }
\begin{equation}
\tau(\cX)\ =\ p_z(\{z:\ p(dX|z)\in \cX\})\,. 
\end{equation}
This formulation of the problem is equivalent to a so-called optimal Bayesian persuasion problem intorduced in \cite{KamGenz2011} (see also \cite{RayoSegal2010}, \cite{DworzakKolotilin2019}, and \cite{BergMor2017}, \cite{kamenica2021bayesian} for excellent overviews). 

\begin{definition}\label{dfn1}
Let $G:\R^L\times \gO\to\R^L$ be a Borel-measurable map and let  $W(a,\go)\in C^3(\R^L\times \gO)$.
Let also 
\[
\bar W(\mu)\ =\ \int_\gO W(a(\mu),\go)d\mu(\go)\,,
\]
with  $a(\mu)$ defined in \eqref{sys-3}. 
The optimal autoencoder problem is to maximize 
\[
\int_{\gD(\gO)} \bar W(\mu)d\tau(\mu)
\]
over all $\tau\in\Delta(\Delta(\gO))$ satisfying 
\begin{equation}\label{bayes-rule}
 \int_{\gD(\gO)} \mu d\tau(\mu)\ =\ p\,. 
\end{equation}
We denote the value of this problem by $V(p)\,.$ A solution $\tau$ to this problem is called an optimal autoencoder. Choosing the code space $\cZ=\gD(\gO),$ and defining the conditional distribution $p(dX|\mu)=\mu,$ we get that \eqref{bayes-rule} is equivalent to $p(dX;\cZ)=p(dX):$
\begin{equation}
p(dX;\cZ)\ =\ \int_{\gD(\gO)} p(dX|\mu) d\tau(\mu)\ =\ \int_{\gD(\gO)} \mu d\tau(\mu)\ =\ p\,. 
\end{equation}
\end{definition}

Definition \ref{dfn1} shows that the relaxed problem is in fact equivalent to the problem of selecting an optimal autoencoder, allowing for an extremely rich code space: The space of all possible probability measures. As we show below, in fact, such a rich space is not necessary and the optimal code space (corresponding to an optimal autoencoder) can always be chosen as a subset of $\R^L.$

We will also need the following technical assumption. 

\begin{assumption}\label{integrability} The  function $W$ is jointly continuous in $(a,\go)$ and is continuously differentiable with respect to $a.$ Furthermore, there exists a function $\psi:\gO\to \R_+$ such that $\psi(\go)\ge \|g(\go)\|^2$ and the set $\{\go:\psi(\go)\le A\}$ is compact for all $A>0,$ and a convex, increasing function $f\ge 1$ such that  $|W(a,\go)|+\|D_aW(a,\go)\|\ +\ \|D_aW(a,\go)\|^2\ \le\ \psi(\go) f(\|a\|^2)$ and $\mathbb{E}[\psi^2(\go) f(\psi(\go))]<\infty.$
\end{assumption}

\begin{theorem}\label{mainth-limit} An optimal autoencoder in the sense of Definition \ref{dfn1} always exists. 

Suppose now that (1) $W(a)$ only depends on $a$ and $G(a,\go)=a-g(\go)$; (2) $L=M$ and $g:\R^L\to \R^L$  is injective and bi-Lipschitz; (2) $g(\gO)$ is convex; and (3) the inverse $D_aW^{-1}$ has a finite number of continuous branches.\footnote{That is, there exist continuous maps $G_i,\ i\in \N$ such that $D_aW(a)=b$ if and only if $a=G_i(b)$ for some $i.$ For example, this is the case when $W(a)$ is a generic real analytic function.}  Then, there exists a deterministic optimal autoencoder $\cA(\go)$ and for this policy the map $D_aW(\cA(g^{-1}(x)))$ is monotone has has convex level sets. 
\end{theorem}

Technical conditions ensuring the existence of a deterministic optimal autoencoder in Theorem \eqref{mainth-limit} are important. Without imposing them, the existence of a deterministic optimal autoencoder cannot be guaranteed, and only probabilistic autoencoders  exist.  Such randomized autoencoders are analogous to Kantorovich relaxations of optimal Monge maps in the optimal transport theory. See, e.g., \cite{mccann2011five} and \cite{kramkov2019optimal}. 

\section{Proof of Theorem \ref{mainth-limit}}\label{sec:proof-pure}

The proof of Theorem \ref{mainth-limit} is structured as follows: First, we prove the result (existence of deterministic optimal autoencoders) for discrete approximations (Theorem \ref{mainth1}). Then, we take the continuous limit in Section \ref{sec:cont-lim} and establish the existence of an optimal autoencoder.

\subsection{Discrete Approximation}\label{sec-discrete}

We call an autoencoder $K$-finite if the support $\tau$  in Definition \ref{dfn1} has cardinality $K.$ An optimal $K$-finite autoencoder is the one attaining the highest   {\it among all $K$-finite autoencoders.} A deterministic $K$-finite autoencoder corresponds to an optimal autoencoder $\cA(\go)$ that only takes $K$ different values $a_1,\cdots,a_K.$ In this case, $\gO_k=\{\go:\ \cA(\go)=a_k\}$ defines a partition of $\gO$. 
  
The following is the main result of this section. 

\begin{theorem}[Optimal $K$-finite autoencoder]\label{mainth1} There always exists an optimal $K$-finite autoencoder which is a partition. 
\end{theorem}

Once the system has processed information and produced a signal $k\in \{1,\cdots,K\}$, it updates the probability distribution of $\go$ using the Bayes rule. To do so, the system just needs to know $\pi_k(\go)=\Pr(k|\go)$,  the probability of observing signal $k$ given that the true state is $\go.$ The distribution $\tau$ from Definition \ref{dfn1} is supported on $\{\mu_1,\cdots,\mu_K\},$ where, by the Bayes rule, 
\[
\mu_k=\Pr(\go|k)\ =\ \frac{\pi_k(\go)p(\go)}{\int_\gO\pi_k(\go)p(\go)d\go},
\]
and the probability of $\mu_k$ is $\tau(\mu_k)=\int_\gO\pi_k(\go)p(\go)d\go.$ 
A $K$-finite autoencoder can therefore be equivalently characterized by a set of measurable functions $\pi_k(\go),\ k=\{1,\cdots,K\}$ satisfying conditions $\pi_k(\go)\in[0,1]$ and $\sum_k\pi_k(\go)\ =\ 1$ with probability one. 

We use $\bar\pi\ =\ (\pi_k(\go))_{k=1}^K\in [0,1]^K$ to denote the random $K$-dimensional vector representing the autoencoder. As we show below, a key implication of this setting is that, with a continuous state space and under appropriate regularity conditions, randomization is never optimal, and hence optimal autoencoder is always given by a {\it partition}. While this result might seem intuitive, its proof is non-trivial. 
To state the main result of this section ---the optimality of partitions---we need also the following definition.

\begin{definition}
We say that functions $\{f_1(\go),\cdots,f_{L_1}(\go)\},\ \go\in\gO,$ are linearly independent modulo $\{g_1(\go),\cdots,g_{L_2}(\go)\}$ if there exist no real vectors $h\in \R^{L_1},\ k\in \R^{L_2}$ with $\|h\|\not=0,$ such that
\[
\sum_i h_i f_i(\go)\ =\ \sum_j k_j g_j(\go)\qquad for\ all\ \go\in\go\,.
\]
In particular, if $L_1=1,$ then $f_1(\go)$ is linearly independent modulo $\{g_1(\go),\cdots,g_{L_2}(\go)\}$ if $f_1(\go)$ cannot be expressed as a linear combination of $\{g_1(\go),\cdots,g_{L_2}(\go)\}.$
\end{definition}

\section{Under Analyticity and Generic Position, All Optimal Policies Are Pure}

We need the following technical condition.

\begin{definition}\label{main-ass-indep} We say that $W, G$ are in a generic position if for any fixed $a,\tilde a\in \cR^N,\ a\not=\tilde a$, the function $W(a,\go)-W(\tilde a,\go)$ is linearly independent modulo $\left\{\{G_n(a,\go)\}_{n=1}^N,\{G_n(\tilde a,\go)\}_{n=1}^N\right\}$;
\end{definition}

$W, G$ are in generic position for generic functions $W$ and $G$.\footnote{The set of $W,G$ that are not in generic position is nowhere dense in the space of continuous functions.} We will also need a key property of real analytic functions\footnote{A function is real analytic if it can be represented by a convergent power series in the neighborhood of any point in its domain.} that we use in our analysis (see, e.g., \cite*{HugonnierMalamudTrubowitz2012}).

\begin{proposition}\label{zero-go} If a real analytic function $f(\go)$ is zero on a set of positive Lebesgue measure, then $f$ is identically zero. Hence, if real analytic functions $\{f_1(\go),\cdots,f_{L_1}(\go)\}$ are linearly dependent modulo $\{g_1(\go),\cdots,g_{L_2}(\go)\}$ on some subset $I\subset\gO$ of positive Lebesgue measure, then this linear dependence also holds on the whole $\gO$ except, possibly, a set of Lebesgue measure zero.
\end{proposition}

Using Proposition \ref{zero-go}, it is possible to prove the main result of this section:

\begin{theorem}[Optimal finite autoencoder]\label{mainth1-analytic} There always exists an optimal $K$-finite autoencoder $\bar\pi^*$ which is a partition. Furthermore, if $W,\ G$ are real analytic in $\go$ for each $a$ and are in generic position, then any $K$-finite optimal autoencoder is a partition.
\end{theorem}

\begin{proof}[Proof of Theorem \ref{mainth1-analytic}] 
The fact that  is bounded and depends smoothly on the $\bar\pi$ follows by standard arguments. Existence of an optimal information autoencoder then follows trivially from compactness. Indeed, since $\pi_k(\go)\in [0,1]$, the are square integrable and, hence, compact in the weak topology of $L_2(p).$ The identity $\sum_k \pi_k=1$ is trivially preserved in the limit. Continuity of utility  in $\pi_k$ follows directly from the assumed integrability and regularity, hence the existence of an optimal autoencoder. 

By \eqref{sys-3}, 
\begin{equation}
\mathbb{E}_{\mu_k}[G(a(k),\go)|k]\ =\ 0\,,
\end{equation}
where 
\[
\mu_k(\go)\ =\ \frac{\pi_k(\go)p(\go)}{\int \pi_k(\go)p(\go) d\go}\,, 
\]
To compute the Frechet differentials of $a(k),$ we take a small perturbation $\eta(\go)$ of $\pi_k(\go)$. By the regularity assumptions and the Implicit Function Theorem,
\[
a(k,\eps)\ =\ a(k)+\eps a^{(1)}(k) +\ O(\eps^2)\,, 
\]
for some $a^{(1)}(k)$. We have 
\begin{equation}
\begin{aligned}
&0\ =\ \int (\pi_k(\go)+\eps \eta(\go))p(\go) G(a(k,\eps),\ \go)d\go\\
&=\ \int (\pi_k(\go)+\eps \eta(\go))p(\go) G(a(k)+\eps a^{(1)}(k))d\go\\
&\approx\ \Bigg(
\int \pi_k(\go) p(\go)\Bigg(G(a(k),\ \go)+G_a (\eps a^{(1)}(k))
\Bigg)d\go\\
&+\eps \int \eta(\go)p(\go) \Bigg(G(a(k))+G_a\eps a^{(1)}(k)
\Bigg)d\go
\Bigg)\ +\ O(\eps^2)\\
&=\ \eps
\Bigg(
\int \pi_k(\go) p(\go)G_a a^{(1)}(k)d\go+\int \eta(\go)p(\go) G(a(k))d\go\Bigg)\ +\ O(\eps^2)
\end{aligned}
\end{equation}
As a result, we get
\begin{equation}\label{x1}
\begin{aligned}
&a^{(1)}(k)\ =\ -\bar G_a(k)^{-1}\,\int \eta(\go)p(\go) G(a(k),\go)d\go,\ \bar G_a(k)\ =\ \int \pi_k(\go) p(\go)G_a d\go\,,
\end{aligned}
\end{equation}
The  function is given by 
\begin{equation}
\begin{aligned}
&\bar W(\pi)\ =\ \mathbb{E}[W(a,\go)]\ =\ \sum_k \int_\gO W(a(k),\go)\pi_k(\go)p(\go)d\go\,.
\end{aligned}
\end{equation}
Suppose that the optimal information structure is not a partition. Then, there exists a subset $I\subset\gO$ of positive $p$-measure and an index $k$ such that $\pi_k(\go)\in (0,1)$ for $p$-almost all $\go\in I.$ Since $\sum_i\pi_i(\go)=1$ and $\pi_i(\go)\in[0,1],$ there must be an index $k_1\not=k$ and a subset $I_1\subset I$ such that $\pi_{k_1}(\go)\in (0,1)$ for $p$-almost all $\go\in I_1.$ Consider a small perturbation $\{\tilde\pi(\eps)\}_i$ of the information autoencoder, keeping $\pi_i,\ i\not=k,k_1$ fixed and changing $\pi_k(\go)\to \pi_k(\go)+\eps \eta(\go),\ \pi_{k_1}(\go)\to\pi_{k_1}(\go)-\eps(\go)$ where $\eta(\go)$ in an arbitrary bounded function with $\eta(\go)=0$ for all $\go\not\in I_1.$ Define $\eta_k(\go)=\eta(\go),\ \eta_{k_1}(\go)=-\eta(\go),$ and $\eta_i(\go)=0$ for all $i\not=k,k_1.$ A second-order Taylor expansion in $\eps$ gives
\begin{equation}\label{w-expan1}
\begin{aligned}
&\sum_{i} \int_\gO W(a(i,\eps),\go) (\pi_i(\go)+\eps \eta_i(\go))p(\go)d\go\\
&\approx\  \int_\gO \Bigg(W(a(i),\go)+W_a(a(i),\go)(\eps a^{(1)}(i))\Bigg) (\pi_i(\go)+\eps \eta_i(\go))p(\go)d\go\ +\ O(\eps^2)\\
&=\ \bar W(\pi)\ +\ \eps\sum_i \Bigg(
\int_\gO (W(a(i),\go) \eta_i(\go)+ W_a(a(i),\go)a^{(1)}(i)\pi_i(\go))p(\go)d\go
\Bigg)\ +\ O(\eps^2)
\end{aligned}
\end{equation}
Since, by assumption, $\{\pi_i\}$ is an optimal information autoencoder, it has to be that the first order term in \eqref{w-expan1} is zero, while the second-order term is always non-positive. We can rewrite the first order term as
\begin{equation}\label{w-expan}
\begin{aligned}
&\sum_i \Bigg(
\int_\gO (W(a(i),\go) \eta_i(\go)+ W_a(a(i),\go)a^{(1)}(i)\pi_i(\go))p(\go)d\go
\Bigg)\\
&=\ \sum_i \int_\gO \Bigg(W(a(i),\go)\\
& -\  \Big(\int W_a(a(i),\go_1)\pi_i(\go_1)p(\go_1)d\go_1
\Big)\bar G_a(i)^{-1}\, G(a(i),\go)
\Bigg)
\eta_i(\go)p(\go)d\go
\end{aligned}
\end{equation}
and hence it is zero for all considered perturbations if and only if
\begin{equation}\label{w-expan2}
\begin{aligned}
&W(a(k),\go)\ -\ \Big(\int W_a(a(k), \go_1)\pi_k(\go_1)p(\go_1)d\go_1
\Big)\bar G_a(k)^{-1}\, G(a(k),\go)\\
&=\ W(a(k_1),\go)\ -\ \Big(\int W_a(a(k_1),\go)\pi_{k_1}(\go_1)p(\go_1)d\go_1
\Big)\bar G_a(k_1)^{-1}\, G(a(k_1),\go)
\end{aligned}
\end{equation}
Lebesgue-almost surely for $\go\in I_1.$ By Proposition \ref{zero-go}, \eqref{w-expan2} also holds for all $\go\in\gO.$ Hence, by Assumption \ref{main-ass-indep}, $a(k)=a(k_1),$ which contradicts our assumption that all $a(k)$ are different. 
\end{proof}

\begin{proof}[Proof of Theorem \ref{mainth1}] Suppose first that $\gO$ is compact. 
	Let now $W_n(a,\go)$ be a sequence of real analytic utility functions in generic positions, uniformly converging to $W(a,\go).$ Let $\{\gO_k(n)\}_{k=1}^K$ be the respective partitions from Theorem \ref{mainth1-analytic}, and $\pi_k(n)={\bf 1}_{\gO_k(n)}.$ Passing to a subsequence, we $\pi_k(n)\to\pi_k^*$ for each $k$ when $n\to\infty.$ Passing to a subsequence once again, we may assume that $a_n\to a_*.$ Now, for any $K$-finite autoencoder $\{\tilde\pi_k\},$ 
	\[
	W(\{\tilde\pi_k\})\ =\ \lim_{n\to\infty} W_n(\{\tilde\pi_k\})\ \le\ \lim_{n\to\infty} W_n(\{\pi_k(n)\})\ =\ W(\{\pi_k^*\})
	\]
	where the last result follows from uniform continuity of $W_n$, uniform convergence, and compactness of $\gO.$  
	
	If $\gO$ is not compact, the proofs can be trivially adjusted by taking a sequence of compact subsets expanding to $\gO.$
\end{proof}

\subsection{The Structure of Optimal Partitions}

The goal of this section is to provide a general characterization of an optimal partition in Theorem \ref{mainth1}.

We use $D_aG(a,\go)\in \R^{M\times M}$ to denote the Jacobian of the map $G$, and, similarly, $D_aW(a,\go)\in \R^{1\times M}$ the gradient of the utility  function $W(a,\go)$ with respect to $a.$ For any vectors $x_k\ \in\ \R^{M},\ k=1,\cdots,K$ and actions $\{a(k)\}_{k=1}^K,$ let us define the partition
\begin{equation}\label{partitions1}
	\begin{aligned}
		\gO_k^*(\{x_\ell\}_{\ell=1}^K,\{a_\ell\}_{\ell=1}^K)\ &=\
		\Bigg\{
		\go\ \in\ \gO\ :\ W(a(k),\ \go)-x_k^\top G(a(k),\ \go)\\
		& =\ \max_{1\le l\le K} \left(W(a(l),\ \go)\ -\ x_l^\top G(a(l),\ \go)\right)
		\Bigg\}
	\end{aligned}
\end{equation}
Equation \eqref{partitions1} is basically the first-order condition for the optimization problem. 

\begin{theorem}\label{regular-partition} Any optimal partition in Theorem \ref{mainth1} satisfies the following conditions:
	\begin{itemize}
		\item local optimality holds: $\gO_k\ =\ \gO_k^*(\{x_\ell\}_{\ell=1}^K,\{a_\ell\}_{\ell=1}^K)$ with $x_k^\top\ =\ \bar D_aW(k)(\bar D_aG(k))^{-1}\,,$
		where we have defined for each $k=1,\cdots,K\,$
		\begin{equation}
			\begin{aligned}
				&\bar D_aW(k)\ =\ \int_{\gO_k} D_aW(a(k),\go)p(\go)d\go\,,\ \bar D_aG(k)\ =\ \int_{\gO_k} D_{a}G(a(k),\go)p(\go)d\go
			\end{aligned}
		\end{equation}
		\item the actions $\{a(k)\}_{k=1}^K$ satisfy the fixed point system
		\begin{equation}\label{gak1}
			\int_{\gO_k} G(a(k),\go)p(\go)d\go\ =\ 0,\ k=1,\,\cdots,\,K\,.
		\end{equation}
		\item the boundaries of $\gO_k$ are a subset of the variety\footnote{This variety is real analytic when so are $W$ and $G.$ A real analytic variety in $\R^L$ is a subset of $\R^L$ defined by a set of identities $f_i(\go)=0,\ i=1,\cdots,I$ where all functions $f_i$ are real analytic. If at least one of functions $f_i(\go)$ is non-zero, then a real analytic variety is always a union of smooth manifolds and hence has a Lebesgue measure of zero. When $W,G$ are real analytic and are in generic position, the variety $\left\{\go\in \R^L\ :\ W(a(k),\ \go)\ -\ x_k^\top G(a(k),\ \go)\ =\ W(a(l),\ \go)\ -\  x_l^\top G(a(l),\ \go)\right\}$ has a Lebesgue measure of zero for each $k\not=l.$}
		\begin{equation}\label{indiff}
			\cup_{k\not=l}\left\{\go\in \R^L:W(a(k),\go)-x_k^\top G(a(k),\go)=W(a(l),\go)-x_l^\top G(a(l),\go)\right\}\,.
		\end{equation}
	\end{itemize}
\end{theorem}

\begin{proof}[Proof of Theorem \ref{regular-partition}]
	Suppose a partition $\go\ =\ \cup_k \gO_k$ is optimal. By regularity, equilibrium actions satisfy the first order conditions
	\[
	\int_{\gO_k} G(a(k),\ \go)p(\go)d\go\ =\ 0\,.
	\]
	Consider a small perturbation, whereby we move a small mass on a set $\cI\subset \gO_k$ to $\gO_l.$ Then, the marginal change in $a_n(k)$ can be determined from
	\begin{equation}
		\begin{aligned}
			&0\ =\ \int_{\gO_k} G(a(k),\ \go)p(\go)d\go\ -\ \int_{\gO_k\setminus\cI} G(a(k,\cI),\ \go)p(\go)d\go\\
			&\approx\ -\int_{\gO_k} D_aG(a(k),\ \go)\Delta a(k)\ p(\go)d\go\ + \int_{\cI} G(a(k),\ \go)p(\go)d\go\,,
		\end{aligned}
	\end{equation}
	implying that the first order change in $a$ is given by
	\[
	\Delta a(k)\ \approx\ (\bar D_aG(k))^{-1}\int_{\cI} G(a(k),\ \go)p(\go)d\go\,.
	\]
	Thus, the change in utility  is\footnote{Note that $D_aW$ is a horizontal (row) vector.}
	\begin{equation}
		\begin{aligned}
			&\Delta W\ =\ \int_{\gO_k}W(a(k),\ \go)p(\go)d\go\ -\ \int_{\gO_k\setminus \cI}W(a(k,\cI),\ \go)p(\go)d\go\\
			&+\int_{\gO_l}W(a(l),\ \go)p(\go)d\go\ -\ \int_{\gO_l\cup \cI}W(a(l,\cI),\ \go)p(\go)d\go\\
			&\approx\ -\int_{\gO_k}D_aW(a(k),\ \go)\Delta a(k)p(\go)d\go+\int_{\cI}W(a(k),\ \go)p(\go)d\go\\
			&-\int_{\gO_l}D_aW(a(l),\ \go)\Delta a(l)p(\go)d\go-\int_{\cI}W(a(l),\ \go)p(\go)d\go\\
			&=\  -\bar D_aW(k)(\bar D_aG(k))^{-1}\int_{\cI} G(a(k),\ \go)p(\go)d\go+\int_{\cI}W(a(k),\ \go)p(\go)d\go\\
			&+\bar D_aW(l)(\bar D_aG(l))^{-1}\int_{\cI} G(a(l),\ \go)p(\go)d\go-\int_{\cI}W(a(l),\ \go)p(\go)d\go\,.
		\end{aligned}
	\end{equation}
	This expression has to be non-negative for any $\cI$ of positive Lebesgue measure. Thus,
	\begin{equation}
		\begin{aligned}
			& -\bar D_aW(k)(\bar D_aG(k))^{-1}G(a(k),\ \go)+W(a(k),\ \go)\\
			&+\bar D_aW(l)(\bar D_aG(l))^{-1} G(a(l),\ \go)\ -\ W(a(l),\ \go)\ \ge\ 0
		\end{aligned}
	\end{equation}
	for Lebesgue almost any $\go\in\gO_k.$
\end{proof}

As we explain above, the problem of finding the optimal autoencoder is equivalent to the problem of Bayesian persuasion and optimal information design (see \cite{KamGenz2011}, \cite{kamenica2019bayesian}, \cite{kamenica2021bayesian}, \cite{bergemann2016information} for an overview).
Several papers study the problem of Bayesian persuasion in the one-dimensional case (i.e., when $L=1$ so that $\go\in \R^1$) and derive conditions under which the optimal signal structure is a monotone partition into intervals. Such a monotonicity result is intuitive, as one would expect the optimal autoencoder to only pool nearby states. The most general results currently available are due to \cite{Hopenhayn2019} and \cite{DworczakMartini2019},\footnote{See also \cite{mensch2018}.} but they cover the case when sender's utility (utility  function in our setting) only depends on $\mathbb{E}[\go]\in \R^1.$ This is equivalent to $G(a,\go)\ =\ a-\go$ in our setting. Under this assumption, \cite{DworczakMartini2019} derive necessary and sufficient conditions guaranteeing that the optimal signal structure is a monotone partition of $\gO$ into a union of disjoint intervals. \cite{Arielietal2020} (see, also, \cite{kleiner2020extreme}) provide a full solution to the autoencoder problem when $a(k)=\mathbb{E}[\go|k]$ and, in particular, show that the partition result does not hold in general when the signal space is continuous. 
Theorem \ref{regular-partition} proves that a $K$-finite optimal autoencoder is in fact always a partition when the state space is continuous and the signal space is discrete. \cite{DworzakKolotilin2019} establish necessary and sufficient conditions for convexity of partitions in multiple dimensions. 
\footnote{Of course, as \cite{DworczakMartini2019} and \cite{Arielietal2020} explain, even in the one-dimensional case the monotonicity cannot be ensured without additional technical conditions. No such conditions are known in the multi-dimensional case. \cite{DworczakMartini2019} present an example with four possible actions $(K=4)$ and a two-dimensional state space $(L=2)$ for which they are able to show that the optimal autoencoder is a partition into four convex polygons.}

Consider the optimal autoencoder of Theorem \ref{regular-partition} and define the piece-wise constant function\,
\begin{equation}\label{opt-discrete}
	\cA(\go)\ =\ \sum_k a(k) {\bf 1}_{\go \in \gO_k}\,.
\end{equation}

\subsection{Convexity} 

\begin{proposition}\label{convexity} Suppose that $G(a,\go)\ =\ a-g(\go)$ and $W=W(a)$. Then, the function $C(x)\ =\ W(a(g^{-1}(x)))-D_a(a(g^{-1}(x)))^\top (a(g^{-1}(x))-x)$ is convex, and $D_aW(\cA(g^{-1}(x)))$ is its sub-gradient. In particular, $D_aW(\cA(g^{-1}(x)))$ is a monotone map and its level sets are convex. 
\end{proposition}

\begin{proof} By Theorem \ref{regular-partition}, we have  
\begin{equation}
C(x)\ =\ \max_k (W(a(k))-D_a(a(k))^\top (a(k)-x))\,,
\end{equation}
and the convexity follows because the supremum of linear functions is convex. Furthemore, inside each $g(\gO_k),$ the function $C(x)$ is linear, and  $D_a(a(k))=D_aW(\cA(g^{-1}(x)))$ is its gradient for all $x\in g(\gO_k).$ The proof is complete.  
\end{proof}

\subsection{Continuous Limit} \label{sec:cont-lim}

In this section, we prove that a deterministic optimal autoencoder (see Definition  \ref{dfn1}) solving the unconstrained problem always exists. We do this by passing to the limit in Theorem \ref{mainth1}. The proof of convergence is non-trivial due to additional complications created by the potential non-compactness of the set $\gO.$\footnote{Note that all existing models of Bayesian persuasion (with the exception of \cite{tamura2018Bayesian}) assume that $\gO$ is compact. This precludes many practical applications where the distributions (such as, e.g., the Gaussian distribution) do not have compact support.}

\begin{lemma}\label{lem-approx} When $K\to\infty,$ maximal  is attained with $K$-finite optimal autoencoders converges to the maximal utility  attained in the full, unconstrained problem of Definition \ref{dfn1}. 
\end{lemma}

\begin{proof}[Proof of Lemma \ref{lem-approx}] The proof requires some additional arguments because $\gO$ is not necessarily compact. 
First, consider an increasing sequence of compact sets $X_n=\{\go:g(\go)\le n\}$ such that $X_n$ converge to $\gO$ as $n\to\infty.$  For any measure $\mu,$ let $\mu_X$ be its restriction on $X.$ Let $a_n=a(\mu_{X_n}).$ The first observation is that Assumptions \ref{integrability} and \ref{ac} imply that $a_n\to a$ uniformly as $n\to \infty$. Indeed, 
\[
\int_{X_n} G(a_n,\go)d\mu(\go)\ =\ \int_{\gO} G(a,\go)d\mu(\go)=0
\]
implies that 
\begin{equation}
\begin{aligned}
&\int_{X_n} (G(a_n,\go)-G(a,\go))d\mu(\go)\ =\ \int_{\gO\setminus X_n} G(a,\go)d\mu(\go)\\
&\le\ \int_{\gO\setminus X_n} \eps^{-1}\|a-a_*(\go)\|d\mu(\go)\ \le\ 2\eps^{-1}\mu(\gO\setminus X_n)^{1/2}\left(\int_{\gO\setminus X_n} \|a_*(\go)\|^2d\mu(\go)\right)^{1/2}\\
&\ \le\ \eps^{-1}(\mu(\gO\setminus X_n)+\int_{\gO\setminus X_n} \|a_*(\go)\|^2d\mu(\go))\,.
\end{aligned}
\end{equation}
Multiplying by $(a-a_n)$, we get 
\[
\eps\,\|a-a_n\|^2 (1-\mu(\gO\setminus X_n))\ \le\  \|a-a_n\| \eps^{-1}(\mu(\gO\setminus X_n)+\int_{\gO\setminus X_n} \|a_*(\go)\|^2d\mu(\go))
\]
Furthermore, by Lemma \ref{existence}, $\|a-a_n\|\le\ 2\left(\int_{\gO} g(\go)d\mu(\go)\right)^{1/2}\ \le\ 1+\int_{\gO} g(\go)d\mu(\go)$ and therefore 
\[
\|a-a_n\|\ \le\ C\Bigg( \mu(\gO\setminus X_n)(1+\int_{\gO} g(\go)d\mu(\go))+\int_{\gO\setminus X_n} g(\go)d\mu(\go)
\Bigg)
\] 
for some constant $C>0.$ Now, pick a $\tau\in \Delta(\Delta(\gO)).$ Since the function $q(x)={\bf 1}_{x>n}$ is monotone increasing in $x$, we get 
\begin{equation}
\begin{aligned}
&\mu(\gO\setminus X_n)\int_{\gO} g(\go)d\mu(\go)\ =\ \int_{\gO} q(g(\go))d\mu(\go) \int_{\gO} g(\go)d\mu(\go)\ \le\ \int_{\gO} q(g(\go))g(\go)d\mu(\go)\\ 
&=\ \int_{\gO\setminus X_n} g(\go)d\mu(\go)\,   
\end{aligned}
\end{equation}

and therefore 
\[
\|a-a_n\|\ \le\ C\int_{\gO\setminus X_n}(1+2g(\go))d\mu(\go)\,.
\]
Then, we have by the Jensen inequality that 
\begin{equation}
\begin{aligned}
&|\bar W(\mu)-\bar W(\mu_{X_n})|\ \le\ \int_{\gO\setminus X_n} |W(a(\mu),\go)|d\mu(\go)\ +\ \int_{X_n}|W(a(\mu),\go)-W(a_n(\mu),\go)|d\mu(\go)\\
&\ \le\ \int_{\gO\setminus X_n}(g(\go) f(\int_{\gO} g(\go)d\mu(\go)))d\mu(\go)\\
&+\ \int_{\gO}\|a(\mu)-a_n(\mu)\| (g(\go) f(\int_{\gO} g(\go)d\mu(\go)))d\mu(\go)\\
&\ \le\ \int_{\gO\setminus X_n}g(\go)d\mu(\go)\,\int_\gO f(g(\go))d\mu(\go)\\
&+\ \|a(\mu)-a_n(\mu)\|  \int_{\gO} g(\go)d\mu(\go)\,\int_\gO f(g(\go))d\mu(\go)\,.
\end{aligned}
\end{equation}
Since the function $q(x)=x{\bf 1}_{x>n}$ is monotone increasing in $x$ and $f$ is monotone increasing, we get 
\[
 \int_{\gO} g(\go)d\mu(\go)\,\int_\gO f(g(\go))d\mu(\go)\ \le\ \int_\gO g(\go) f(g(\go))d\mu(\go)
\]
and therefore, by the same monotonicity argument, 
\begin{equation}
\begin{aligned}
&\|a(\mu)-a_n(\mu)\|  \int_{\gO} g(\go)d\mu(\go)\,\int_\gO f(g(\go))d\mu(\go)\\ 
&\le\ C\int_{\gO\setminus X_n}(1+2g(\go))d\mu(\go)\,\int_\gO g(\go) f(g(\go))d\mu(\go)\\
&\le\ C\int_{\gO\setminus X_n}(1+2g(\go)) g(\go) f(g(\go))d\mu(\go)\,.
\end{aligned}
\end{equation}
Similarly, 
\begin{equation}
\begin{aligned}
&\int_{\gO\setminus X_n}g(\go)d\mu(\go)\,\int_\gO f(g(\go))d\mu(\go)\ =\ \int_{\gO}q(g(\go))d\mu(\go)\,\int_\gO f(g(\go))d\mu(\go)\\ 
&\le\ \int_\gO q(g(\go))f(g(\go))d\mu(\go)\ =\ \int_{\gO\setminus X_n}g(\go)f(g(\go))d\mu(\go)\,.
\end{aligned}
\end{equation}
Therefore, by the Fubini Theorem, 
\begin{equation}
\begin{aligned}
&|\int_{\Delta(\mu)} (\bar W(\mu)-\bar W(\mu_{X_n}))d\tau(\mu)|\\
& \le\ \int_{\gD(\gO)}\int_{\gO\setminus X_n}g(\go)f(g(\go))d\mu(\go)d\tau(\mu)\\
&+\  \int_{\gD(\gO)}C\int_{\gO\setminus X_n}(1+2g(\go)) g(\go) f(g(\go))d\mu(\go)d\tau(\mu)\\
&=\ \int_{\gO\setminus X_n}(g(\go)f(g(\go))+C(1+2g(\go)) g(\go) f(g(\go)))p(d\go)\,.
\end{aligned}
\end{equation}
Thus, Assumption \ref{integrability} implies that we can restrict our attention to the case when $\gO=X_n$ is compact. 

In this case, the Prokhorov Theorem implies that $\Delta(\gO)$ is compact in the weak* topology and this topology is metrizable. Thus, for any $\eps>0,$ we can decompose $\Delta(\gO)=Q_1\cup\cdots\cup Q_K,$ where all $Q_k$ have diameters less than $\eps.$ We can now approximate $\tau$ by $\tilde\tau=\sum_k \nu_k \gd_{\mu_k}$ with $\mu_k=\int_{Q_k} \mu d\tau(\mu)/\nu_k$ and $\nu_k= \int_{Q_k} d\tau(\mu).$ Clearly, $\int \mu d\tilde\tau(\mu)=p,$ and therefore it remains to show that $\bar W$ is continuous in the weak* topology. 

To this end, suppose that $\mu_n\to\mu$ in the weak* topology. Let us first show $a_n=a(\mu_n)\to a(\mu)=a.$ Suppose the contrary.  Since $\gO$ is compact and $G$ is continuous and bounded, Lemma \ref{existence} implies that $a_n$ are uniformly bounded. Pick a subsequence such that $\|a_n-a\|>\eps$ for some $\eps>0$ and subsequence $a_n\to b$ for some $b\not=a.$ Since $G(a_n,\go)\to G(b,\go)$ uniformly on $\gO,$ we get a contradiction because 
\[
\int G(a_n,\go)d\mu_n-\int G(b,\go)d\mu\ =\ \int (G(a_n,\go)-G(b,\go))d\mu_n\ +\ \int G(b,\go)d(\mu_n-\mu)\,.
\]
The second term converges to zero because of weak* convergence. The first term can be bounded by 
\[
|\int (G(a_n,\go)-G(b,\go))d\mu_n|\ \le\ C \|a_n-b\|
\]
and hence also converges to zero. Thus, $\int G(b,\go)d\mu=\int G(a,\go)d\mu=0$, implying that $a=b$ by the strict monotonicity of the map $G.$ The same argument implies the required continuity of $\bar W(\mu).$ 
\end{proof}

\subsection{Last Step of the proof of Theorem \ref{mainth-limit}} 
\label{sec-lim}
\begin{proof}[Proof of Theorem \ref{mainth-limit}]. The existence of an optimal autoencoder follows directly from weak* compactness of $\Delta(\Delta(\gO))$ and the continuity proved in Lemma \ref{lem-approx}. Thus, it remains to prove the existence of a deterministic optimal autoencoder. Let $a_K(\go)$ 
correspond to the optimal autoencoder from Theorem \ref{regular-partition}, defined using \eqref{opt-discrete}. We now take the limit as $K\to\infty.$ By Lemma \ref{existence}, $a_K(\go)$ have uniformly bounded $L_2$-norms and, hence, contain weakly converging sub-sequence. We will now prove that it has a subsequence that converges Lebesgue-almost surely. To this end, we use Proposition \ref{convexity} and notice that the corresponding functions $C_K(x)$ are convex in $x$ and have bounded first moments. As a result, $C_K(x)$ must be bounded on compact subsets. Indeed, otherwise there exists a point $x_*$ such that $C_K(x_*)\to\infty$. Passing to a subsequence, we may assume that the subgradient directions $DC_K(x_*)/\|DC_K(x_*)\|$ also converge and, as a result $C_K(x)\to\infty$ for all $x$ in the half-space $DC_K(x_*)^\top (x-x_*)\ge 0,$ which is impossible since the pull-back of $p$ under $g$ assigns positive measure to this half-space.

Thus, they must be bounded on compact subsets. A locally bounded sequence of convex functions always has a convergent sub-sequence, and the respective sub-gradients also converge Lebesgue almost-surely. Thus, $q_K(\go)=D_aW(a_K(\go))$ converges almost surely to a Borel-measurable limit $q(\go).$ Let $\gO_{k,i}=\{\go:\ a_K(\go)=G_i(q_K(\go))\}$ where $G_i$ are the branches of $D_aW^{-1}.$ Then, 
\begin{equation}
a_K(\go)\ =\ \sum_{i=1}^\infty {\bf 1}_{\gO_{i,k}}G_i(q_K(\go))\,. 
\end{equation}
By continuity, $G_i(q_K(\go))$ converges to $G_i(q(\go))$ almost surely. Passing to a subsequence, we may assume that ${\bf 1}_{\gO_{i,k}}\to {\bf 1}_{\gO^*_i}$ as $k\to\infty$ in $L_2$ for some partition $\{\gO^*_i\}$. Hence, the convergence also happens almost surely.  
\end{proof}

\section{Properties of Optimal Policies} 

\subsection{First Order Conditions} 

An optimal autoencoder is a probability distribution $\tau$ on $\Delta(\gO).$ For any $\mu\in \gD(\gO),$ we have $\mathbb{E}[G(a,\go)|\mu]=\mathbb{E}[G(a,\go)|a]=0$. Furthermore, the marginal distribution of $\go$ always coincides with $p(\go).$ Conversely, for any joint distribution $\gamma(a,\go)\in \gD(\R^L\times \gO)$ satisfying the $E^{\gamma}[G(a,\go)|a]=0$ we can define an optimal autoencoder with $\mu$ being the conditional distribution of $\go$ conditional on $a.$ Define $\Gamma\subset \gD(\R^L\times \gO)$ to be the set of distributions satisfying these two constraints: $\gamma(\R^L,\go)=p(\go)$ and $\mathbb{E}[G(a,\go)|a]=0.$\footnote{See, \cite{kramkov2019optimal} where this representation is derived for a special case of this problem with $g(\go)=\go,\ L=2$, and $W(a)=a_1a_2$.} Then, we can reformulate the optimal autoencoder problem as
\begin{equation}
\max_{\gamma\in \Gamma} \mathbb{E}[W(a,\go)]\,. 
\end{equation}
This formulation is extremely convenient because it allows to directly derive analytical first order conditions for this problem.

\begin{proposition}\label{delta-dev} Let $\gamma$ be the joint distribution of $(a,\go)$ for an optimal autoencoder. Let 
\begin{equation}
x(a)\ =\ \int D_a W(a,\go)d\gamma(a,\go|a)\,\left(\int D_a G(a,\go)d\gamma(a,\go|a)\right)^{-1}.
\end{equation}
Then, 
\begin{equation}\label{foc}
\int (x(a)^\top G(a,\go)- W(a,\go))d\eta\ +\ \int W(\tilde a_*,\go)d\eta(\R^L,\go)\ \le\ 0
\end{equation}
for every measure $\eta$ such that $\supp(\eta)\subset\supp(\gamma)\,$ such that $\int f(\|a\|^2)\psi^2(\go)d\eta(a,\go) <\infty.$ 

If $W=W(a)$ and $G(a,\go)=a-g(\go),$ we have 
\begin{equation}\label{kramkov2019optimal-12}
\int (D_aW(a)(a -g(\go))-W(a))d\eta\ +\ W(\int g(\go)d\eta)\le 0\,. 
\end{equation}
\end{proposition}

\begin{proof}[Proof of Proposition \ref{delta-dev}] We closely follow the arguments and notation in \cite{kramkov2019optimal}. Let $\gamma$ be the joint distribution of the random variables $\go$ and $\cA(\go).$ We first establish \eqref{kramkov2019optimal-12} for a Borel probability measure $\eta$ that has a bounded density with respect to $\gamma.$ Then, the general result follows by a simple modification of the argument in the proof of Theorem A.1 in \cite{kramkov2019optimal}. Let
\[
V(a,\go)\ =\ \frac{d\eta}{d\gamma}(a,\go)\,.
\]
We choose a non-atom $q\in \R^L$ of $\mu(da)\ =\ \gamma(da,\R^L)$ and define the probability measure
\[
\zeta(da,d\go)\ =\ \gd_q(da)\eta(\R^L,d\go)\,,
\]
where $\gd_q$ is the Dirac measure concentrated at $q.$ For sufficiently small $\eps>0$ the probability measure
\[
\tilde\gamma\ =\ \gamma\ +\ \eps (\zeta-\eta)
\]
is well-defined and has the same $\go$-marginal $p(\go)$ as $\gamma$. Let $\tilde a$ be the optimal action satisfying 
\[
\tilde\gamma (G(\tilde a,\go)|\tilde a)\ =\ 0\,.
\]
The optimality of $\gamma$ implies that 
\begin{equation}\label{kramkov2019optimal-13}
\int W(\tilde a,\go)d\tilde\gamma\ \le\ \int W(a,\go)d\gamma\,. 
\end{equation}
By direct calculation, 
\begin{equation}
\begin{aligned}
&0\ =\ \tilde\gamma (G(\tilde a,\go)|a)\\
&=\ {\bf 1}_{a\not=q}\frac{\int G(\tilde a,\go)d((\gamma|a)-\eps (\eta|a))}{\int d(\gamma-\eps \eta)}\ +\ {\bf 1}_{a=q}\int G(\tilde a,\go) d\eta(\R^L,\go)\\
&=\ {\bf 1}_{a\not=q}\frac{\int G(\tilde a,\go)d(\gamma|a)-\eps\int G(\tilde a,\go)d (\eta|a)}{1-\eps U(a)}\ +\ {\bf 1}_{a=q}\int G(\tilde a,\go) d\eta(\R^L,\go)
\end{aligned}
\end{equation}
where $U(a)=\gamma(V(a,\go)|a)\,.$
Now, we know that 
\[
\int G(a,\go)d(\gamma|a)\ =\ 0,
\]
and the assumed regularity of $G$ together with the implicit function theorem imply that 
\[
\tilde a(a)\ =\ a\ +\ \eps Q(a)\ +\ O(\eps^2)
\]
if $a\not=q$ and
\[
\tilde a\ =\ \tilde a_*\,,
\]
where $\tilde a_*$ is the unique solution to 
\[
\int G(\tilde a_*,\go) d\eta(\R^L,\go)\ =\ 0
\]
for $a=q.$ Here, 
\begin{equation}
\begin{aligned}
&0\ =\ O(\eps^2)\ +\ \int G(a\ +\ \eps Q(a),\go)d(\gamma|a)-\eps\int G(a,\go) V(a,\go) d (\gamma|a)\\
&=\ O(\eps^2)\ +\ \eps \int D_aG(a,\go) d(\gamma|a) Q(a)-\eps\int G(a,\go) V(a,\go) d (\gamma|a)
\end{aligned}
\end{equation}
so that 
\[
Q(a)\ =\ \left(\int D_aG(a,\go) d(\gamma|a)\right)^{-1}\int G(a,\go) V(a,\go) d (\gamma|a)\,. 
\]
Thus, 
\begin{equation}
\begin{aligned}
&\int W(\tilde a(a),\go)d\tilde \gamma\ =\ \int W(\tilde a(a),\go)(1-\eps V(a,\go))d\gamma+\eps \int W(\tilde a_*,\go)d\eta(\R^L,\go)\\
&=\ O(\eps^2)\ +\ \int W(a,\go)d\gamma\ +\ \eps\Bigg(\int (D_aW(a,\go) Q(a) - V(a,\go)) d\gamma\ +\ \int W(\tilde a_*,\go)d\eta(\R^L,\go)
\Bigg)
\end{aligned}
\end{equation}
In view of \eqref{kramkov2019optimal-13}, the first-order term is non-positive:
\[
\int (D_aW(a,\go) Q(a) - W(a,\go)V(a,\go)) d\gamma\ +\ \int W(\tilde a_*,\go)d\eta(\R^L,\go)\ \le\ 0\,.
\]
Substituting, we get 
\[
\int (x(a)^\top \int G(a,\go) V(a,\go) d (\gamma|a)\ - W(a,\go)V(a,\go)) d\gamma\ +\ \int W(\tilde a_*,\go)d\eta(\R^L,\go)\ \le\ 0\,,
\]
which is equivalent to 
\[
\int (x(a)^\top G(a,\go)- W(a,\go))d\eta\ +\ \int W(\tilde a_*,\go)d\eta(\R^L,\go)\ \le\ 0
\]
In the case when $G(a,\go)=a-g(\go)$, we get 
\[
Q(a)\ =\ a U(a)\ -\ R(a)\,,
\]
where we have defined 
\[
U(a)\ =\ \gamma(V(a,\go)|a),\ R(a)\ =\ \gamma(g(\go)V(a,\go)|a)\,,
\]
and 
\[
\tilde a_*\ =\ \int g(\go)d\eta\,.
\]
Thus, we get 
\begin{equation}
\begin{aligned}
&0\ \ge\ \int (D_aW(a) Q(a) - W(a)V(a,\go)) d\gamma\ +\ W(\tilde a_*)\\
&=\ \int (D_aW(a)(a U(a)\ -\ R(a))-W(a)V(a,\go))d\gamma\ +\ W(\tilde a_*)\\
&=\ \int (D_aW(a)(a -g(\go))-W(a))d\eta\ +\ W(\int g(\go)d\eta)\,. 
\end{aligned}
\end{equation}
\end{proof}

An immediate consequence of the first order conditions is the projection result.

\begin{lemma}\label{super-G} Let $a_*(\go)$ be the unique solution to $G(a_*(\go),\go)=0.$ Then, for $\gamma$-almost every $(a,\go)$ we have
\[
x(a)^\top G(a,\go)\ -\ W(a,\go)\ +\ W(a_*(\go),\go) \ \le\ 0\,.
\]
Furthermore, defining 
\begin{equation}
c(a,\go;x)\ =\ W(a_*(\go),\go)-W(a,\go)+x^\top G(a,\go)\,,
\end{equation}
and letting $\Xi$ to be a support of the measure $\gamma(a,\R^L),$ we have 
\[
c(a,\go;x)\ =\ \min_{b\in\Xi}\,c(b,\go;x)\,,
\]
$\gamma(a,\go|a)$ almost surely. In particular, if the autoencoder is deterministic, given by a map $\cA(\go)$ with support $\Xi,$ we have 
\[
c(\cA(\go),\go;x)\ =\ \min_{b\in\Xi}\,c(b,\go;x)\,
\]
for Lebesgue-almost every $\go.$
\end{lemma}

\begin{proof} The first claim follows by selecting $\eta=\gd_{(a,\go)}.$ The second one follows by selecting $\eta=t\gd_{a_1}{\bf 1}_{\gO_1}\gamma|a_1+(1-\kappa t)\gd_{a_2}\gamma|a_2$ for some open set $\gO_1$ and $\kappa=\gamma(\gO_1|a_1)$. In this case, we get from \eqref{foc} that 
\begin{equation}\label{aux10}
\begin{aligned}
&t\int (x(a_1)^\top G(a_1,\go)- W(a_1,\go)){\bf 1}_{\gO_1}d\gamma(\go|a_1)\ +\ (1-\kappa t)\int (x(a_2)^\top G(a_2,\go)- W(a_1,\go))d\gamma(\go|a_2)\\
& +\ \int W(\tilde a_*,\go)(t{\bf 1}_{\gO_1}d\gamma|a_1+(1-\kappa t)d\gamma|a_2)\ \le\ 0\,. 
\end{aligned}
\end{equation}
where $\tilde a_*(t)$ is uniquely determined by 
\[
t\int G(\tilde a_*(t),\go){\bf 1}_{\gO_1}d(\gamma|a_1)+(1-\kappa t)\int G(\tilde a_*(t),\go)d(\gamma|a_2)\ =\ 0\,.
\]
Clearly, \eqref{aux10} is equivalent to 
\begin{equation}\label{aux11111}
\begin{aligned}
&t\int (W(\tilde a_*(t),\go)-W(a_1,\go)+x(a_1)^\top G(a_1,\go)){\bf 1}_{\gO_1} d(\gamma|a_1)\\
&+(1-\kappa t)\int (W(\tilde a_*(t),\go)-W(a_2,\go)+x(a_2)^\top G(a_2,\go))d(\gamma|a_2)\ \le\ 0\,. 
\end{aligned}
\end{equation}
Assuming that $t$ is small, we get 
\[
\tilde a_*(t)\ =\ a_2+t \hat a\ +\ o(t),\ \hat a\ =\ -\bar D_aG(a_2)^{-1}\int G(a_2,\go){\bf 1}_{\gO_1}d(\gamma|a_1)
\]
and hence
\begin{equation}\label{aux11}
\begin{aligned}
&0\ge t\int (W(\tilde a_*(t),\go)-W(a_1,\go)+x(a_1)^\top G(a_1,\go)){\bf 1}_{\gO_1} d(\gamma|a_1)\\
&+(1-\kappa t)\int (W(\tilde a_*(t),\go)-W(a_2,\go)+x(a_2)^\top G(a_2,\go))d(\gamma|a_2)\\
&=\ t\int (W(a_2,\go)-W(a_1,\go)+x(a_1)^\top G(a_1,\go)){\bf 1}_{\gO_1} d(\gamma|a_1)\\
&+t\bar D_aW(a_2)\hat a\ +\ o(t)\\
&=\ t\int (W(a_2,\go)-W(a_1,\go)+x(a_1)^\top G(a_1,\go)){\bf 1}_{\gO_1} d(\gamma|a_1)\\
&-t\bar D_aW(a_2)\bar D_aG(a_2)^{-1}\int G(a_2,\go){\bf 1}_{\gO_1}d(\gamma|a_1)\ +\ o(t)\\
&=\ t\int(c(a_1,\go;x)-c(a_2,\go;x)){\bf 1}_{\gO_1}d(\gamma|a_1)\ +\ o(t)\,.
\end{aligned}
\end{equation}
Since $\gO_1$ is arbitrary, we get that 
\[
c(a_1,\go;x)\ \le\ c(a_2,\go;x)
\]
almost surely with respect to $\gamma|a_1.$ 
\end{proof}

\section{Properties of Optimal Policies} 

Everywhere in the sequel, we assume that $W(a,\go)$ only depends on $a$ and that $G(a,\go)=a-g(\go)$ for some Borel-measurable map $g:\ \R^L\to \R^L.$ 

We define
\begin{equation}\label{cab}
c(a,b)\ =\ W(b)\ -\ W(a)\ +\ D_aW(a)\,(a-b)\,.
\end{equation}
As one can see from \eqref{cab},  $c$ coincides with the classic Bregman divergence that plays an important role in convex analysis (see, e.g., \cite{rockafellar1970convex}). We also define the {\it Bregman Projection} $\cP_\Xi$ onto a set $\Xi$ via 
\begin{equation}\label{bregman-proj}
\cP_\Xi(b)\ =\ \arg\min_{a\in\Xi} c(a,b)\,.
\end{equation}
In other words, $\cP_\Xi$ projects $b$ onto the point $a\in \Xi$ that attains the lowest Bregman  divergence.

Understanding further fine properties of optimal policies will require deriving subtle properties of the dimensions of the state that get compressed. 
We will also need the following definitions. 

\begin{definition} Let $Pool(a)=\supp(\gamma(a,\go|a))$ be the set of states $\go$ compressed to the same representation, $a.$ 
\end{definition}

\begin{definition} For any subset $X\subset\R^L,$ we denote by $conv(X)$ the convex hull of $X.$ That is, the smallest convex set containing $X.$ 
\end{definition}

We will now use first order conditions (Proposition \ref{delta-dev}) to derive useful properties of Pools and the support $\Xi$ of $\gamma(a,\R)$ (the optimal feature manifold). 

\begin{lemma}\label{key-geometry} For $\gamma\times \gamma$ almost every $(a_1,a_2)\in \Xi$, and any $x_i\in conv(g(Pool(a_i))$ and any $t\in [0,1],$ we have 
\begin{equation}\label{1st-i}
W(tx_1+(1-t)x_2)-t W(x_1)-(1-t)W(x_2)\ +\ t c(a_1,x_1)+(1-t)c(a_2,x_2)\ \le\ 0\,. 
\end{equation}
In particular, since $a_i\in conv(g(Pool(a_i))$, we get 
\begin{itemize}
    \item 
\begin{equation}\label{2-i}
W(ta_1+(1-t)a_2)\ \le\ t W(a_1)+(1-t)W(a_2)
\end{equation}
for almost every $a_1,a_2\in \Xi$;

\item 
\begin{equation}\label{3-i}
c(a_1,a_2)\ \ge\ 0
\end{equation}
for almost all $a_1,a_2\in \Xi;$

\item 
\begin{equation}\label{4-i}
c(a,x)\ \le\ 0\ 
\end{equation}
for $x\in conv(g(Pool(a)),$ almost surely. 
\end{itemize}
\end{lemma}

\begin{proof}[Proof of Lemma \ref{key-geometry}] Let $x_i=\int g(\go)d\eta_i(\go)$ where $\eta_i(\go)$ is absolutely continuous with respect to $\gamma(a_i,\go|a_i).$ Then, defining $\eta=t \eta_1 \gd_{a_1}+(1-t)\eta_2 \gd_{a_2},$ we get \eqref{1st-i}. 

Inequality \eqref{2-i} follows by setting $x_i=a_i$ (note that $a_i=\mathbb{E}[g(\go)|a_i]$ and hence $a_i\in conv(g(Pool(a_i)))$). 

Inequality \eqref{3-i} follows from \eqref{2-i} by first order Taylor approximation around $t=0.$ 

Inequality \eqref{4-i} follows from \eqref{1st-i} by setting $x_1=x_2,\ a_1=a_2.$ 
\end{proof}

\section{Optimal Feature Manifold} 

\begin{definition} Given an optimal autoencoder, let $\gamma(a,\go)$ be the joint distribution of $(a,\go).$ We call the support of $\gamma(a,\R)$ an optimal feature manifold. 

If $\gamma$ corresponds to a deterministic optimal autoencoder given by a feature map $a:\ \go:\ \cA(\go),$ then the optimal feature manifold $\Xi$ coincides with the support of the map, 
\begin{equation}
\Xi\ =\ \{\go\in \gO:\ p(a^{-1}(B_\eps(\cA(\go))))>0\ \forall\eps>0\}\,,
\end{equation}
where $B_\eps$ is an $\eps$-ball. 
\end{definition}

\begin{definition}\label{max-def}
A set $\Xi\subset \R^L$ is $X$-maximal if $\inf_{a\in \Xi}c(a,b)\le 0$ for all $b\in X.$ A set $\Xi$ is $W$-monotone if $c(a_1,a_2)\ge 0$ for all $a_1,a_2\in \Xi.$ A set $\Xi$ is $W$-convex if $W(ta_1+(1-t)a_2)\le t W(a_1)+(1-t)W(a_2)$ for all $a_1,a_2\in\Xi,\ t\in [0,1].$
\end{definition}

We now state the first important result of this section: A deterministic optimal autoencoder always exists; and any deterministic optimal autoencoder is a (Bregman) projection. 

\begin{theorem}[Optimal Policies are Projections onto an Optimal Feature Manifold] \label{cor-moment}  We have $a\in \cP_\Xi(g(\go))$ for $\gamma$-almost every $(a,\go).$ In particular, for a deterministic optimal autoencoder, we have $\cA(\go)\in \cP_\Xi(g(\go))$ Lebesgue-almost surely. 

Any optimal feature manifold is $conv(g(\gO))$-maximal, $W$-convex, and $W$-monotone. 
\end{theorem}

\begin{proof}[Proof of Theorem \ref{cor-moment}] The first claim follows directly from Lemma \ref{super-G}. The second claim follows directly from Lemma \ref{key-geometry}. 
\end{proof}

The converse is also true. 

\begin{theorem}[Maximality is both necessary and sufficient] \label{converse} Let $\Xi$ be a $conv(g(\gO))$-maximal subset of $\R^L$. Suppose that there exists an feature map $\cA(\go)$ such that  $\cA(\go)\ \in\ \cP_\Xi(g(\go))$ and $a\ =\ \mathbb{E}[g(\go)|\cA(\go)=a]$ for $\gamma$ almost every $(a,\go).$ Then, $a$ is an optimal autoencoder. 
\end{theorem}

\begin{proof}[Proof of Theorem \ref{converse}] The proof of sufficiency closely follows ideas from \cite{kramkov2019optimal}. 

Let $\cA(\go)$ be a policy satisfying the conditions Theorem \ref{converse}. Note that $a=\mathbb{E}[g(\go)|a]$ and therefore in terms of the function $c,$ 
\[
\mathbb{E}[W(a)]\ =\ \mathbb{E}[W(g(\go))-c(a,g(\go))]\,.
\]
Thus, maximizing $\mathbb{E}[W(a)]$ is equivalent to minimizing $c(a,g(\go)).$ Our objective is thus to show that 
\[
\min_{\gamma}\mathbb{E}[c(a,g(\go))]\ =\ \mathbb{E}[c(\cA(\go),g(\go))]. 
\]
Next, we note that the assumed maximality implies that $c(a,g(\go))\ =\ \cP_\Xi(g(\go))\le 0$ $\gamma$-almost surely. Now, for any feasible policy $\tilde\gamma$  we have $\mathbb{E}_{\tilde\gamma|b}[g(\go)|b]=b\in conv(g(\gO))$ and therefore $\mathbb{E}_{\tilde\gamma|b}[D_aW(a)\,(b-g(\go)]=\mathbb{E}_{\tilde\gamma|b}[D_aW(b)\,(b-g(\go)]=0$  for any fixed $a\in \R^L,$ and we have 
\begin{equation}
\begin{aligned}
&\mathbb{E}_{\tilde\gamma|b}[c(a,g(\go))\ -\ c(b,g(\go))]\\
& =\ \mathbb{E}_{\tilde\gamma|b}[W(g(\go))\ -\ W(a)\ +\ D_aW(a)\,(a-b+b-g(\go))\\
&-(W(g(\go))\ -\ W(b)\ +\ D_aW(b)\,(b-g(\go)))|b]\\
&\ =\ W(b)-W(a)+D_aW(a)(a-b)\ =\ c(a,b)\,.
\end{aligned}
\end{equation}
Taking the infinum over a dense, countable set of $a,$ we get 
\[
\inf_{a\in\Xi} \mathbb{E}_{\tilde\gamma|b}[c(a,g(\go))\ -\ c(b,g(\go))|b]\ =\ \inf_{a\in\Xi}  c(a,b)\le 0
\]
by the maximality of $\Xi$ and therefore 
\begin{equation}\label{last-c}
\begin{aligned}
&\mathbb{E}_{\tilde\gamma|b}[c(\cA(\go),g(\go))\ -\ c(b,g(\go))]\ =\ \mathbb{E}_{\tilde\gamma|b}[\inf_{a\in \Xi}c(a,g(\go))\ -\ c(b,g(\go))]\\ 
&\le\ \inf_{a\in\Xi}\mathbb{E}_{\tilde\gamma|b}[c(a,g(\go))\ -\ c(b,g(\go))]= \inf_{a\in \Xi} c(a,b)\le 0\,.
\end{aligned}
\end{equation}
Therefore, integrating over $b$ under the $\tilde\gamma$-policy and using that $\tilde\gamma(\R^L,\go)$ coincides with $p(\go),$ we get  
\begin{equation}\label{last-c1}
\mathbb{E}[c(\cA(\go),g(\go))]\ \le\  \mathbb{E}_{\tilde\gamma}[c(b,g(\go))]\,. 
\end{equation}
The proof is complete. 
\end{proof}

Theorem \ref{converse} is an important verification result that allows us to verify if a candidate solution is indeed an optimal autoencoder. 

\begin{proposition}\label{uniqueness} Let $\cA(\go)$ be an optimal feature map with support $\Xi:\ \gamma(\Xi,\R^L)=1.$  Let also $Q_\Xi=\{b\in \R^L:\ \inf_{a\in\Xi}c(a,b)=0\}.$ Then, $\Xi\subseteq Q_\Xi.$ Furthermore,  if $\tilde \gamma$ is another optimal autoencoder with support $\tilde\Xi,$ then   $\tilde\Xi\subseteq Q_{\Xi}$.
	
Thus, if $\Xi=Q_\Xi$ and $\arg\min_{a\in \Xi}c(a,b)$ is a singleton for all $b\in conv(g(\gO)),$ then the deterministic optimal autoencoder is unique. 
\end{proposition}

We conjecture that the conditions in Proposition \ref{uniqueness} hold generically and hence optimal autoencoder is unique for generic $W.$ 

\begin{proof}[Proof of Proposition \ref{uniqueness}] Since $\Xi$ is $W$-monotone, we have $c(a,b)\ge 0$ for all $a,b\in \Xi$ and hence $ \inf_{a\in\Xi}c(a,b)\ge 0$ for all $b\in \Xi$ and hence $\Xi\subset Q_\Xi.$ Let now  $\tilde\gamma$ be another optimal autoencoder. Then, by \eqref{last-c} and \eqref{last-c1}, we have 
\[
\mathbb{E}[c(\cA(\go),g(\go))]\ \le\  \mathbb{E}_{\tilde\gamma}[c(b,g(\go))]\ +\ \mathbb{E}_{\tilde\gamma}[\inf_{a\in \Xi} c(a,b)]
\]
and $\inf_{a\in \Xi} c(a,b)\le 0.$ Optimality of $\tilde\gamma$ implies 
\[
\mathbb{E}[c(\cA(\go),g(\go))]\ =\ \mathbb{E}_{\tilde\gamma}[c(b,g(\go))]\,. 
\]
Thus, $\inf_{a\in \Xi} c(a,b)=0,$\ $\tilde\gamma$-almost surely. That is, $\tilde\Xi\subseteq Q_\Xi.$ 
	
If $\Xi=Q_\Xi,$ we get that $\tilde\Xi\subseteq \Xi$ and hence $\inf_{a\in \Xi} c(a,b)\le \inf_{a\in \tilde\Xi} c(a,b)$ for all $b.$  Suppose now that $\tilde\gamma$ comes from a deterministic autoencoder, and let $b(\go)$ be the corresponding optimal feature map. Then, 
	\begin{equation}
	\begin{aligned}
&\int c(b(\go),g(\go))p(\go)d\go\ =\ \int\inf_{a\in \tilde\Xi} c(a,g(\go))p(\go)d\go\\ 
&\ge \int\inf_{a\in \Xi} c(a,g(\go))p(\go)d\go\ =\ 	\int c(\cA(\go),g(\go))p(\go)d\go\,. 
\end{aligned}
	\end{equation}
	Since both policies are optimal, we must have 
	$\inf_{a\in \Xi} c(a,g(\go))= \inf_{a\in \tilde\Xi} c(a,g(\go))$ almost surely, 
	 and the singleton assumption implies that $\cA(\go)=b(\go)$ almost surely.  
\end{proof}

We now prove our first main result. 

\begin{theorem}[$\Xi$ is a lower-dimensional manifold] \label{thm-unif} Let $\Xi$ be an optimal feature manifold (the support of an optimal autoencoder) and $\nu(a)=\nu(D_{aa}W(a))$ be the local degree of convexity of $W$. Then, for any open set $B,$ $\Xi\cap B$ is a subset of a Lipschitz manifold of dimension at most $\sup_{a\in B} \nu(a).$
\end{theorem}

\begin{proof}[Proof of Theorem \ref{thm-unif}] The proof only uses one property of $\Xi$: the fact that $\Xi$ is a $W$-monotone set. We have for any $a_1,a_2\in \Xi$ that 
\begin{equation}
0\ \le\ c(a_1,a_2)\ =\ 0.5 (a_2-a_1)^\top D_{aa}W(a_1) (a_2-a_1)\ +\ O(\|a_1-a_2\|^3)\,. 
\end{equation}
Picking a sufficiently small ball, we may assume that 
\[
0.5 (a_2-a_1)^\top (D_{aa}W(a_1)+\eps I) (a_2-a_1)\ \ge\ 0
\]
Diagonalizing $D_{aa}W(a_1)$, let $(\gl_1,\cdots,\gl_{M-\nu_+})$ be its strictly negative eigenvalues. Let $a=(a^+,a^-)$ with $a^-$ being of dimension $M-\nu_+$ be the corresponding orthogonal decomposition.  Then, picking $\eps$ sufficiently small, we get that there exist constants $K_1,K_2>0$ such that 
\[
K_2\|a^+_1-a^+_2\|^2-K_1\|a^-_1-a^-_2\|^2\ \ge\ 0
\]
for all $a_1,a_2\in \Xi\cap B_\eps(a_1).$ This condition  immediately implies the existence of a map $f:\ \R^{\nu_+}\to\R^{M-\nu_+}$ such that $a^- = f(a^+)$ for all $a\in \Xi$ because the coincidence of $a^+_1,a^+_2$ always implies the coincidence of $a^-_1,a^-_2$. Furthermore, this condition implies that $f$ is Lipschitz-continuous with the Lipschitz constant of at most $K_2/K_1$. The classic \cite{kirszbraun1934zusammenziehende} theorem implies that $f$ can always be extended to the whole $\R^{\nu_+}.$ Thus, {\it $\Xi$ is a subset of a $\nu_+$-dimensional Lipschitz manifold.} 
\end{proof}

We now proceed to showing when the upper bound on the dimension is exact. 

\begin{proposition}\label{regularity} Let $\Xi$ be an optimal feature manifold. Suppose that an $a_0$ is such that $D_{aa}W(a_0)$ is non-degenerate. Suppose also hat $L=M$ and the map $g:\gO\to\R^L$ is $C^1$ almost surely regular. That is, $D_\go g(\go)$ is Lebesgue-almost surely non-degenerate.\footnote{For example, this is the case if $g$ is real analytic and $D_\go g(\go)$ is non-degenerate in at least one point.}

Let $a_0$ be such that for all $\eps>0,$ $\cup_{a\in B_\eps(a_0)}Pool(a)$ has positive Lebesgue measure. Then, for sufficiently small $\eps>0,$ $\Xi\cap B_\eps(a_0)$ has Hausdorff dimension of exactly $\nu(a_0).$ 
\end{proposition}

\begin{proof}[Proof of Proposition \ref{regularity}]

\begin{lemma}\label{conv-hull-low} For $\gamma$-almost every $a$ there exists a convex subset $\hat\cX(a)$ and a Borel-Measurable $r(a)>0$ such that 
\begin{itemize}
    \item $g^{-1}(a+\hat\cX(a))$ is a support of $\gamma(a,\go|a).$
    
    \item for any $y\in \hat\cX(a)$ and any $\gd,\ |\gd|<r(a)$ we have $\gd y\in \cX(a)$

\end{itemize}

\end{lemma}

\begin{proof}[Proof of Lemma \ref{conv-hull-low}]  Let $\cX(a)\ =\ conv(g(Pool(a))-a.$ By the definition of $Pool(a),$ $g^{-1}(\cX(a)+a)$ is a support of $\gamma|a.$
If $0$ is in the interior of $\cX(a)$,  then we are done. Suppose the contrary. Then, $0$ is at the boundary of $\cX(a)$ and hence there is a supporting hyperplane of the convex set $\cX(a)$ that passes through it. Let $v$ be the normal to this supporting hyperplane. Then, 
\begin{equation}
0\ =\ \int_{\cX(a)}(g(\go)-a)\eta(\go) d\gamma(a,\go|a)
\end{equation}
for some density $\eta(\go)$, which implies 
\begin{equation}
0\ =\ \int_{\cX(a)}(v^\top g(\go)-v^\top a)\eta(\go)d\gamma(a,\go|a)\,.
\end{equation}
Thus, $g(\go)$ belongs $\gamma(a,\go|a)$-almost surely to the supporting hyperplane $\cC,$ and hence, by the convexity of a hyperplane, $\cX(a)\cap\cC$ is convex and $g^{-1}(a+(\cX(a)\cap\cC))$ is a support of $\gamma|a.$
Repeating this argument, we continue dimension reduction until we get a lower-dimensional convex subset of $\hat\cX(a)\subset \cX(a)$ on which $\gamma(a,\go|a)$ is supported and such that $0$ is in its interior. 
 Then, there exists an $r(a)$ such that $y \gd\in \hat\cX(a)-a$ for all $\gd$ with $|\gd|\le r(a).$  
\end{proof}

Note that we may assume without loss of generality that the inner radius function $r(a)$ is uniformly bounded away from zero. Otherwise, we just pass to a coverage of $\Xi$ by subsets $\Xi_i$ where $\Xi_i=\{a\in \Xi:\ r(a)>1/i\}.$ Most importantly, up to rescaling, this implies that we may assume that $y,-y$ both belong to $\hat\cX(a)-a$ for each $a$ when $\|y\|$ is sufficiently small.  

 Pick an $a_i\in \Xi\cap B_\eps(a_0),\ i=1,2.$ By \eqref{1st-i}, for any $a_1$ and $a_2$ and any $x_i\in conv(g(a^{-1}(a_i))$ and any $t_i\in [0,1],\ t_1+t_2=1$ we have 
\begin{equation}\label{28-1}
\begin{aligned}
&W(t_1 x_1+t_2 x_2)\ +\ \sum_i t_i (D_aW(a_i)(a_i-x_i)-W(a_i))\ \le 0
\end{aligned}
\end{equation}
Since $a_i\in conv(g(Pool(a_i)),$\footnote{Indeed, $a_i=\mathbb{E}[g(\go)|\go \in Pool(a_i)].$} 
we have that the whole interval $a_i(1-\eps_1)+\eps_1 x_i\in conv(g(Pool(a_i)),\ \eps_1\in [0,1].$ 

Let $a=t_1a_1+t_2 a_2.$  We choose $t_1=t_2=0.5.$ Then, Taylor approximation plus the two-times continuous differentiability imply 
\begin{equation}
\begin{aligned}
&W(a_1)\ =\ W(a)\ +\ 0.5 D_aW(a)(a_1-a_2)\ +\ 2^{-3} (a_1-a_2)^\top D_{aa}W(a)(a_1-a_2)\ +\ o(\|a_1-a_2\|^2)\\
&W(a_2)\ =\ W(a)\ +\ 0.5 D_aW(a)(a_2-a_1)\ +\ 2^{-3} (a_1-a_2)^\top D_{aa}W(a)(a_1-a_2)+\ o(\|a_1-a_2\|^2)\\ 
&D_aW(a_1)(a_1-x_1)\ =\ D_aW(a)(a_1-x_1)\ +\ 0.5 (a_1-a_2)^\top D_{aa}W(a) (a_1-x_1)\ +\ O(\|a_1-a_2\|^2)\\
&D_aW(a_2)(a_2-x_2)\ =\ D_aW(a)(a_2-x_2)\ +\ 0.5 (a_2-a_1)^\top D_{aa}W(a) (a_2-x_2)\ +\ O(\|a_1-a_2\|^2)
\end{aligned}
\end{equation}
and therefore 
\begin{equation}
\begin{aligned}
&\sum_i t_i (D_aW(a_i)(a_i-x_i)-W(a_i))\\ 
&=\ O(\|a_1-a_2\|^2)\ -W(a)+D_aW(a) (0.5(a_1-x_1)+0.5(a_2-x_2))\\
&+0.25 (a_1-a_2)^\top D_{aa}W(a) ((a_1-x_1)-(a_2-x_2))\,. 
\end{aligned}
\end{equation}
Furthermore, 
\begin{equation}
\begin{aligned}
&W(t_1 x_1+t_2 x_2)\ =\ W(a)\ +\ D_aW(a) (0.5(x_1-a_1)+0.5(x_2-a_2))\\ 
&+\ 0.5 (0.5(x_1-a_1)+0.5(x_2-a_2))^\top D_{aa}W(a) (0.5(x_1-a_1)+0.5(x_2-a_2))\\
&+\ o(\|0.5(x_1-a_1)+0.5(x_2-a_2)\|^2)\,.
\end{aligned}
\end{equation}
Let $y_i=x_i-a_i.$ Then, \eqref{28-1} takes the form 
\begin{equation}
\begin{aligned}
&2^{-3} (y_1+y_2)^\top D_{aa}W(a) (y_1+y_2)\\
&+\ o(\|y_1+y_2\|^2)+0.25 (a_1-a_2)^\top D_{aa}W(a) (y_2-y_1)+O(\|a_1-a_2\|^2)\ \le\ 0\,. 
\end{aligned}
\end{equation}
for all $y_i\in \cX(a_i).$ 

Pick arbitrary $y_i\in \hat\cX(a_i).$ Lemma \ref{conv-hull-low} above implies that we may  assume without loss of generality that $\pm y_i\in \hat\cX(a_i)$ and, hence, we get 
\begin{equation}
\begin{aligned}
&2^{-3} (y_1-y_2)^\top D_{aa}W(a) (y_1-y_2)\\
&+\ o(\|y_1-y_2\|^2)+0.25 (a_1-a_2)^\top D_{aa}W(a) (-y_2-y_1)+O(\|a_1-a_2\|^2)\ \le\ 0\\
&2^{-3} (y_1-y_2)^\top D_{aa}W(a) (y_1-y_2)\\
&+\ o(\|y_1-y_2\|^2)+0.25 (a_1-a_2)^\top D_{aa}W(a) (y_1+y_2)+O(\|a_1-a_2\|^2)\ \le\ 0
\end{aligned}
\end{equation}
Summing up these two inequalities, we get 
\begin{equation}\label{bound-a}
\begin{aligned}
&2^{-3} (y_1-y_2)^\top D_{aa}W(a) (y_1-y_2)\ +\ o(\|y_1-y_2\|^2)+O(\|a_1-a_2\|^2)\ \le\ 0\,. 
\end{aligned}
\end{equation}
By the non-degeneracy of $D_{aa}W(a),$ rotating and re-scaling the coordinates, we may assume that $D_{aa}W(a) =\diag({\bf 1}_\nu,-{\bf 1}_{M-\nu}).$ By the three time continuous differentiability of $W(a),$  Defining the corresponding decomposition $y=(y^+,y^-)\in \R^+\oplus \R^-,$ we get that \eqref{bound-a} is equivalent to the existence of a sufficiently large constant $K$ such that 
\begin{equation}\label{main-bound}
\|y_1^+ - y_2^+\|^2\ \le\ K(\|y_1^--y_2^-\|^2\ +\ \|a_1-a_2\|^2)
\end{equation}
for any $y_i\in \hat\cX(a_i)$ such that $\|y_1-y_2\|$ is sufficiently small and $\|a_1-a_2\|$ is sufficiently small. As a result, 
\begin{equation}
\|y_1-y_2\|^2\ \le\ \tilde K(\|y_1^--y_2^-\|^2\ +\ \|a_1-a_2\|^2)
\end{equation}
for any $y_i\in \cX(a_i)$ (by re-scaling $y_i$ is necessary). 

Suppose now towards a contradiction that $\Xi\cap B_\eps(a_0)$ has Hausdorff dimension strictly smaller than $\nu.$ Let $Y\ =\ \cup_{a\in \Xi\cap B_\eps(a_0)}(\hat\cX(a)+a).$  Our claim is that the inequality \eqref{main-bound} implies that $\dim_H(Y)<M.$ Indeed, \eqref{main-bound} implies the existence of a Lipschitz map from $(\Xi\cap B_\eps(a_0))\times \R^-$ onto $Y$. First, we have 
\[
\dim_H(X\times \R^{M-\nu})\ \le\ M-\nu+\dim_H(X)
\]
for any Borel set $X.$ Second, sine Lipschitz maps cannot increase Hausdorff dimension, we get the required claim $\dim_H(Y)<M.$ Thus, $Y$ has zero Lebesgue measure. Our next observation is that $g^{-1}(Y)$ has zero Lebesgue measure. Indeed, suppose the contrary.  By assumption, removing a set of measure zero, we may assume that $D_\go g(\go)$ is non-degenerate for all $\go\in g^{-1}(Y).$ Then, by the implicit function theorem, there exists an $\eps>0$ such that $g$ is a diffeomorphism on $B_\eps(\go)\cap g^{-1}(Y)$ and hence $g(B_\eps(\go)\cap g^{-1}(Y))\subset Y$ has a positive Lebesgue measure, leading to a contradiction. 

As we have shown above, $\gamma(\R,\go)$ is supported on $g^{-1}(Y)$ and hence it cannot coincide with $p.$ The proof of Proposition \ref{regularity} is complete. 
\end{proof}

\section{Properties of Pools: Which Dimensions Get Compressed?} 

Recall that $Pool(a)$ is the set of states $\go$ that get the same representation $a.$ The following is true.

\begin{proposition}[(Convexity of pools)] \label{main convexity} 
Suppose that $g(\go)=\go$ and $\gO$ is convex. Suppose also that $D_aW$ satisfies the technical condition of Theorem \ref{mainth-limit}.  Then, there exists a deterministic optimal autoencoder $\cA(\go)$ such that the map $\go\to D_aW(\cA(\go))$ is monotone increasing on $\gO$\footnote{In fact, $c(\cA(\go),\go)$ is convex on $\gO$ and $D_aW(\cA(\go))$ is a subgradient of $c(\cA(\go),\go).$} and the set 
\begin{equation}\label{pool-union}
\{\go\in\gO:\ D_aW(\cA(\go))=a\}\ =\ \cup_{b\in (D_aW)^{-1}(a)}Pool(b)
\end{equation}
is always convex. If the map $a\to D_aW(a)$ is injective, then the pool of every signal is convex (up to a set of measure zero)\footnote{The last claim follows because level sets for a monotone map are convex.} and 
$\cA(\go)$ is an idempotent: $\cA(\cA(\go))\ =\ \cA(\go).$ 
\end{proposition}

\begin{proof}[Proof of Proposition \ref{main convexity}] The proof follows directly from the proof of Theorem \ref{mainth-limit} because the $c(\cA(\go),\go)$ constructed in that proof is convex, and $D_aW(\cA(\go))$ is its sub-gradient. 
\end{proof}

We would now like to understand the fine properties of signal pools. Which states get pooled together? Is there an analytical way to describe $Pool(a)$ for a given $a?$
Theorem \ref{thm-unif} only implies that we can characterize $\Xi$ as $\Xi=\{a\in \R^L:\ a=f(\theta),\ \theta\in \Theta\},$ where $\Theta \subset\R^\nu$ is a lower-dimensional subset with {\it unknown properties} and $f$ is a Lipschitz map. Rewriting \eqref{bregman-proj} as
\begin{equation}\label{attempt-aa}
\cA(\go)\ =\ \arg\min_{\theta\in \Theta} c(f(\theta), g(\go))\,,
\end{equation}
one might be tempted to differentiate \eqref{attempt-aa} with respect to $\theta.$ Indeed, as $f$ is Lipschitz continuous, it is differentiable Lebesgue-almost everywhere by the Rademacher Theorem.\footnote{See, e.g., \cite{cheeger1999differentiability}.} However, differentiation in \eqref{attempt-aa} is only possible if the set $\Theta$ is ``sufficiently rich", extending ``in all possible directions."  Establishing richness is extremely difficult. In this section, we use techniques from geometric  measure theory to achieve this goal. Intuitively, Corollary \ref{dimension-pool} tells us that each $Pool(a)$ has dimension $M-\nu$ and hence $\Xi$ ought to have dimension $\nu$ because $\gO\ =\ \cup_{a\in \Xi}Pool(a).$ 
This Hausdorff dimension result gives enough richness to perform differentiation in \eqref{attempt-aa}. The following is true.

\begin{corollary}[Pools are low-dimensional sets]\label{dimension-pool} Let $conv(g(Pool(a)))$ be the smallest convex set containing $\supp(\gamma|a):$
\begin{equation}
conv(g(Pool(a)))\ =\ \cap_{Borel\ X:\ \gamma(a,X|a)=1}conv(X)\,. 
\end{equation}
Then we have $dim(conv(g(Pool(a))))\ \le\ M-\nu_+(D_{aa}W(a)).$ In particular, if $D_{aa}W(a)$ has at least one strictly positive eigenvalue for any $a,$ then  $conv(g(Pool(a)))$ has Lebesgue measure zero. If $g$ is locally injective and bi-Lipschitz,\footnote{$g$ is locally bi-Lipschitz if both $g$ and its local inverse, $(g|_X)^{-1}$ are Lipschitz for any compact set $X.$}then $Pool(a)$ also has Lebesgue measure zero for each $a.$ 
\end{corollary}

\begin{proof}[Proof of Corollary \ref{dimension-pool}] Let $\hat\cX(a)+a$ be the convex set constructed in Lemma \ref{conv-hull-low}. Then, clearly, $\hat\cX(a)+a=conv(g(Pool(a)))).$ By \eqref{4-i}, we have 
\begin{equation}
W(x)-W(a)+D_aW(a)(a-x)\ \le\ 0
\end{equation}
for all $x\in conv(Pool(a))$ and $a\in conv(Pool(a)).$ Let $y=\eps x+(1-\eps)a\in conv(Pool(a)).$ Then, using the Taylor approximation, we get 
\begin{equation}
(y-a)^\top D_{aa}W(a)(y-a)\ \le\ 0
\end{equation}
and the definition of the set $\hat\cX(a)$ implies that, in fact, $z^\top D_{aa}W(a)z\le 0$ for all $z$ is the minimal subspace of $\R^L$ containing $\hat\cX.$ The eigenvalue interlacing theorem implies that the dimension of this subspace is less than or equal to $M-\nu_+(D_{aa}W(a)).$ The proof is complete. 
\end{proof}

As we know from Theorem \ref{thm-unif}, the support of $\gamma(a,\R^L)$ (the optimal feature manifold) is a subset of a Lipschitz manifold of dimension at most $\sup \nu(D_{aa}(W)).$ That is, for every $a\in \Xi,$ there exists an $\eps>0$,  a subset $\Theta\subset\R^\nu,$ and a Lipschitz coordinate map 

\begin{corollary}[Characterization of Pools]\label{char-pools} Let $\cA(\go)$ be a deterministic optimal autoencoder and $\Xi$ the corresponding optimal feature manifold. Suppose $a_0\in \Xi$ is such that the technical conditions of Proposition  \ref{regularity} are satisfied. Let $f:\R^\nu\to \Xi$ be local Lipschitz coordinates from Theorem \ref{thm-unif} in a small neighborhood of $a_0,$ and let $\Theta=f^{-1}(\Xi\cap B_\eps(a_0)).$ Then, for Lebesgue-almost every  $\theta\in \Theta,$ $f$ is differentiable, with a Jacobian $Df(\theta)\in \R^{M\times\nu},$ and we have 
\begin{itemize}
\item[(1)] the matrix $Df(\theta)^\top D_{aa}W(f(\theta)) Df(\theta)\in \R^{\nu\times\nu}$ is Lebesgue-almost surely symmetric and positive semi-definite. 

\item[(2)] Lebesgue-almost every $(\go,\theta)$ satisfies 
\begin{equation}\label{downward-sloping-pools} 
Df(\theta)^\top D_{aa}W(f(\theta))(f(\theta)-g(\go))\ =\ 0\,\,,
\end{equation}
\end{itemize}
when $\go\in Pool(f(\theta)).$
\end{corollary}

\begin{proof}[Proof of Corollary \ref{char-pools}]  
	We will need 
	
	\begin{lemma}\label{lem7} For any set $Y\subset\R^\nu$ of Hausdorff dimension bigger than $\nu-1.$\footnote{For example, a set $Y\subset\R^\nu$ of positive Lebesgue measure has Hausdorff dimension $\nu.$} Then, the closure of the set $(\{(x-y)/\|x-y\|:\ x,\ y\in Y\})$ coincides with the unit sphere in $\R^\nu.$  
		
		Furthermore, the complement of the set 
		\begin{equation}
			Y^*=\{y\in Y:\ \forall\ \ga,\ \|\ga\|=1\ \exists y_k\in Y,\ y_k\to y, \ga^\top (y-y_k)/\|y-y_k\|\to 1\}
		\end{equation}
		has Hausdorff dimension less than or equal to $\nu-1.$ 
	\end{lemma}
	
	\begin{proof} Our proof is based on an application of the famous Frostman's lemma (see, e.g., \cite{mattila1999geometry}). 
		
		\begin{lemma}[Frostman's lemma] 
			Define the $s$-capacity of a Borel set $A$ as follows:
			\[
			C_s(A)\ =\ \sup\left\{
			\left(
			\int_{A\times A}\frac{d\mu(x)d\mu(y)}{\|x-y\|^s}
			\right)^{-1}:\ \mu\ \text{is a Borel measure and $\mu(A)=1$}
			\right\}\,.
			\]
			(Here, we take $\inf\emptyset=\infty$ and $1/\infty=0.$) Then, the Hausdorff dimension $\dim_H(A)$ is given by 
			\[
			\dim_H(A)\ =\ \sup\{s\ge 0:\ C_s(A)\ >\ 0\}\,. 
			\]
		\end{lemma}
		
		We will need a small modification of this lemma: 
		
		\begin{lemma}[Modified Frostman's Lemma] \label{mod-frost}
			Define 
			\[
			C_s(A;\eps)\ =\ \sup\left\{
			\left(
			\int_{A\times A, \|x-y\|\le \eps}\frac{d\mu(x)d\mu(y)}{\|x-y\|^s}
			\right)^{-1}:\ \mu\ \text{is a Borel measure and $\mu(A)=1$}
			\right\}\,.
			\]
			(Here, we take $\inf\emptyset=\infty$ and $1/\infty=0.$) Then, the Hausdorff dimension $\dim_H(A)$ is given by 
			\[
			\dim_H(A)\ =\ \sup\{s\ge 0:\ C_s(A;\eps)\ >\ 0\}\,. 
			\]
		\end{lemma}
		
		\begin{proof}[Proof of Modified Frostman's Lemma] By direct calculation, 
			\[
			\int_{A\times A}\frac{d\mu(x)d\mu(y)}{\|x-y\|^s}\ =\ \int_{A\times A, \|x-y\|\le \eps}\frac{d\mu(x)d\mu(y)}{\|x-y\|^s}\ +\ \int_{A\times A,\ \|x-y\|>\eps }\frac{d\mu(x)d\mu(y)}{\|x-y\|^s}
			\]
			and, hence, 
			\begin{equation}
				\left(\eps^{-s}+
				\int_{A\times A, \|x-y\|\le \eps}\frac{d\mu(x)d\mu(y)}{\|x-y\|^s}
				\right)^{-1}\ \le\ 	\left(
				\int_{A\times A}\frac{d\mu(x)d\mu(y)}{\|x-y\|^s}
				\right)^{-1}\ \le\ 	\left(
				\int_{A\times A, \|x-y\|\le \eps}\frac{d\mu(x)d\mu(y)}{\|x-y\|^s}
				\right)^{-1}\,. 
			\end{equation}
			implying the required. 
		\end{proof}

		Clearly, the second part of the statement implies the first one. This second statement can be reformulated as follows: For a vast majority of points $y$ in $Y$ (e.g., a set of full Lebesgue measure; but, in fact, the actual statement is stronger: It means that the set of points outside of $Y^*$ is low-dimensional; it is effectively a form of a ``boundary" of $Y$), we have that for every $\ga$ on the unit sphere there exists a sequence $y_k\to y$ such that the direction of $y_k-y$ converges to one of $\pm \ga.$ 
		
		To prove the result, we first define the set 
		\begin{equation}
			Y_{\eps,\ga,\gd}\ =\ \{x\in Y:\ (x-y)^\top \ga /\|x-y\|\ \le 1-\gd\ \forall\ y\in Y\cap B_\eps(x)\}\,. 
		\end{equation}
		Pick an everywhere dense, countable subset $\Gamma$ of the unit sphere and two sequences $\eps_k\to 0,\ \gd_k\to 0.$ 
		
		\begin{lemma} We have 
			\begin{equation}
				(Y^*)^c\ =\ \cup_{\ga\in\Gamma }\cup_i\cup_j Y_{\eps_i,\ga,\gd_j}
			\end{equation}
		\end{lemma}
		
		\begin{proof} By definition, $y\in (Y^*)^c$ if and only if there exists an $\alpha$ on the unit sphere, and $\eps,\gd>0$ such that $(x-y)^\top \ga /\|x-y\|\ \le 1-\gd\ \forall\ y\in Y\cap B_\eps(x).$ By continuity and boundedness of $Y\cap B_\eps(x),$ we may assume that $\ga\in\Gamma$ and $\gd=\gd_i,\ \eps=\eps_j$ for some $i,j.$ The proof is complete. 
		\end{proof}
		
		Since the Hausdorff dimension of a countable union is bounded from above by the highest Hausdorff dimension, proving Lemma \ref{lem7} reduces to proving that $\dim_H(Y_{\eps,\ga,\gd})\ \le\ \nu-1.$ 
		
		So, let us fix $\eps,\ga,\gd$. Without loss of generality (rotating coordinates if necessary), we may assume that $\ga=e_1,$ the first basis vector. We also denote by $x_{-1},y_{-1}$ the projections of $x,y$ onto the orthogonal complement of $e_1.$ Consider two points $x,y\in  Y_{\eps,\ga,\gd}$ such that $\|x-y\|<\eps.$ Permuting the order of the two points if necessary, we may assume that $y_1>x_1.$ Then, by the definition of the set $Y_{\eps,\ga,\gd},$ we have 
		\begin{equation}
			0\le\ y_1-x_1\ =\ (y-x)^\top\ga\ <\ (1-\gd)\|y-x\|\ =\ (1-\gd)(|x_1-y_1|^2+\|x_{-1}-y_{-1}\|^2)^{1/2}\,, 
		\end{equation}
		which implies 
		\begin{equation}
			\|x_{-1}-y_{-1}\|\ \ge\ (1-(1-\gd)^2)^{1/2}|x_1-y_1|
		\end{equation}
		and, hence, there exists a constant $K>1$ such that 
		\begin{equation}
			\|x-y\|\ \le\ K \|x_{-1}-y_{-1}\|
		\end{equation}
		for all $x,y\in Y_{\eps,\ga,\gd}$ whenever $\|x-y\|\le \eps.$   Let $Y_{\eps,\ga,\gd}^{-1}$ be the projection of $Y_{\eps,\ga,\gd}$ onto the orthogonal complement of $e_1.$ Since this complement had dimension $\nu-1,$ the Hausdorff dimension of $Y_{\eps,\ga,\gd}^{-1}$ is at most $\nu-1.$ At the same time, 
		\begin{equation}
			\begin{aligned}
				&C_s(Y;\eps)\ =\ \sup\left\{
				\left(
				\int_{Y_{\eps,\ga,\gd}\times Y_{\eps,\ga,\gd},\ \|x-y\|\le \eps}\frac{d\mu(x)d\mu(y)}{\|x-y\|^s}
				\right)^{-1}:\ \mu\ \text{is a Borel measure and $\mu(Y)=1$}
				\right\}\\
				& \le\ \sup\left\{
				\left(
				\int_{Y_{\eps,\ga,\gd}\times Y_{\eps,\ga,\gd},\ \|x-y\|\le \eps}\frac{d\mu(x)d\mu(y)}{K \|x_{-1}-y_{-1}\|^s}
				\right)^{-1}:\ \mu\ \text{is a Borel measure and $\mu(Y)=1$}
				\right\}\\
				&\le\ \sup\left\{
				\left(
				\int_{Y_{\eps,\ga,\gd}^{-1}\times Y_{\eps,\ga,\gd}^{-1},\ \|x-y\|\le \tilde\eps}\frac{d\mu(x)d\mu(y)}{K \|x-y\|^s}
				\right)^{-1}:\ \mu\ \text{is a Borel measure and $\mu(Y_{\eps,\ga,\gd}^{-1})=1$}
				\right\}
			\end{aligned}
		\end{equation}
		for an appropriately rescaled $\tilde\eps.$  Hence, by Lemma \ref{mod-frost}, $\dim_H(Y_{\eps,\ga,\gd})\le \dim_H(Y_{\eps,\ga,\gd}^{-1})\le \nu-1.$ The proof of Lemma \ref{lem7} is complete.  
	\end{proof}

	By the Rademacher Theorem, $f(\theta)$ is almost everywhere differentiable, and hence we can assume that $f$ is differentiable on the whole of $\Theta.$  
	First, the fact that $\Xi$ is $W$-monotone means that 
	\begin{equation}
		c(f(\theta_1),f(\theta_2))\ \ge\ 0
	\end{equation}
	for all $\theta_1,\theta_2\in\Theta,$ where 
	\begin{equation}
		c(a,b)\ =\ W(b)\ -\ W(a)\ +\ D_aW(a)(a-b)\ =\ 0.5 (b-a)^\top D_{aa}W(a)(b-a)\ +\ o(\|b-a\|^2)\,.
	\end{equation}
	By Lemma \ref{lem7}, we may assume that $\theta\in \Theta^*.$ Hence, for any $\ga,\ \|\ga\|=1,$ there exists a a sequence $\theta_{k}\in\Theta$ such that $(\theta_k-\theta)^\top/\|\theta_{k}-\theta\|\to \alpha$. Therefore, with $b=f(\theta_{k}),\ a=f(\theta),$ we get 
	\begin{equation}\label{df}
		f(\theta_{k})-f(\theta)\ =\ DF(\theta)(\theta_{k}-\theta)+o(\|\theta_{k}-\theta\|)
	\end{equation}
	and therefore 
	\begin{equation}
		0\ \le\ \lim_{k\to\infty} c(f(\theta_k),f(\theta))/\|\theta_{k}-\theta\|^2\ =\ \ga^\top DF(\theta)^\top D_{aa}W(f(\theta))DF(\theta)\ga 
	\end{equation}
	for any $\ga$ on the unit sphere. Hence, the matrix $DF(\theta)^\top D_{aa}W(f(\theta))DF(\theta)$ is positive semi-definite. 
	
	We now proceed with deriving the first order condition. We have 
	\begin{equation}\label{attempt-aa-1}
		\cA(\go)\ =\ \arg\min_{\theta\in \Theta} c(f(\theta), g(\go))\,
	\end{equation}	
	and therefore 
	\begin{equation}
		\begin{aligned}
			&0\ \le\ c(f(\theta_k), g(\go))- c(f(\theta), g(\go))\\ 
			&=\  W(g(\go))\ -\ W(f(\theta_k))\ +\ D_aW(f(\theta_k))(f(\theta_k)-g(\go))\\
			&-(W(g(\go))\ -\ W(f(\theta))\ +\ D_aW(f(\theta))(f(\theta)-g(\go)))\\
		\end{aligned}
	\end{equation}
	Using Taylor approximations 
	\begin{equation}
		\begin{aligned}
			&W(g(\go))\ -\ W(f(\theta_k))\\
			& =\ -D_aW(f(\theta))(f(\theta_k)-f(\theta))-0.5 (f(\theta_k)-f(\theta))^\top D_{aa}W(f(\theta))(f(\theta_k)-f(\theta))\\
			&-o(\|(f(\theta_k)-f(\theta))\|^2)\\
			&D_aW(f(\theta_k))(f(\theta_k)-g(\go))\\ 
			&=\ (D_aW(f(\theta))+ (f(\theta_k)-f(\theta))^\top D_{aa}W(f(\theta))+O(\|(f(\theta_k)-f(\theta))\|^2))(f(\theta_k)-g(\go)),
		\end{aligned}
	\end{equation}
	we get 
	\begin{equation}
		\begin{aligned}
			&0\ \le\ W(f(\theta))-W(f(\theta_k))\ +\ D_aW(f(\theta_k))(f(\theta_k)-g(\go))-D_aW(f(\theta))(f(\theta)-g(\go))\\
			&=\ -D_aW(f(\theta))(f(\theta_k)-f(\theta))-0.5 (f(\theta_k)-f(\theta))^\top D_{aa}W(f(\theta))(f(\theta_k)-f(\theta))\\
			&-o(\|(f(\theta_k)-f(\theta))\|^2)\\
			&+\ (D_aW(f(\theta))+ (f(\theta_k)-f(\theta))^\top D_{aa}W(f(\theta))+O(\|(f(\theta_k)-f(\theta))\|^2))(f(\theta_k)-g(\go))\\
			&-D_aW(f(\theta))(f(\theta)-g(\go))\\
			&=\ O(\|(f(\theta_k)-f(\theta))\|^2)\ +\ (f(\theta_k)-f(\theta))^\top D_{aa}W(f(\theta)) (f(\theta_k)-g(\go))\,. 
		\end{aligned}
	\end{equation}
	Substituting \eqref{df}, dividing by $\|\theta_k-\theta\|$ and taking the limit as $k\to\infty,$ we get 
	\begin{equation}
		\alpha^\top  Df(\theta)^\top D_{aa}W(f(\theta)) (f(\theta)-g(\go))\ \ge\ 0\,. 
	\end{equation}
	Since this inequality holds for any $\alpha$ on the unit sphere, we get that 
	\begin{equation}
		Df(\theta)^\top D_{aa}W(f(\theta)) (f(\theta)-g(\go))\ =\ 0\,.
	\end{equation}
	The proof of Corollary \ref{char-pools} is complete. 
\end{proof}

\section{When Is the Optimal Representation Compact?} 

To gain some intuition, suppose first that $W$ is concave for large $a:$ $D_{aa}W$ is negative semi-definite for all $a$ with $\|a\|>K$ for some $K>0.$ Then, by Corollary \ref{dimension-pool}, any optimal feature manifold satisfies $\Xi\subset \{a:\ \|a\|\le K\}$ and is therefore bounded. 
Of course, concavity is a very strong condition. It turns out that the boundedness of optimal feature manifolds can be established under much weaker conditions. We will need the following definition. 

\begin{definition}\label{conc-rays} Let $\cC(a,\eps)=\{b\in \R^L:\ b^\top a/(\|a\|\,\cdot\|b\|)>1-\eps\}$ be the $\eps$-cone around $a:$ the set of vectors $b$ that point in approximately the same direction as $a.$  We say that the  function $W$ is concave along rays for large $a$ if there exists a small $\eps>0$ and a large $K>0$ such that $b^\top D_{aa}W(a)b\ <\ 0$ for all $a$ with $\|a\|>K$ and all $b\in \cC(a,\eps)).$ 

We also say that a set $\Theta\subset \R^\nu$ extends indefinitely in all directions if the projection of $\Theta $ on any ray from the origin is unbounded. 
\end{definition}

Note that if $W$ is quadratic, $W(a)=a^\top H a+h^\top a$, we have $D_{aa}W(a)=2H$ and hence $a^\top D_{aa}W(a)a=2a^\top H a.$ Thus, $W$ is concave along rays if and only if $W$ is globally concave, implying that it is optimal to fully compress information. As we show below, quadratic  function represents a knife-edge case, as even slight deviations from it may drastically alter the nature of optimal policies. The following is true. 

\begin{proposition}[Compact Representations] \label{conv-bound} Suppose that $conv(g(\gO))=\R^L$ and let $\Xi$ be an optimal feature manifold.
\begin{itemize}
\item If $W(a)=a^\top H a+h^\top a$ with $\det H\not=0$, then $\Xi=f(\Theta)$ for some Lipschitz $f:\R^{\nu(H)}\to\R^L,$ where $\Theta$ extends indefinitely in all directions;

\item If $W$ is concave along rays for large $a,$ then there exists a constant $K$ independent of the prior $p$, such that any optimal feature manifold satisfies $\Xi\subset B_K(0).$
\end{itemize}

\end{proposition}

Proposition \ref{conv-bound} shows how a weak form non-linearity of of the  function makes it optimal to  compress unbounded information to a bounded representation. The following claim follows by direct calculation from Proposition \ref{conv-bound}. 

\begin{corollary}\label{concave-bounded} Let $H$ be a non-degenerate, $M\times M$ positive-definite matrix. Suppose that $W(a)=\varphi(a'Ha)$ for some $\varphi$ with $-\varphi''(x)/|\varphi'(x)|>\eps$ for some $\eps>0$ and all sufficiently large $x.$ Then, $\nu(D_{aa}(W(a)))\ge M-1$ for all $a.$ Yet, $W$ is concave along rays for large $a$ and, hence, any optimal feature manifold is bounded, contained in a ball of radius $K$ that is independent of the prior $p.$ 
\end{corollary}

 Consider as an illustration $W(a)=\varphi(a_1^2+\gl a_2^2)$ and $g(\go)=\go.$ First, let $\varphi(x)=x.$ If $\gl<0,$ the first item of Proposition \ref{conv-bound} applies, and we get that $\Xi$ is the graph of a Lipschitz function that extends indefinitely in all directions. Making $\gl$ more negative will lead to a rotation of the optimal feature manifold, but will not alter its shape. Consider now a case when $\gl>0.$ If $\varphi(x)=x,$ $W(a)$ is convex and no compression is optimal: $\cA(\go)=\go.$ However, even a slight degree of concavity for $\varphi$ leads to information compression and a bounded representation of the states. The optimal feature manifold, $\Xi,$ is bounded and, hence, cannot be a graph of a function extending indefinitely. Instead, $\Xi$ is a bounded curve in $\R^2$ (e.g., a circle).

 \section{Proofs of Proposition \ref{conv-bound} and Corollary \ref{concave-bounded}}

 \begin{proof}[Proof of Proposition \ref{conv-bound}] Suppose first that $W$ is quadratic. Then, the claim follows directly from the maximality of $\Xi:$ If there exists a hyperplane such that $\Theta$ is on one side of it, then it is possible to extend $f$ preserving its Lipschitz constant beyond this hyperplane using the standard Lipschitz extension argument from the Kirszbraun theorem. See, \cite{kirszbraun1934zusammenziehende}.\footnote{We just pick one point on the other side of the hyperplane and extend $f$ to this point as in \cite{kirszbraun1934zusammenziehende}. Hence, $\Xi$ cannot be maximal.}
 	
 	Suppose now on the contrary that there exists an unbounded, $W$-convex set $\Xi.$ Let $a_k\to\infty$ be a sequence of points in $\Xi$ and let $\theta_k=a_k/\|a_k\|.$ Passing to a subsequence, we may assume that $\theta_k\to\theta_*.$ Then, 
 	\[
 	\inf_{t\in[0,1]}(a_{k+1}-a_k)'D_{aa}W(a_kt+a_{k+1}(1-t))(a_{k+1}-a_k)\ \ge\ 0\,. 
 	\]
 	Passing to a subsequence, we may assume that $\|a_{k+1}\|=2^k \|a_k\|$ and the whole sequence stays in $\cC(\theta_*,\eps).$ This is a contradiction. 
 \end{proof}
 
 \begin{proof}[Proof of Corollary \ref{concave-bounded}] We just need to show that the conditions of Proposition \ref{conv-bound} are satisfied. We have $D_aW(a)\ =\ 2\varphi'(a^\top H a)Ha$ and $D_{aa}W(a)\ =\ 4\varphi''(a^\top H a)H a a^\top H\ +\ 2\varphi'(a^\top H a) H.$ Thus, 
 	\[
 	b^\top D_{aa}Wb\ =\ 4\varphi''(a^\top H a) (a^\top H b)^2\ +\ 2\varphi'(a^\top H a) b^\top H b
 	\]
 	and the claim follows because $a^\top H b \approx b^\top H b\approx a^\top H a$ and the first (negative) term dominates when $a^\top H b\to\infty.$ 
 \end{proof}
 
 \section{Examples} 
 
 \subsection{\(\Xi\) is a hyper-plane}

 \begin{proposition}[$\Xi$ is a hyper-plane]\label{tamura} Suppose that $W(a)=a^\top H a$ and $g(\go)=\go.$ Define $P_+$ to be the orthogonal projection onto the span of eigenvectors associated with all positive eigenvalues of $V$.  Then,  
 	$\cA(\go)\ =\ \Sigma^{1/2}P_+\Sigma^{-1/2}\go\,$ is an optimal autoencoder. In particular, 
 	\begin{itemize}
 		\item the optimal feature manifold is if the $\nu_+(H)$-dimensional hyperplane $\Xi\ =\ \Sigma^{1/2}P_+\Sigma^{-1/2}\R^L;$ 
 		
 		\item The pool of every signal is an $(M-\nu(H))$-dimensional hyperplane, 
 		\[
 		Pool(a)\ =\ \{\go\ =\ a\ +\ (Id-\Sigma^{1/2}P_+\Sigma^{-1/2})y:\ y\in (Id-\Sigma^{1/2}P_+\Sigma^{-1/2})\R^L\}\cap\gO
 		\]
 	\end{itemize}
 	Furthermore, if $\det(H)\not=0,$ then the optimal autoencoder is unique. In particular, there are no non-linear optimal policies. 
 \end{proposition}

\begin{proof}[Proof of Proposition \ref{tamura}] 
In this case, Theorem \ref{converse} implies that, for any optimal autoencoder, $\Xi$ has to be monotonic, meaning that $(a_1-a_2)^\top H(a_1-a_2)\ge 0$ for all $a_1,a_2\in \Xi.$ The question we ask is: Under what conditions is $\cA(\go)=A\go$ with some matrix $A$ of rank $M_1\le M$ is optimal with $g(\go)=\go.$ Clearly, it is necessary that $p$ have linear conditional expectations,\footnote{This is, e.g., the case for all elliptical distributions, but also for many other distributions. See \cite{wei1999linear}.} $\mathbb{E}[\go|A\go]\ =\ A\go$.  But then, since $\mathbb{E}[A\go|A\go]=A\go,$ we must have $A^2=A,$ so that $A$ is necessarily a projection. Maximal monotonicity implies that $Q=A^\top H A$ is positive semi-definite, and\footnote{Here, $Q^{-1}$ is the Moore-Penrose inverse.} 
\[
\cA(b)\ =\ \min_{\go}(b^\top Hb+\go^\top A^\top H (A \go-2b))\ =\ b^\top HAQ^{-1}A^\top A(AQ^{-1}A^\top H -2Id)b
\]
with the minimizer $Q^{-1}A^\top H b.$ Thus, $A$ satisfies the fixed point equation $A\ =\ Q^{-1}A^\top H$ and hence $A^\top=HAQ^{-1}.$ Furthermore, maximality of $\Xi$ implies that 
\[
H\ +\ HAQ^{-1}A^\top A(AQ^{-1}A^\top H -2Id)\ =\ H-A^\top A
\]
is negative semi-definite. As a result, $(Id-A^\top)(H-A^\top A)(Id-A)=(Id-A^\top)H(Id-A)$ is also negative semi-definite, implying that $A$ and $Id-A$ ``perfectly split" positive and negative eigenvalues of $H$. Here, it is instructive to make two observations: First, optimality requires that $\cA(\go)$ ``lives" on positive eigenvalues of $H$. Second, maximality (the fact that $\cA(b)\le 0$ for all $b$) requires that $A$ absorbs all positive eigenvalues, justifying the term ``maximal". \end{proof}

 Proposition \ref{tamura} is a particularly clean illustration of our key results: $\Xi$ is a $\nu$-dimensional manifold (Theorem \ref{thm-unif}), and pools have dimension $M-\nu(H)$ (Corollary \ref{dimension-pool}) and are convex (Proposition \ref{main convexity}). One interesting observation is that maximality (Theorem \ref{converse}) takes the form of the requirement that $\Xi$ must be spanned by {\it all} eigenvectors with positive eigenvalues. Finally, Proposition \ref{uniqueness} ensures that the policy is unique. 
 
 \cite{tamura2018Bayesian} was the first to show that linear optimal policies of the form described in Proposition \ref{tamura} are optimal when $p$ is Gaussian. Proposition \ref{tamura} extends his results to general elliptic distributions and establishes the uniqueness of optimal policies. The key simplification in Proposition \ref{tamura} comes from the assumption that $p$ is elliptic, implying that the optimal autoencoder and the optimal feature manifold are linear. 
 
 \subsection{\(\Xi\) is a sphere} 
 
 Consistent with Proposition \ref{conv-bound}, the linear manifold $\Xi$ {\it extends indefinitely in all directions}. As we know from Corollary \ref{concave-bounded}, the situation changes when we abandon the assumption of quadratic preferences. The following is true.

 \begin{corollary}[$\Xi$ is a sphere]\label{sphere} Suppose that $g(\go)=\go\,\psi(\|\go\|^2)$ for some function $\psi\ge 0$ and $p(\go)=\mu_*(\|\go\|^2),$ and $W(a)\ =\ \varphi(\|a\|^2)$. Let $\beta\ =\ \mathbb{E}[\|\go\|].$
 	If $\varphi'(\beta^2)>0$ and 
 	\begin{equation}\label{dwor1}
 		\max_{\|b\|\le \sup_{x\ge 0} (x\psi(x^2))}(\varphi(\|b\|^2)-\varphi(\beta^2)+2\varphi'(\beta^2)\beta (\beta-\|b\|))\ \le\ 0\,, 
 	\end{equation}
 	then: (1) $\cA(\go)=\beta \go/\|\go\|$ is an optimal autoencoder; (2) the optimal feature manifold is the sphere $\{\go:\ \|\go\|=\beta\}$; and (3) pools are rays from the origin. The optimal autoencoder is unique if the maximum in \eqref{dwor1} is attained only when $\beta=\|b\|.$ 
 \end{corollary}

\begin{proof}[Proof of Corollary \ref{sphere}] The proof follows directly from Theorem \ref{converse}. Indeed, by this theorem, we only need to check three conditions: 

(1) $\Xi$ is $conv(g(\gO))$-maximal; (2) $\cA(\go)\ \in\ \cP_\Xi(g(\go))$ and (3) $a\ =\ \mathbb{E}[g(\go)|\cA(\go)=a]$. All these conditions follow directly from the hypotheses of Corollary \ref{sphere}. 
\end{proof}

\subsection{Separable \(W\). Example 1}

Suppose now that $W(a)\ =\ a_1\,\sum_{i=1}^N G(a_{i+1})\,,$ where $a_1=\mathbb{E}[\pi|s]$ and $a_{i+1}=\mathbb{E}[v_i|s],\ i=1,\cdots,N.$ For simplicity, we will assume that, for each $i$, $G_i$ is either strictly convex or strictly concave. Furthermore, we will also assume that the function 
$q(x)\ =\ -\sum_i (G_i'(x_i))^2/G''(x_i)$ does not change the sign for $x\in conv(\gO_{1+}).$\footnote{$\gO_{1+}\in\R^N$ is the projection of $\gO$ onto the last $N$ coordinates.} For example, this is the case when all of $G_i''$ have the same sign. Under these assumptions, $D_{aa}W$ is non-degenerate and Corollary \ref{char-pools} and Proposition \ref{main convexity} allow us to characterize signal pools as well as the local structure of $\Xi.$ 

\begin{proposition}\label{Rayo-Segal1} Let $\nu$ be the number of $G_i$ with $G_i''>0.$ There always exists a deterministic optimal autoencoder $\cA(\go)$ such that:

\begin{itemize}
\item The optimal feature manifold is a $(\nu+ {\bf 1}_{q(a)>0})$-dimensional Lipschitz manifold, while pools are at most $(N+1-(\nu+ {\bf 1}_{q(a)>0}))$-dimensional. If all $G_i''$ have the same sign, then all pools are convex. 

\item if $G_i''(a)>0$ for all $i,$ then $\nu=N, {\bf 1}_{q(a)>0}=0$ and for each each deterministic optimal autoencoder there exists a function $f(a_{1+}):\R^N\to \R$ such that $a_1(\go)=f(a_{1+}(\go))$ for all $\go$ and, hence, the optimal feature manifold $\Xi$ is an $N$-dimensional subset of the graph $\{(f(a_{1+}),a_{1+})\}.$ For each $i,$ the function $f(a_{1+}) (G_i'(a_{i+1}))^{1/2}$ is monotone increasing in $a_{i+1}$, and there exist functions $\kappa_i(a)$ such that the pool of Lebesgue-almost every signal $a_{1+}$ is a convex subset (a segment) of the one-dimensional line 
\begin{equation}\label{optimal-rs-1}
Pool(a_{1+})\ \subset\ \{\binom{\pi}{v}:\ a_i-v_i\ =\ (\pi-f(a))\kappa_i(a)\ for\ all\ i>1\,\}\ \subset\ R^{N+1}\,. 
\end{equation}
Furthermore, these lines are downward sloping on average in the following sense:
\[
\sum_i\kappa_i(a)G_i'(a)\ \ge\ 0
\]

\item if $G_i''(a)<0$ for all $i,$ then $\nu=0,\ {\bf 1}_{q(a)>0}=1,$ and hence $\Xi$ is a one-dimensional curve. For each deterministic optimal autoencoder there exists a map $f(a_{1})=(f_i(a_1))_{i=1}^N:\R\to \R^N$ such that $a_{1+i}(\go)=f_i(a_{1}(\go)).$ The function $\sum_i G_i(f_i(a_1))$ is monotone increasing in $a_1$  and there exists a map $\kappa(a_1):\R\to\R^N$ such that the pool of Lebesgue-almost every signal $a_{1}$ is a convex subset of the $N$-dimensional hyperplane
\begin{equation}\label{optimal-rs-1-1}
Pool(a_{1+})\ =\ \{\binom{\pi}{v}:\ \pi-f(a)\ =\ \kappa(a)^\top (a-v)\}\ \subset\ R^{N+1}\,. 
\end{equation}
\end{itemize}
\end{proposition}

\begin{proof}[Proof of Proposition \ref{Rayo-Segal1}] 
By direct calculation, 
\begin{equation}\label{cab1-1}
\begin{aligned}
&c(a,b)\ =\ W(b)\ -\ W(a)\ +\ D_aW(a)\,(a-b)\\
& =\ \sum_i (b_1G_i(b_{i+1})-a_1G_i(a_{i+1}))+\sum_i \Big(G_i(a_{i+1})(a_1-b_1)\ +\ a_1G_i'(a_{i+1})(a_{i+1}-b_{i+1})\Big)\\
&=\ 
b_1\sum_i (G_i(b_{i+1})-G_i(a_{i+1}))-a_1\sum_i G_i'(a_{i+1})(b_{i+1}-a_{i+1})\,.
\end{aligned}
\end{equation}
Let $a_{1+}=(a_{i+1})_{i=1}^N$ and $G(a_{1+})=(G_i(a_{1+i}))_{i=1}^N.$ Then, $D_aW\ =\ \binom{G(a_{1+})^\top {\bf 1}}{a_1 G'(a_{1+})}$ and\footnote{We use ${\bf 1}$ to denote a vector of ones.} 
\begin{equation}
D_{aa}W(a)\ =\ \begin{pmatrix}
0&G'(a_{1+})^\top\\
G'(a_{1+})&a_1\diag(G''(a_{1+}))
\end{pmatrix}
\end{equation}
By direct calculation, $D_aW$ is injective if all $G_i''$ the the same sign.\footnote{We have that $a=(D_aW)^{-1}(x)$ satisfies  $\sum_{i=1}^N G(a_{i+1})=x_1,\ a_{i+1}=(G_i')^{-1}(x_{i+1}/a_1)$ and hence there is a unique $a_1$ solving $\sum_{i=1}^N G((G_i')^{-1}(x_{i+1}/a_1))=x_1.$} In this case, convexity of pools can be guaranteed by Corollary \ref{main convexity}. Let $\nu$ be the number of $G_i$ with $G_i''>0.$ By direct calculation, $D_{aa}W(a)$ always has exactly $\nu+{\bf 1}_{q(a)>0}$ positive eigenvalues. 


In the first case, by direct calculation, we have 
\begin{equation}
Df\ =\ \binom{D_af}{I}
\end{equation}
and, hence, 
\[
Df^\top D_{aa}W(a) Df\ =\ f(a)\,\diag(G''(a))\ +\ (G' D_af^\top+D_af\,(G')^\top)\ \ge\ 0\,.
\]
In particular, diagonal elements are 
\[
2G_i'(a_i)f_{a_i}+G_i''(a_i)f\ =\ 2((G_i')^{1/2}f)_{a_i}\ \ge\ 0\,,
\]
implying the required monotonicity. Furthermore, it also implies that 
\[
Q=\diag(f(a)\,G''(a))^{-1}\ +\ \diag(f(a)\,G''(a))^{-1}(G' D_af^\top+D_af\,(G')^\top)\diag(f(a)\,G''(a))^{-1}\ \ge\ 0
\]
and hence the matrix 
\[
\begin{pmatrix}
(G')^\top QG'& (G')^\top Q D_af\\
(G')^\top Q D_af &(D_af)^\top Q D_af
\end{pmatrix}\ \ge\ 0
\]
that is
\[
\begin{pmatrix}
A+2AB&B+AC+B^2\\
B+AC+B^2&C+2CB
\end{pmatrix}\ \ge\ 0\,,
\]
where we have defined 
\begin{equation}
\begin{aligned}
&A=(G'(a))^\top \diag(f(a)G''(a))^{-1}G',\ B=(G'(a))^\top \diag(f(a)G''(a))^{-1}D_af,\\
&C=(D_af)^\top \diag(f(a)G''(a))^{-1}D_af
\end{aligned}
\end{equation}
In particular,
\[
1+(G'(a))^\top \diag(f(a)G''(a))^{-1}D_af\ >\ 0\,.
\]
The pool equation takes the form 
\[
Df^\top D_aaW(a)(\binom{f(a)}{a}\ -\ \binom{\pi}{v})\ =\ 0
\]
which is equivalent to the system 
\[
(f(a)\,\diag(G''(a))+D_af\,(G'(a))^\top)(a-v)\ =\ (\pi-f(a))G'(a)
\]
By the Sherman-Morrison formula, 
\begin{equation}
\begin{aligned}
&(f(a)\,\diag(G''(a))+D_af\,(G'(a))^\top)^{-1}\\
&=\ \diag(f(a)G''(a))^{-1}-\frac{\diag(f(a)G''(a))^{-1}D_af\,(G'(a))^\top \diag(f(a)G''(a))^{-1}}{(1+(G'(a))^\top \diag(f(a)G''(a))^{-1}D_af)}
\end{aligned}
\end{equation}
implying that 
\[
a_i-v_i\ =\ (\pi-f(a))\kappa_i(a)
\]
where 
\[
\kappa_i(a)\ =\ (f(a)G_i''(a_i))^{-1}\Big(G_i'(a_i)-D_{a_i}f\,\frac{((G'(a))^\top \diag(f(a)G''(a))^{-1}G'(a))}{(1+(G'(a))^\top \diag(f(a)G''(a))^{-1}D_af)}\Big)\,. 
\]
Then, 
\begin{equation}
\begin{aligned}
&\sum_i G_i'(a_i)\kappa_i(a)\ =\ ((G'(a))^\top \diag(f(a)G''(a))^{-1}G'(a))\\
&-((G'(a))^\top \diag(f(a)G''(a))^{-1}D_af)\frac{((G'(a))^\top \diag(f(a)G''(a))^{-1}G'(a))}{(1+(G'(a))^\top \diag(f(a)G''(a))^{-1}D_af)}\\
&=\ A-B\frac{A}{1+B}\ =\ \frac{A}{1+B}\ >\ 0\,. 
\end{aligned}
\end{equation}
In the concave $G$ case, we have 
\[
0\ \le\ (1,D_af^\top)\begin{pmatrix}
0&G'(a_{1+})^\top\\
G'(a_{1+})&a_1\diag(G''(a_{1+}))
\end{pmatrix}\binom{1}{D_af}\ =\ \sum_i (2f_i'(a_1)G_i'(f_i(a_1))+a_1(f_i'(a_1))^2G_i''(f_i(a_1)))\,.
\]
In particular, $D_af^\top G'(a)\ge 0.$ 
The pool equation is 
\[
(1,D_af^\top)\begin{pmatrix}
0&G'(a_{1+})^\top\\
G'(a_{1+})&a_1\diag(G''(a_{1+}))
\end{pmatrix}(\binom{a_1}{f(a_1)}\ -\ \binom{\pi}{v})\ =\ 0\,,
\]
that is 
\[
G'(a_{1+})^\top (f(a_1)-v)\ +\ D_af^\top G'(a) (a_1-\pi)+D_af^\top  a_1G''(a) (f(a_1)-v)\ =\ 0\,
\]
\end{proof}

\subsection{Separable \(W\). Example 2}

We now consider the case of $
W(a)\ =\ \sum_{i=1}^N a_i\,G_i(a_{i+N})\,,$ 
where $a_i=\mathbb{E}[\pi_i|s]$ and $a_{i+N}=\mathbb{E}[v_i|s]$. In this case, $D_{aa}W(a)$ is block-diagonal as there are no cross-effects across different pairs $\binom{a_i}{a_{i+N}}:$ only $a_i$ and $a_{i+N}$ are substitutes.
As a result, $D_{aa}W(a)$ always has exactly $N$ positive eigenvalues, independent of the properties of $G_i.$ 

Recall that a map $f:\R^N\to \R^N$ is monotone increasing  if $(a-b)^\top (f(a)-f(b))\ge 0$ for any $a,\ b\,\in\R^N.$ Let $a_{N+}=(a_{N+i})_{i=1}^N\in \R^N,\ a_{-N}=(a_{i})_{i=1}^N.$ 


We now assume that $G_i'(x)>\eps$ for some $\eps>0$ and define  $\varphi_i(b)$ to be the unique monotone increasing to the differential equation 
\begin{equation}
\varphi_i(x)'\ =\  (G_i'(\varphi_i(x)))^{-1/2},\ \varphi_i(0)=0\,, \varphi(b)\ =\ (\varphi_i(b_i))_{i=1}^N:\ \R^N\to\R^N\,.
\end{equation}
Define $\widetilde f(a)\ =\ \diag((G'(a))^{1/2})\,f(a)$ and $\widehat f(x)\ =\ \widetilde f(\varphi(x))\,.$
Note that when $N=1,$ we have $\widetilde f(a)=(G'(a))^{1/2} f(a)$ and $\tilde f(a)$ is monotonic if and only if so is $\widehat f.$ However, in multiple dimensions this is not the case anymore: It might happen that $\widehat f$ is monotonic, while $\widetilde f$ is not. 
The following is true. 

\begin{proposition} \label{arbitraryGi} There always exists a deterministic optimal autoencoder $\cA(\go).$ For each such policy, there exists a map $f=(f_i)_{i=1}^N:\ \R^N\to\R^N$ such that $a_{-N}(\go)=f(a_{N+}(\go))$ for all $\go$ and, hence, the optimal feature manifold $\Xi$ is the $N$-dimensional graph $\{(f(a_{N+}),a_{N+})\}$ of the map. Furthermore, the map $\widehat f(a_{N+})$ is monotone increasing. The pool of Lebesgue-almost every signal $a_{1+}$ is given by the $N$-dimensional hyperplane 
\begin{equation}\label{optimal-rs-2-gen}
Pool(a_{N+})\ =\ \{\binom{\pi}{v}:\ \pi\ =\ \kappa_1(a_{N+})\,v\ +\ \kappa_2(a_{N+})\ for\ all\ i\,\}\ \subset\ R^{2N}\,. 
\end{equation}
where 
\[
\kappa_1(a_{N+})\ =\ -(\diag(G')^{-1}(D_af)^\top \diag(G')  +\diag(f G'' /G'))\ \in \R^{N\times N}
\]
and 
\[
\kappa_2(a_{N+})\ =\ \diag(f G'){\bf 1}\ -\ \kappa_1(a_{N+})a_{N+}\ \in\ \R^N\,. 
\]
Furthermore, the matrix $\diag(G'(a_{N+}))\kappa_1(a_{N+})+0.5 \diag(f G'')$ is negative semi-definite. 
\end{proposition}

Monotonicity of the map $\widehat f$ is the multi-dimensional analog of simple coordinate-wise monotonicity of Proposition \ref{Rayo-Segal1}. 
The monotonicity of $\widehat f$ implies that 
\[
(x-y)^\top (\widehat f(x)-\widehat f(y))\ \ge\ 0
\]
for any $x,\ y.$ Thus 
\[
0\ \le\ (\varphi^{-1}(a)-\varphi^{-1}(b))^\top (\widetilde f(a)-\widetilde f(b))\ =\ \sum_i (\varphi_i^{-1}(a_i)-\varphi_i^{-1}(b_i))((G_i'(a_i))^{1/2}f_i(a)-(G_i'(b_i))^{1/2}f_i(b))
\]
That is, for any signal $s,$ the vectors 
\[
\binom{(G_i'(\mathbb{E}[v_i|s]))^{1/2} \mathbb{E}[\pi_i|s]}{\varphi_i^{-1}(\mathbb{E}[v_i|s])}
\]
are aligned. 


\begin{proof}[Proof of Proposition \ref{arbitraryGi}] We have 
\begin{equation}\label{cab1-2}
\begin{aligned}
&c(a,b)\ =\ W(b)\ -\ W(a)\ +\ D_aW(a)\,(a-b)\\
& =\ \sum_i (b_i G_i(b_{i+N})-a_i G_i(a_{i+N}))+\sum_i \Big(G_i(a_{i+N})(a_i-b_i)\ +\ a_i G_i'(a_{i+N})(a_{i+N}-b_{i+N})\Big)\\
&=\ \sum_i b_i (G_i(b_{i+N})-G_i(a_{i+N}))-\sum_i a_i G_i'(a_{i+N})(b_{i+N}-a_{i+N})\\
&=\ \sum_i  (f_i(b_{N+})-f_i(a_{N+}))(G_i(b_{i+N})-G_i(a_{i+N}))\\
&+\sum_i f_i(a_{N+}) (G_i(b_{i+N})-G_i(a_{i+N})-G_i'(a_{i+N})(b_{i+N}-a_{i+N}))\\
\end{aligned}
\end{equation}
When $b_{N+}\to a_{N+}$ and $da=b_{N+}-a_{N+},$ we get 
\[
0=\ (da)^\top \diag(G'(a)) D_af\,(da)\ +\ 0.5 (da)^\top \diag(f(a) G''(a)) (da)\,. 
\]
where $(D_af)_{i,j}=\partial f_i/\partial a_j.$ That is, the matrix 
\[
\diag(G'(a)) D_af\ +\ 0.5 \diag(f(a) G''(a))
\]
is positive semi-definite. Therefore, so is the matrix 
\[
Q(a)\ =\ \diag(G'(a)^{1/2}) D_af \diag(G'(a)^{-1/2})\ +\ 0.5 \diag(f(a) G''(a)G'(a)^{-1})
\]
Let now $a=\varphi(x).$ Then, by direct calculation, $Q(\varphi(x))$ is the Jacobian of $\widehat f(x)$ and the claim follows because a map is monotone increasing if and only if its Jacobian is positive semi-definite. Since monotone maps are differentiable Lebesgue-almost surely, we get the first order condition 
\[
\max_{a_{N+}}\Bigg(\sum_i \pi_i (G_i(v_{i})-G_i(a_{i+N}))-\sum_i f_i(a_{N+}) G_i'(a_{i+N})(v_i-a_{i+N})\Bigg)
\]
takes the form 
\[
-\diag(G')\pi\ -\ ((D_af)^\top \diag(G')  +\diag(f G'')) (v-a)\ +\ \diag(f G'){\bf 1}\ =\ 0
\]
and the claim follows. 
\end{proof}

\section{The Integro-Differential Equation}

\begin{proposition}\label{prop-change-variables} Let $F$ be a  bijective, bi-Lipshitz map,\footnote{A map $F$ is bi-Lipschitz if both $F$ and $F^{-1}$ are Lipschitz continuous.} $F:\ X\to \gO$ for some open set $X\subset \R^L.$ Let also $M_1\le M$ and $x=(\theta,r)$ with $\theta\in X_1$, the projection of $X$ onto $\R^{M_1}$ and $r\in X_2,$ the projection of $X$ onto $\R^{L-M_1}.$ Define 
\begin{equation}\label{cond-f}
\begin{aligned}
&f(\theta)\ =\ f(\theta;F)\ \equiv\ \frac{\int_{X_{2}} |\det(D F(\theta,r))|p(F(\theta,r))\,g(F(\theta,r))dr}
{\int_{X_{2}} |\det(D F(\theta,r))|p(F(\theta,r))dr}\,.
\end{aligned}
\end{equation}
Suppose that $f$ is an injective map, $f:X_1\to \R^{M}$ and define 
\begin{equation}\label{opt-dif}
\cA(b)\ =\ \min_{\theta\in X_1}\{W(b)-W(f(\theta))+D_aW(f(\theta))^\top(f(\theta)-b)\}\,.
\end{equation}
Suppose also that the min in \eqref{opt-dif} for $b=g(F(\theta,r))$ is attained at $\theta$ and that
$\cA(b)\ \le\ 0\ \forall\ b\in conv(g(\gO)).$ Then, $\cA(\go)=f((F^{-1}(\go))_1)$ is an optimal autoencoder. If $x_1=\arg\min$ in \eqref{opt-dif} with $b=F(\theta,r)$ for all $\theta,r$ and $\cA(b)<0$ for all $b\in conv(g(\gO))\setminus \Xi$ with $\Xi=f(X_1),$ then the optimal autoencoder is unique. 

If $f$ is Lipshitz-continuous and the minimum in \eqref{opt-dif} is attained at an interior point, we get a system of second order partial integro-differential equations for the $F$ map: 
\begin{equation}\label{pde}
D_{\theta}f(\theta)^\top D_{aa}W(f(\theta))(f(\theta)-g(F(\theta,r)))\ =\ 0\,.
\end{equation}
\end{proposition}

\begin{proof}[Proof of Proposition \ref{prop-change-variables}] Let $\Xi=f(X_1)\cap conv(g(\gO)).$ 

Clearly, $\Xi$ is an $M_1$-dimensional manifold, and we need to verify that $\cA(\go)=f((F^{-1}(\go))_1)$ satisfies the three conditions of  Theorem \ref{converse}: 
\begin{itemize}
\item $\cA(\go)\ =\ \mathbb{E}[g(\go)|\cA(\go)]$

\item $\cA(\go)\ \in \cP_\Xi(g(\go))$ for all $\go$

\item $\Xi$ is $conv(g(\gO))$-maximal. 
\end{itemize}
The first condition is equivalent to \eqref{cond-f} by the change of variables formula. 
The second condition is equivalent to the fact that the minimum in \eqref{opt-dif} is attained for $b=g(F(\theta,r))$. The third condition follows from the fact that $\cA(b)\ \le\ 0\ \forall\ b\in conv(g(\gO)).$
\end{proof}

\section{Extend Empirical Studies}
\label{sec:empirics}

The key testable implication of our theory is the existence of an optimal bottleneck (latent) dimension for the encoder: With too few latent dimensions, the model is not rich enough; with too many, it encodes malignant dimensions that hurt (or simply do not improve) performance: The encoded information ``saturates.''. Within this section, we present a pseudocode of the algorithms, as well as provide information on training and hyperparameter details, along with the neural network architectures employed in our studies. Our repository is available at \url{https://github.com/tengandreaxu/benign-autoencoders}.

\subsection{Distance regularized GANs}

{\bf Dataset.} In this experiment, we utilized the CelebA-HQ dataset, which was introduced by \cite{karras2017progressive} and is a subset of the larger CelebA dataset \cite{liu2015faceattributes}. The CelebA-HQ dataset consists of 30,000 high-quality images of celebrities. We specifically opted for the CelebA-HQ dataset due to constraints in available hardware resources, as it provides a more manageable dataset compared to the original CelebA dataset, which contains 202,599 images \cite{liu2015faceattributes}.

{\bf Preprocessing.} To prepare the images for training, we resize the original RGB images from a resolution of $1024 \times 1024$ to a smaller size of $64 \times 64$. Subsequently, we normalize the images by adjusting their pixel values to have a mean of 0.5 and a standard deviation of 0.5. This common normalization step helps standardize the data and facilitates convergence during training.

{\bf Training.} Standard GAN training is known for its tendency to miss the true data-generating distribution modes. To address this limitation, we adopt a distance-regularized GAN training approach inspired by the work of \cite{che2016mode}. Additionally, following the insights from \cite{goodfellow2020generative}, we alternate between one gradient descent step on the discriminator, denoted as $D$, and one step on the autoencoder, denoted as $\mathcal{A} \coloneqq \cD \circ \cE$. For optimization, we employ minibatch SGD and utilize the Adam solver proposed by \cite{kingma2014adam}. We set the learning rate to 0.0002 and choose momentum parameters of $\beta_1 = 0.5$ and $\beta_2 = 0.999$. $D$ and $\mathcal{A}$ are trained for 100 epochs with a batch size of 128. The weights of the networks are initialized by sampling from a normal distribution with a mean of zero and a standard deviation of 0.02 \cite{radford2015unsupervised}.

{\bf Loss Function.} The discriminator $D_w$ is optimized via gradient ascent and its loss function is defined as 
\begin{align*}
    \mathcal{L}_{adv}\ =\ \log (D_{w}(\mathbf{x}))\ +\ \log (1-D_w(\cD_{\theta}(\cE_{\phi}(\mathbf{x})))),
\end{align*}
while the $\cE_{\phi}$ and $\cD_{\theta}$ both maximize
\begin{align*}
    \mathcal{L_{BAE}}\ =\ \log (D_{w}(\cD_{\theta}(\cE_{\phi}(\mathbf{x}))))\ -\ \lVert \mathbf{x} - \cD_{\theta}(\cE_{\phi}(\mathbf{x})) \rVert_2^2.
\end{align*}
{\bf Algorithm.}  Algorithm~\ref{alg:distance_gan} summarizes what described above. We run Algorithm~\ref{alg:distance_gan} for a grid of $\nu \in \{1, 10, 50, 100, 500, 1000\}$. 
\begin{figure}
    \centering
    \makebox[\textwidth]{\includegraphics[width=1.4\textwidth]{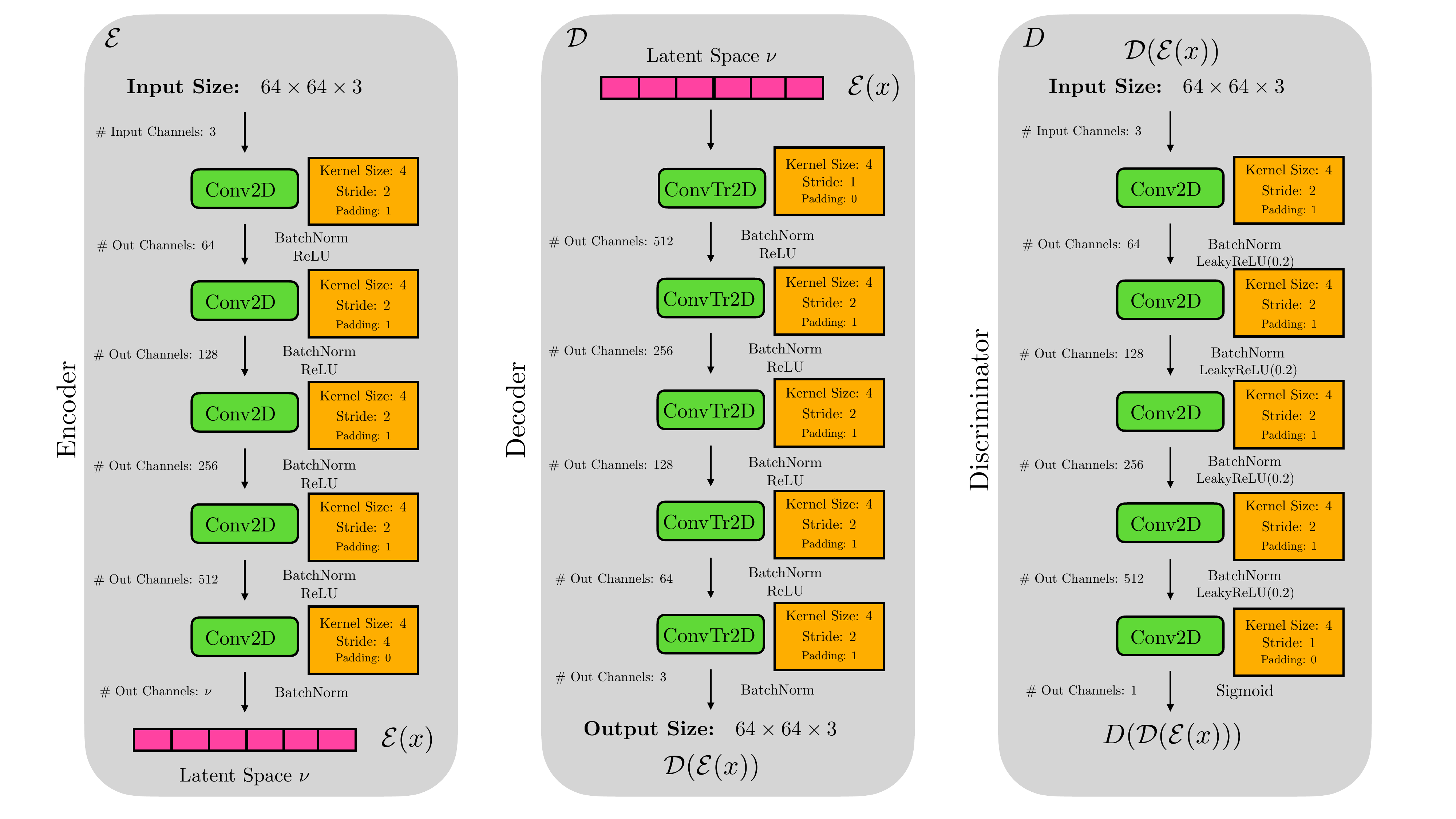}}
    \caption{From left to right, the exact neural network architectures used for the encoder $\cE$, the decoder $\cD$, and the discriminator $D$ in the distance-regularized GAN.}
    \label{fig:arch_distance_gan}
\end{figure}
\begin{algorithm}[H]
\caption{Distance Regularized GAN}
\label{alg:distance_gan}
\begin{algorithmic}[1]
\State \textbf{Input:} Training set $\cX$, number of epochs $N$, batch size $m$, number batches $B$, and latent dimension $\nu$.
\State Initialize discriminator $D_{w}(x)$ with $w \sim \mathcal{N}(0, 0.02).$
\State Initialize autoencoder $\mathcal{A} \coloneqq \cD_{\theta}(\cE_{\phi}(x))$ with $\theta \sim \cN(0, 0.02)$ and $\phi \sim \cN(0, 0.02).$ 
\For{$n \gets 1$ \textbf{to} $N$}
\For{$b \gets 1$ \textbf{to} $B$}
\State Sample $\{\mathbf{x_1}, \mathbf{x_2}, \ldots, \mathbf{x_m} \}$ from data generating distribution $p_{data}(x)$
\State Update discriminator $D$ using SGD with gradient ascent:

\begin{align*}
\nabla_{w}\frac{1}{m} \sum_{i=1}^{m}[\log (D_{w}(\mathbf{x_i})) + \log (1-D(\cD_{\theta}(\cE_{\phi}(\mathbf{x_i}))))]    
\end{align*}

\State Update encoder $\cE_{\phi}$ and decoder/generator $\cD_{\theta}$ using SGD gradient ascent:

\begin{align*}
    \nabla_{\theta, \phi} \frac{1}{m} \sum_{i=1}^{m}[\log (D_{w}(\cD_{\theta}(\cE_{\phi}(\mathbf{x_i})))) - \lVert \mathbf{x_i} - \cD_{\theta}(\cE_{\phi}(\mathbf{x_i})) \rVert_2^2] 
\end{align*}
\EndFor
\EndFor
\end{algorithmic}
\end{algorithm}

{\bf Architectures.} Figure~\ref{fig:arch_distance_gan} shows our simple DCGAN architecture. 

{\bf Evaluation.} In line with the approach described in \cite{heusel2017gans}, we generate a large number of images from our generative model. In this case, we match the total count of images in the original dataset, 30,000. We utilized the PyTorch port version of the official FID implementation to compute the FID score. The original implementation was initially developed in TensorFlow by \cite{heusel2017gans}, and a PyTorch port of the FID implementation can be found in \cite{Seitzer2020FID}. The TensorFlow version of the FID implementation is available in the official repository of the authors.\footnote{\url{https://github.com/bioinf-jku/TTUR}}

{\bf Results.} We report the results outlined in the main paper. To demonstrate the existence of an optimal $\nu$, we train the auto-encoder while varying $\nu \in \{1,10,50,100,500,1000\},$ maintaining constant architectures for $D$, and the non-bottleneck layers of $\cE_{\phi}$ and $\cD_{\theta}$. Our experiment, conducted on the CelebA-HQ dataset \cite{karras2017progressive, CelebAMask-HQ}, assesses the quality of the generative model using the FID score. Figure~\ref{fig:distance_gan_app} indicates a striking agreement with our theory, with the optimal latent dimension $\nu$ being about 100. Conversely, when the latent dimension $\nu$ becomes larger, the performance of the generative model deteriorates.  

\begin{figure}
    \centering
    \makebox[\textwidth]{\includegraphics[width=1.2\textwidth]{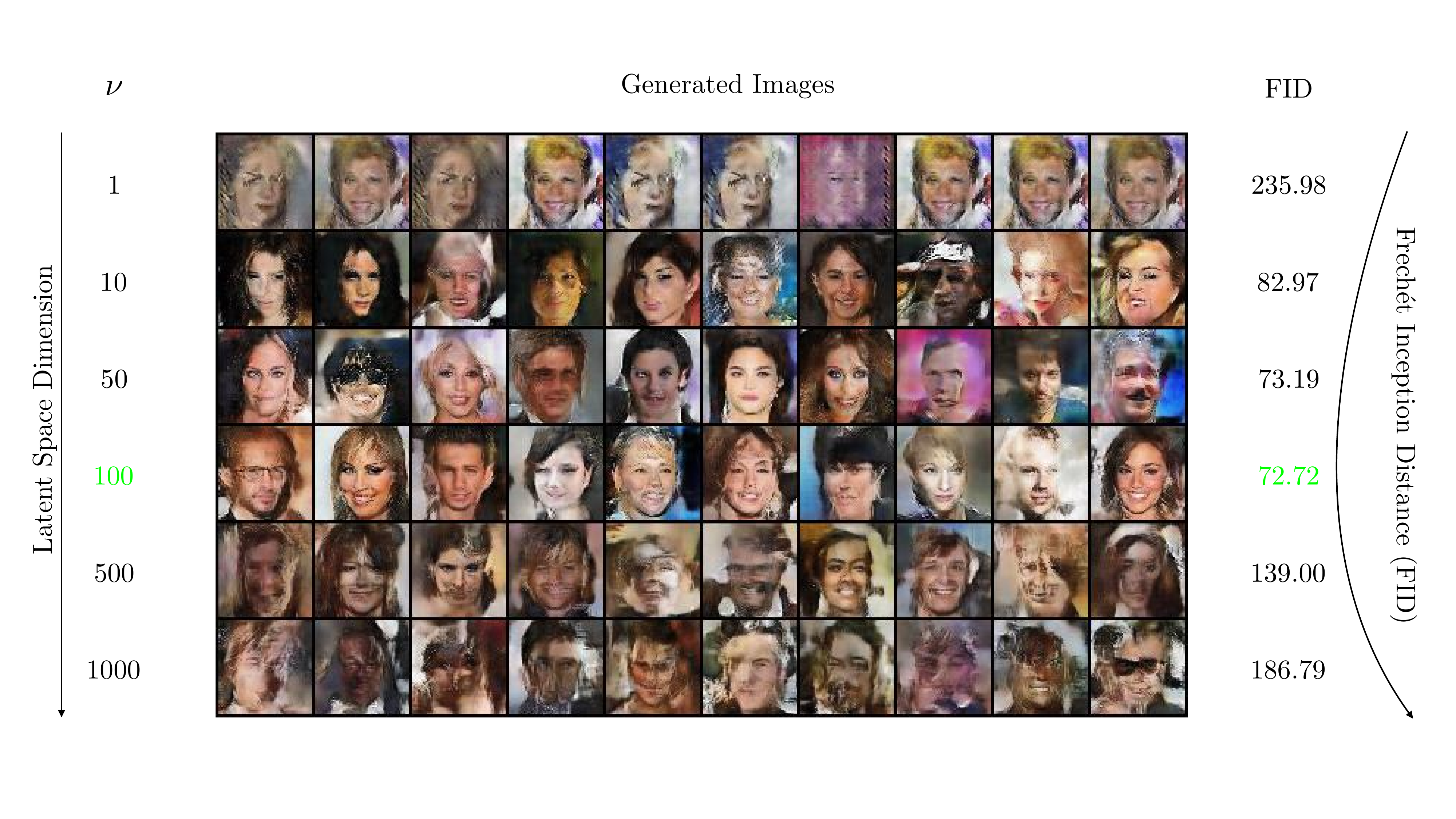}}
    \caption{Distance-regularized GAN on CelebA-HQ. FID score with varying latent space dimension
$\nu$, while keeping constant the discriminator $D$ and the non-bottleneck layers of the decoder $\cD_{\theta}$ and
encoder $\cE_{\phi}$ architectures. Images were resized to $64 \times 64$.}
    \label{fig:distance_gan_app}
\end{figure}

\subsection{Context-Encoders}

{\bf Dataset.} In this experiment, we once again utilize the CelebA-HQ dataset. We further divide the dataset into a train set and a test set, with sizes of 26,000 and 4,000, respectively.

{\bf Preprocessing.} To prepare the images for training, we resize the original RGB images from a resolution of $1024 \times 1024$ to a smaller size of $128 \times 128$. Subsequently, we normalize the images by adjusting their pixel values to have a mean of 0.5 and a standard deviation of 0.5. Finally, we apply a mask $\hat M$ of size $64 \times 64$ to patch the center of the image.

{\bf Training.} Following the insights from \cite{goodfellow2020generative}, we alternate between one gradient descent step on the discriminator, denoted as $D$, and one step on the autoencoder, denoted as $\mathcal{A} \coloneqq\ \cD\circ\ \cE$. For optimization, we employ minibatch SGD and utilize the Adam solver proposed by \cite{kingma2014adam}. We set the learning rate to 0.0002 and choose momentum parameters of $\beta_1 = 0.5$ and $\beta_2 = 0.999$. $D$ and $\mathcal{A}$ are trained for 150 epochs with a batch size of 32. The weights of the networks are initialized by sampling from a normal distribution with a mean of zero and a standard deviation of 0.02 \cite{radford2015unsupervised}. Denote with $x$ the pre-processed image, then $(1 - \hat M)\odot\ x$ is the masked image, $\hat M\odot\ x$ is the content, and $\cD_{\theta}(\cE_{\phi}((1 - \hat M)\odot\ x))$ is the reconstructed content.

{\bf Loss Function.} The discriminator $D_w$ is optimized via gradient descent, and its loss function is defined as 
\begin{align*}
    \mathcal{L}_{adv}\ =\ \lVert \mathbf{1} - D_w((1 - \hat M)\odot\ x))\rVert^2 + \lVert D_w(\cD_{\theta}(\cE_{\phi}((1 - \hat M)\odot\ x))) \rVert^2,
\end{align*}
while the $\cE_{\phi}$ and $\cD_{\theta}$ both minimize
\begin{align*}
    \mathcal{L_{BAE}}\ =\ \lambda_{adv}\lVert \mathbf{1} - D_w(\cD_{\theta}(\cE_{\phi}((1 - \hat M)\odot\ x))) \rVert^2\ +\ \lambda_{rec}\lVert \hat M\odot\ x - \cD_{\theta}(\cE_{\phi}((1 - \hat M)\odot\ x)) \rVert.
\end{align*}

Following the original context-encoder paper \cite{pathak2016context}, we set $\lambda_{adv} =\ 0.001$ and $\lambda_{rec} =\ 0.999$.

\begin{figure}
    \centering
    \makebox[\textwidth]{\includegraphics[width=1.4\textwidth]{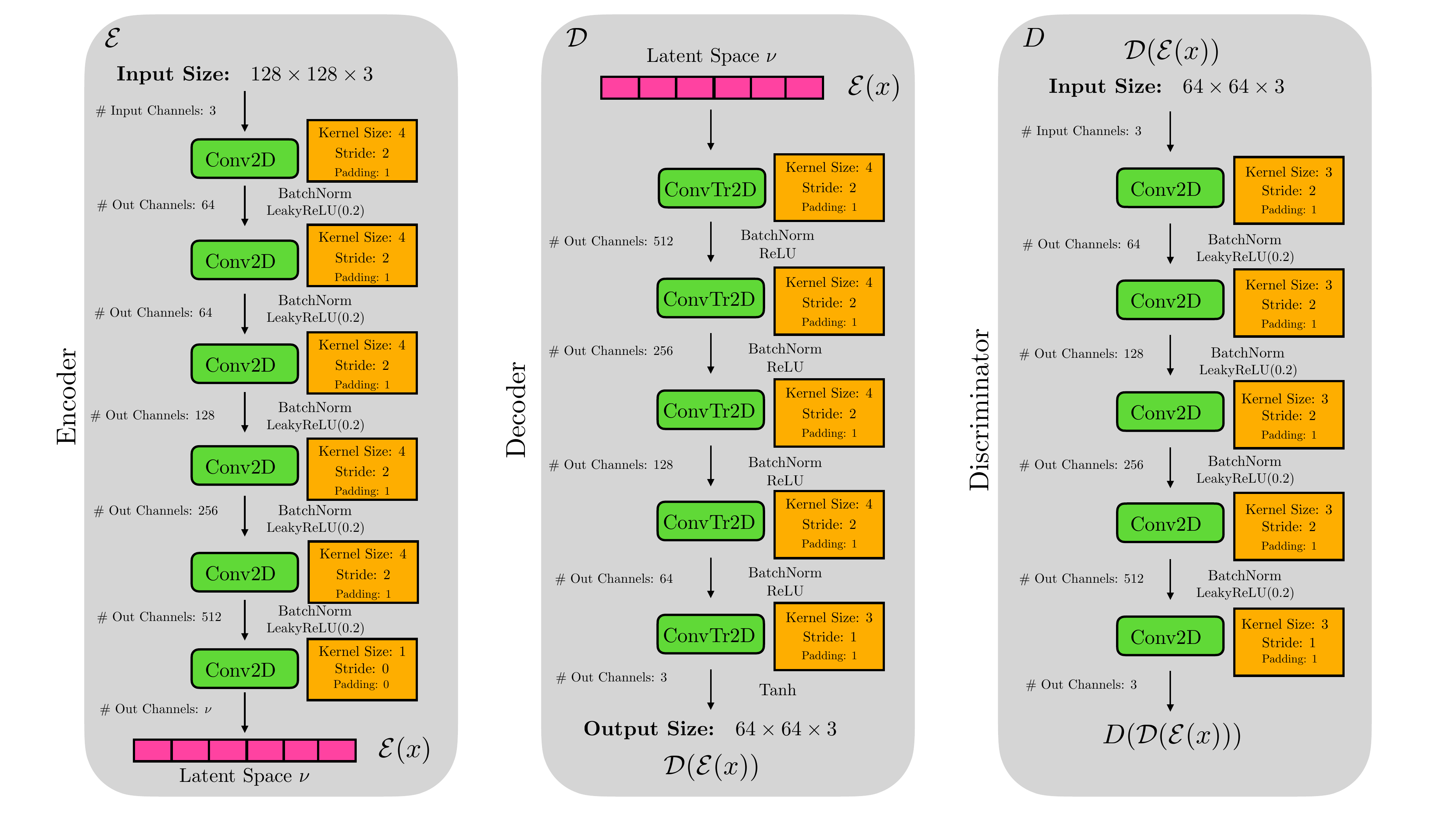}}
    \caption{From left to right, the exact neural network architectures used for the encoder $\cE$, the decoder $\cD$, and the discriminator $D$ in the context-encoder experiments.}
    \label{fig:arch_context_encoder}
\end{figure}

{\bf Algorithm.} We run Algorithm~\ref{alg:context_gan} for a grid of $\nu \in \{1 \times 4 \times 4, 10 \times 4 \times 4, 50 \times 4 \times 4, 100 \times 4 \times 4, 500 \times 4 \times 4, 1000 \times 4 \times 4, 4000 \times 4 \times 4 \}$. 
\begin{figure}
    \centering
    \makebox[\textwidth]{\includegraphics[width=1.2\textwidth]{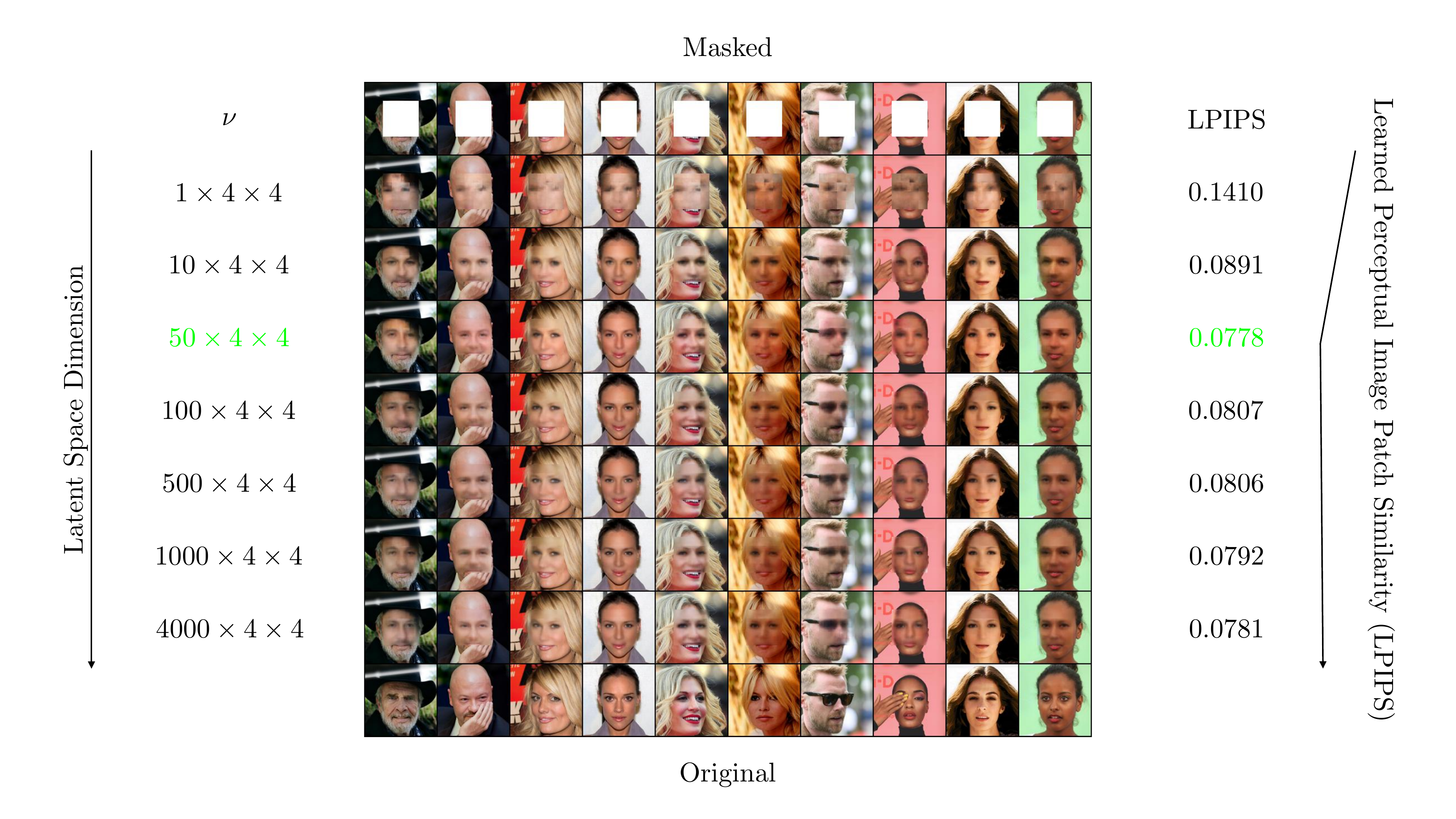}}
    \caption{Context-Encoder on CelebA-HQ. LPIPS score with varying latent space dimension $\nu$, while keeping constant the discriminator $D$ and the non-bottleneck layers of the decoder $\cD_{\theta}$ and encoder $\cE_{\phi}$ architectures. Images were resized to $128 \times 128.$ The mask area is $64 \times 64$. }
    \label{fig:context_app}
\end{figure}
\begin{algorithm}[H]
\caption{Context-Encoder}
\label{alg:context_gan}
\begin{algorithmic}[1]
\State \textbf{Input:} Training set $\cX$, number of epochs $N$, batch size $m$, number batches $B$, latent dimension $\nu$, $\lambda_{adv}$, and $\lambda_{rec}$.
\State Initialize discriminator $D_{w}(x)$ with $w \sim \mathcal{N}(0, 0.02).$
\State Initialize autoencoder $\mathcal{A} \coloneqq \cD_{\theta}(\cE_{\phi}(x))$ with $\theta \sim \cN(0, 0.02)$ and $\phi \sim \cN(0, 0.02).$ 
\For{$n \gets 1$ \textbf{to} $N$}
\For{$b \gets 1$ \textbf{to} $B$}
\State Sample $\{\mathbf{x_1}, \mathbf{x_2}, \ldots, \mathbf{x_m} \}$ from data generating distribution $p_{data}(x)$
\State Update discriminator $D$ using SGD with gradient descent:

\begin{align*}
\nabla_{w}\frac{1}{m} \sum_{i=1}^{m} \bigg[((1 - D_w((1 - \hat M)\odot\ \mathbf{x_i})))^2 + (D_w(\cD_{\theta}(\cE_{\phi}((1 - \hat M)\odot\ \mathbf{x_i}))))^2 \bigg],
\end{align*}

\State Update encoder $\cE_{\phi}$ and decoder/generator $\cD_{\theta}$ using SGD gradient descent:

\begin{align*}
    \nabla_{\theta, \phi} \frac{1}{m} \sum_{i=1}^{m} \bigg[\lambda_{adv}\big( 1 - D_w(\cD_{\theta}(\cE_{\phi}((1 - \hat M)\odot\ \mathbf{x_i}))) \big)^2\ \\
    +\ \lambda_{rec}\big( \hat M\odot\ \mathbf{x_i} - \cD_{\theta}(\cE_{\phi}((1 - \hat M)\odot\ \mathbf{x_i})) \big) \bigg].
\end{align*}
\EndFor
\EndFor
\end{algorithmic}
\end{algorithm}

{\bf Architectures.} Figure~\ref{fig:arch_context_encoder} shows our simple DCGAN architecture. 

{\bf Evaluation.} We in-paint the 4,000 samples in the test set and compute the LPIPS distance \cite{zhang2018unreasonable} between the in-painted image and the ground truth. We used the authors' official PyTorch implementation \cite{zhang2018unreasonable} to compute LPIPS.\footnote{\url{https://github.com/richzhang/PerceptualSimilarity}}

{\bf Results.} Similar to the distance GAN experiment, the results offer empirical evidence that supports our main Theorem~\ref{main-gen-char-app}. Importantly, we noticed that the optimal LPIPS score is attained when utilizing a {\it compressibility dimension} of approximately $50 \times 4 \times 4$. Figure~\ref{fig:context_app} shows that further increasing the dimension did not lead to improved performance.

\subsection{Evaluating the quality of the generator with a discriminator}

\begin{figure}[ht]
\centering

\makebox[\textwidth][c]{\parbox{\textwidth}{%
\begin{subfigure}{.5\textwidth}
\includegraphics[width=\linewidth]{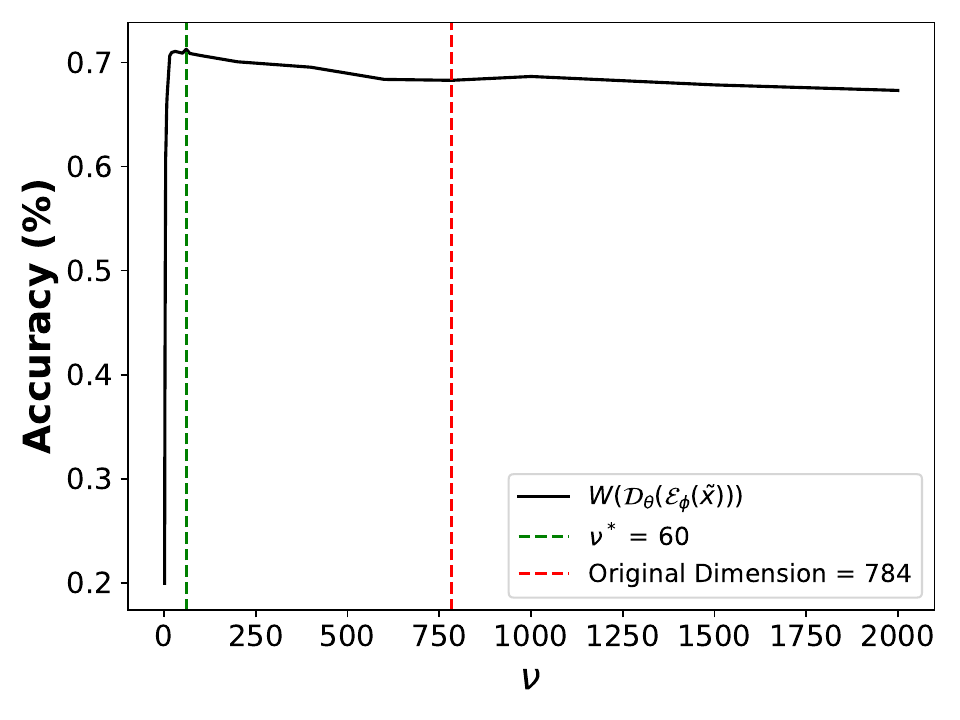}
\caption{MNIST}
\label{fig:ensembles_cnn1_ew}
\end{subfigure}%
\begin{subfigure}{.5\textwidth}
\includegraphics[width=\linewidth]{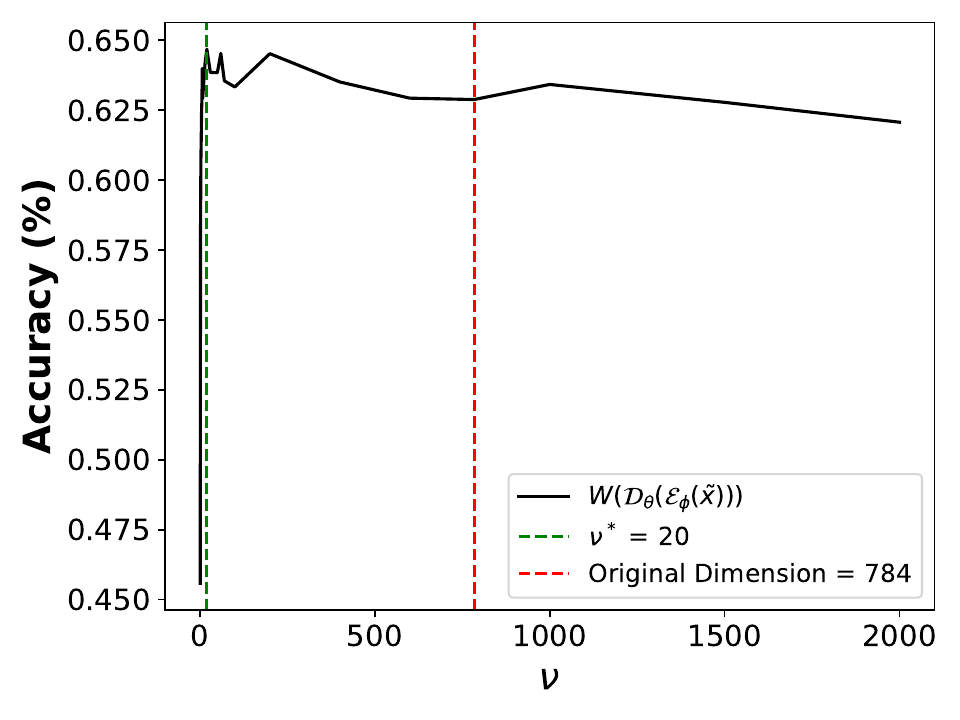}
\caption{FMNIST}
\label{fig:ensembles_cnn5_ew}
\end{subfigure}
}}

\caption{The pre-trained discriminator $D_w$ achieves the best accuracy when $\nu \ll L$. The accuracy is computed using the reconstructed data $\hat x = \cD_{\theta}(\cE_{\phi}(\tilde x))$ from noisy test data $\tilde x =\ x + \epsilon.$ } 
\label{fig:denoising}
\end{figure}

{\bf Dataset.} We conduct the experiments on the MNIST \cite{lecun1989handwritten} and FMNIST \cite{xiao2017fashion} datasets.

{\bf Preprocessing.} To prepare the images for training we normalize the images by adjusting their pixel values to have a mean of 0.5 and a standard deviation of 0.5.

{\bf Training.} We pre-train a discriminator $D$ for each dataset to achieve high accuracy ($99\%$ for MNIST and $92\%$ for FMNIST, respectively). Following this step, we optimize $\cD_{\theta}$ and $\cE_{\phi}$ with respect to the binary cross-entropy distance, $d(x, \cD_{\theta}(\cE_{\phi}(\tilde{x}))$, where $\tilde{x} = x + \epsilon$\footnote{$\epsilon \sim \mathcal{N}(0,1).$ One can view this noise as a simple form of a distribution shift.} is the noised image. The autoencoder $\cA \coloneqq \cD\circ\ \cE$ goal is to find the optimal latent space $\nu$ to denoise $\tilde x$. This process is further penalized by an additional classifier distance, given by the cross-entropy $\ell(\cdot,\cdot)$ with the reconstructed image, $\ell(y, D(\cD_{\theta}(\cE_{\phi}(\tilde{x}))$. For optimization, we employ minibatch SGD and utilize the Adam solver proposed by \cite{kingma2014adam}. We set the learning rate to 0.001 and choose momentum parameters of $\beta_1 = 0.5$ and $\beta_2 = 0.999$. $\cD_{\theta}$ and $\cE_{\phi}$ are trained for 20 epochs with a batch size of 32. The weights of the networks are initialized using He initialization \cite{he2015delving}.

{\bf Loss Function.} First, the discriminator $D_w$ is optimized via gradient descent minimizing cross-entropy $\ell(y, D_w(x))$. Then, with $D_w$ fixed, $\cE_{\phi}$ and $\cD_{\theta}$ both minimize
\begin{align*}
    \mathcal{L_{BAE}}\ =\ \ell(y, D_w(\cD_{\theta}(\cE_{\phi}(\tilde x))))\ +\ d(x, \cD_{\theta}(\cE_{\phi}(\tilde x)))
\end{align*}

{\bf Algorithm.} We run Algorithm~\ref{alg:context_gan} for a grid of \\
$\nu \in \{ 1, 2, 3, 4, 5, 6, 7, 8, 15, 20, 30, 50, 60, 70, 100, 200, 400, 600, 784, 1000, 1500, 2000\}$. 

\begin{algorithm}[H]
\caption{Quality of an Autoencoder Evaluation}
\label{alg:quality_autoencoder}
\begin{algorithmic}[1]
\State \textbf{Input:} Training set $\cX$, number of epochs $N$, batch size $m$, number batches $B$, a pre-trained classifier $D$.
\State Initialize autoencoder $\mathcal{A} \coloneqq \cD_{\theta}(\cE_{\phi}(x))$ with random weights.
\For{$n \gets 1$ \textbf{to} $N$}
\For{$b \gets 1$ \textbf{to} $B$}
\State Sample $\{\mathbf{x_1}, \mathbf{x_2}, \ldots, \mathbf{x_m} \}$ from data generating distribution $p_{data}(x)$
\State Let $\mathbf{\tilde  x_i}\ =\ \mathbf{x_i} + \epsilon$, where $\epsilon \sim \mathcal{N}(0,1)$, for every $i=1, \ldots, m$.
\State Update encoder $\cE_{\phi}$ and decoder/generator $\cD_{\theta}$ using SGD gradient descent:

\begin{align*}
    \nabla_{\theta, \phi} \frac{1}{m} \sum_{i=1}^{m}[\log (D(\cD_{\theta}(\cE_{\phi}(\mathbf{\tilde x_i})))) + d (\mathbf{x_i},\cD_{\theta}(\cE_{\phi}(\mathbf{\tilde x_i})))] 
\end{align*}
\EndFor
\EndFor
\end{algorithmic}
\end{algorithm}

{\bf Architectures.} These datasets are straightforward yet suitable for our experiments. The encoder $\cE$ is implemented as a simple dense 4-Layer MLP with ReLU activation functions and a descending number of nodes: 784, 512, 256, $\nu$. Likewise, the decoder $\cD$ mirrors the architecture of the encoder $\cE$ and consists of a 4-Layer MLP with the following node counts: $\nu$, 256, 512, 784. The output layer of the decoder utilizes a Sigmoid activation function. Lastly, the discriminator $D$ is implemented as a 4-Layer convolutional neural network comprising of 3 convolutional layers and 1 dense (output) layer. The convolution layers are equipped with 16, 32, and 64 filters (for FMNIST we use twice as much), respectively, along with ReLU activation functions and a Max-Pool(2,2) operation. Prior to the output layer, we incorporate global average pooling to flatten the data for classification purposes.

{\bf Evaluation.} After training, we evaluate the  autoencoder $\cA \coloneqq\ \cD_{\theta}(\cE_{\phi}(x))$ computing the accuracy of $D_w(\cD_{\theta}(\cE_{\phi}(\tilde x))),$ where $\tilde x$ is the noisy test set as described above.

{\bf Results.} As outlined in the main paper, we varied the latent dimension $\nu$, which represents the bottleneck size of $\cE_{\phi}$ and thus the input shape of $\cD_{\theta}$. We then reported the accuracy achieved by the fixed discriminator $D_w$ on the reconstructed test set $\hat{x} = \cA(\tilde{x})$. Figure~\ref{fig:denoising} shows a "peak" in accuracy achieved with a lower-dimensional $\nu$ indicated by a dashed green vertical line, while a higher latent space dimension results in even worse performance.

\clearpage
\bibliography{scibib.bib}
\bibliographystyle{jf}

\end{document}